%% file: main.tex
\documentclass{article} 
\usepackage{iclr2027_conference,times}

\usepackage{hyperref}
\usepackage{url}
\usepackage{graphicx}
\usepackage{booktabs}
\usepackage{multirow}
\usepackage{bm}
\usepackage{tabularx}
\usepackage{amsmath} 
\usepackage{wrapfig}
\usepackage{algorithm}
\usepackage{algorithmic}
\usepackage{subcaption}
\usepackage{bbm}
\usepackage{amssymb}
\usepackage[most]{tcolorbox}
\usepackage[utf8]{inputenc} 
\usepackage[T1]{fontenc}    
\usepackage{hyperref}       
\usepackage{url}            
\usepackage{booktabs}       
\usepackage{amsfonts}       
\usepackage{nicefrac}       
\usepackage{microtype}      
\usepackage{xcolor}         
\usepackage{booktabs}
\usepackage{multirow}
\usepackage{graphicx}
\usepackage{bm}
\newtheorem{proposition}{Proposition}

\title{Beyond Site Agreement: Re-estimation for Brain Network Generalization}

\author{Yingxu Wang\textsuperscript{1}\thanks{Equal contribution.}, Kunyu Zhang\textsuperscript{2}\footnotemark[1], Yanwu Yang\textsuperscript{3}, Thomas Wolfers\textsuperscript{3},
\\
\textbf{Yujie Wu\textsuperscript{4}, Siyang Gao\textsuperscript{5}, Nan Yin\textsuperscript{6}}
\\
\textsuperscript{1} Mohamed bin Zayed University of Artificial Intelligence \;\;
\textsuperscript{2} Zhengzhou University \\
\textsuperscript{3} University Hospital T\"{u}bingen \;\;
\textsuperscript{4} The Hong Kong Polytechnic University \\
\textsuperscript{5} City University of Hong Kong \;\;
\textsuperscript{6} The Education University of Hong Kong \\
\texttt{\{yingxv.wang,kunyu.zky,yangyanwu1111,yinnan8911\}@gmail.com} \\
\texttt{dr.thomas.wolfers@gmail.com, wu-yj16@tsinghua.org.cn} \\
\texttt{siyangao@cityu.edu.hk} }

\iclrfinalcopy
\def\method{BRIO} 
\begin{document}

\maketitle

\begin{abstract}
Cross-site out-of-distribution (OOD) generalization in resting-state functional magnetic resonance imaging (rs-fMRI) often relies on learning task-discriminative representations from full-scan functional connectivity (FC) graphs and promoting invariance across source sites. However, FC graphs are estimated from finite, temporally correlated blood-oxygen-level-dependent (BOLD) sequences. Cross-site agreement therefore does not necessarily imply that predictive evidence remains supported under FC re-estimation within the same scan. In this paper, we propose \textbf{B}rain Network \textbf{R}e-estimation-\textbf{I}nformed \textbf{O}OD Learning (\method{}), a framework that uses within-scan FC re-estimation to guide cross-site alignment. \method{} maps full-scan graphs and their re-estimates into consistently indexed connectome factors, enabling comparisons of their predictive contributions. It assesses re-estimation support from changes in these contributions relative to within-class subject variability and class separation. For each source-site pair and class, this task-calibrated support from both sites is combined with predictive relevance to form pairwise qualifications, which determine relative factor weights and overall alignment strength. Leave-one-site-out experiments on four real-world datasets (ABIDE, REST-meta-MDD, SRPBS, and ABCD) show that \method{} consistently outperforms competitive baselines, with relative improvements of up to 3.8\% in accuracy. These gains also persist under an alternative brain parcellation on ABIDE.
\end{abstract}

\input{main/1_intro}
\input{main/2_related}
\input{main/4_method}
\input{main/5_experiments}

\input{main/6_conclusion}

\bibliography{reference}
\bibliographystyle{iclr2027_conference}

\clearpage
\appendix
\input{main/7_appendix}

\end{document}

%% file: main/1_intro.tex
\section{Introduction}

Resting-state functional magnetic resonance imaging (rs-fMRI) enables non-invasive characterization of macroscale functional interactions among brain regions~\cite{finn2015functional,bessadok2022graph}. A standard pipeline parcellates the brain into regions of interest (ROIs) and constructs subject-level functional connectivity (FC) graphs, where nodes represent ROIs and edges encode statistical dependencies estimated from regional blood-oxygen-level-dependent (BOLD) time series~\cite{kan2022brainnttf,zhang2026modeling}. To leverage these interregional dependencies for predicting neurological and psychiatric disorders, graph neural networks (GNNs) aggregate information across connected ROIs to learn subject-level predictive representations~\cite{li2021braingnn,gu2025fchgnn}.

Despite this progress, brain network models often generalize poorly to unseen sites because differences in scanners, acquisition protocols, and cohort composition induce systematic shifts in FC distributions~\cite{qiu2025metaexplainer,yamashita2025computational,yu2018statistical}. To address these shifts, many cross-site out-of-distribution (OOD) methods seek FC graph representations that are both task-discriminative and invariant across source sites~\cite{li2025out,xu2025brainood,yu2025causal,wang2025protomol,zhang2026brainriem}. They pursue this goal through site-confounder removal, invariant feature or subgraph learning, and class-conditional alignment~\cite{wang2026brain,chen2022learning,wang2024dissecting}. However, each underlying FC graph is a finite-sample estimate derived from a temporally correlated BOLD sequence~\cite{birn2013effect,noble2019decade,pakravan2026uncertainty,xiang2025jailbreaking,zhang2025mvho}. Re-estimating FC using different temporal samples from the same scan may change the estimated connectivity patterns and their contributions to model predictions~\cite{yangbrain}. These within-subject changes are not directly characterized by the discrepancy between representation distributions across sites. Consequently, low cross-site discrepancy among full-scan representations does not by itself reveal whether the predictive evidence used for alignment remains supported under within-scan FC re-estimation.

In this paper, we revisit cross-site OOD generalization in brain networks by accounting for variability from finite-sample FC estimation. This raises three sequential challenges:
(1) \textit{How can FC re-estimation variability be probed from a single rs-fMRI scan?}
A full-scan FC graph provides only one finite-sample estimate. Probing this variability involves varying the temporal sample composition while retaining local temporal dependence and cross-ROI temporal correspondence~\cite{bellec2008bootstrap,kudela2017assessing}.
(2) \textit{How can changes in predictive evidence under FC re-estimation be represented and compared consistently?}
A GNN with graph-level readout aggregates information from multiple connectivity patterns into a single predictive representation~\cite{cui2022interpretable,han2026rethinking}. Comparing full-scan and re-estimated representations as a whole does not directly reveal changes in individual patterns' predictive contributions. Resolving these changes requires consistent pattern correspondence across subjects, sites, and FC estimates.
(3) \textit{How can within-subject re-estimation variability be interpreted for cross-site alignment?}
FC re-estimation probes changes within subjects, whereas cross-site alignment compares representation distributions across sites. The same contribution change can have different implications depending on within-class subject variability and class separation, which may vary across sites~\cite{yamashita2025computational}. Thus, contribution changes need to be interpreted in the task context of both sites to assess a pattern's suitability for alignment.

To address these challenges, we propose \textbf{B}rain Network \textbf{R}e-estimation-\textbf{I}nformed \textbf{O}OD Learning (\method{}), a unified framework using within-scan FC re-estimation to guide cross-site alignment. It comprises three components:
(i) \textit{Correspondence-Preserving Connectome Factorization} constructs FC re-estimates from each BOLD sequence and maps full-scan and re-estimated graphs into consistently indexed connectome factors, enabling comparisons of predictive contributions.
(ii) \textit{Task-Calibrated Re-estimation Qualification} evaluates each factor's predictive relevance and assesses its re-estimation support for each site and class by calibrating contribution changes under FC re-estimation against within-class subject variability and class separation.
(iii) \textit{Qualification-Guided Cross-Site Alignment} combines predictive relevance with support from both sites to form factor qualifications for each source-site pair and class. Normalized qualifications determine relative factor weights, while total qualification modulates class-conditional alignment strength between full-scan representations. Leave-one-site-out (LOSO) experiments on real-world datasets (ABIDE, REST-meta-MDD, SRPBS, and ABCD) show that \method{} consistently outperforms competitive baselines, achieving relative accuracy gains of up to 3.8\% and maintaining its advantage under the alternative brain parcellation.

Our contributions are summarized as follows:
(1) We revisit cross-site brain network OOD generalization by considering within-scan re-estimation variability, consistent comparisons of pattern-level predictive contributions, and interpretation of re-estimation variability for cross-site alignment.
(2) We propose \method{}, which learns consistently indexed connectome factors and combines predictive relevance with task-calibrated re-estimation support to determine factor weights and class-conditional alignment strength.
(3) We conduct leave-one-site-out experiments on four real-world datasets, showing consistent gains over competitive baselines, including under the alternative brain parcellation.

%% file: main/2_related.tex
\vspace{-0.4cm}
\section{Related Work}
\vspace{-0.1cm}
\textbf{Cross-Site OOD Generalization in Brain Networks.}
Graph OOD generalization aims to learn robust predictors under environment shifts through invariant subgraph learning, separation of invariant and spurious information, and representation alignment~\cite{li2025out,chen2022learning,wang2024dissecting}. In multi-site neuroimaging, harmonization methods reduce site effects in feature distributions and covariance~\cite{yu2018statistical,chen2022mitigating,yangbrain}, while related analyses characterize how subject, scanner, and acquisition effects influence predictive biomarkers~\cite{yamashita2025computational,wang2026sgac}. To learn representations that generalize to unseen sites, brain network OOD methods employ transferable explanations, feature selection, and invariant subgraph learning~\cite{qiu2025metaexplainer}. Other approaches use causal augmentation~\cite{yu2025causal} or combine site-aware deconfounding with transferable connectivity dynamics~\cite{wang2026brain,wang2026usbd}. \method{} uses within-scan FC re-estimation support to qualify predictive factors for cross-site alignment.

\textbf{Functional Connectome Estimation and Reliability.}
Functional connectomes are estimated from finite BOLD sequences, with scan duration, temporal sampling, and autocorrelation affecting their reliability and statistical properties~\cite{birn2013effect,huotari2019sampling,xiang2026agents,xiang2026safety}. Test--retest studies examine how acquisition and processing choices affect FC reliability~\cite{noble2019decade}, while methodological comparisons show how connectivity estimators shape network structure~\cite{smith2011network}. Statistical approaches include empirical Bayes normalization of connectivity metrics and bootstrap procedures for inference and uncertainty quantification~\cite{chen2015empirical,bellec2008bootstrap,kudela2017assessing}. Recent work develops learnable connectome representations from BOLD signals~\cite{yangbrain}, uses temporal segments for contrastive representation learning~\cite{lamprou2026varconet}, and explicitly models connectivity uncertainty~\cite{pakravan2026uncertainty}. \method{} assesses task-calibrated re-estimation support through changes in each connectome factor's predictive contribution.

\begin{figure}
    \centering
    \includegraphics[width=0.98\linewidth]{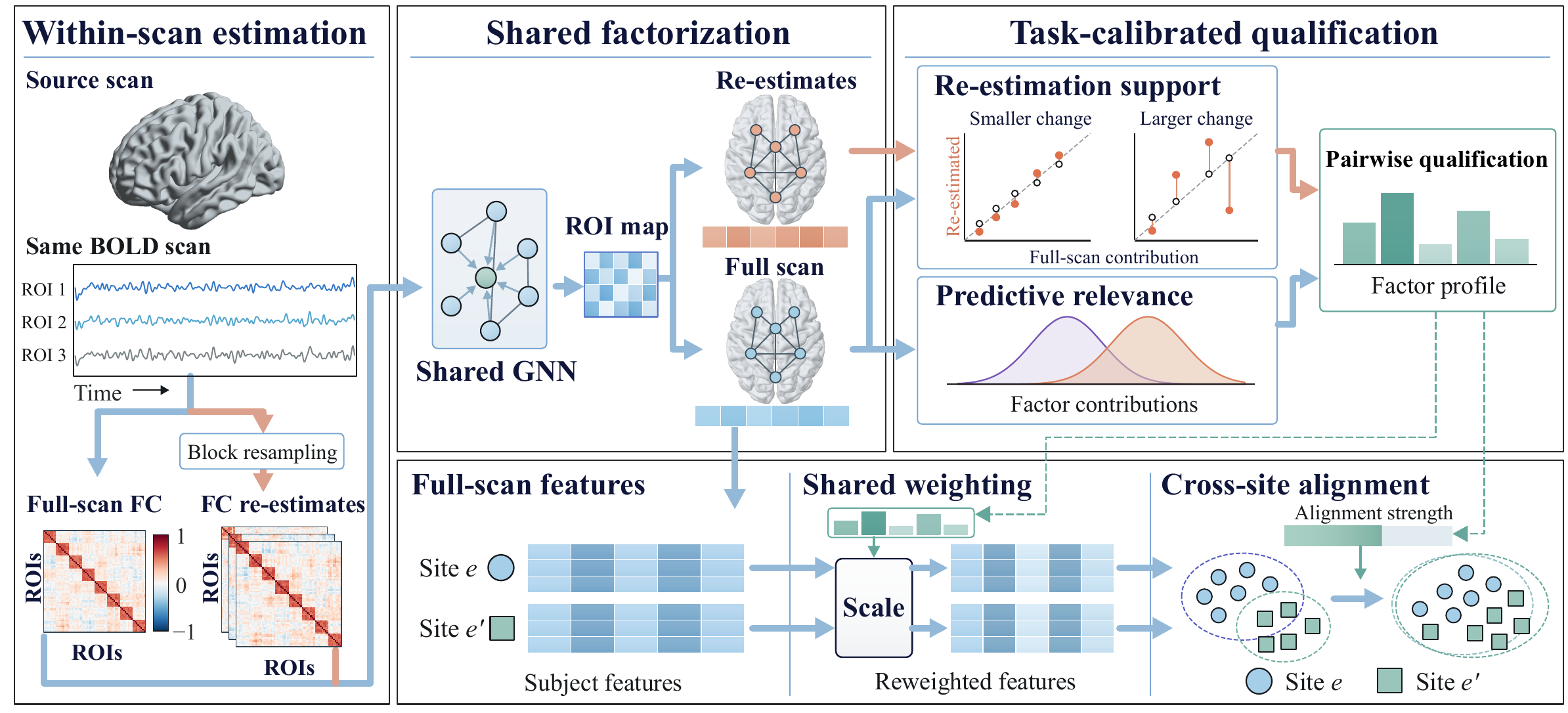}
    \vspace{-0.2cm}
    \caption{Overview of \method{}. Block resampling of the same BOLD scan produces FC re-estimates. Shared factorization uses a shared GNN and ROI map to obtain consistently indexed factors from full-scan and re-estimated graphs. Predictive relevance and task-calibrated re-estimation support from both sites form pairwise qualifications for each site pair and class, which control shared weighting and alignment strength for cross-site alignment of reweighted full-scan features.}
    \label{fig:framework}
    \vspace{-0.5cm}
\end{figure}

%% file: main/4_method.tex
\vspace{-0.1cm}
\section{Methodology}
\vspace{-0.1cm}

\textbf{Problem Setup.}
We study cross-site OOD generalization in brain network analysis with imaging sites as environments. Let $\mathcal E_{\mathrm{tr}}=\{1,\ldots,E\}$, with $E\geq2$, and $\mathcal E_{\mathrm{te}}$ denote disjoint source and unseen test environments. Each environment $e$ has a joint distribution $\mathbb P^e$ over BOLD sequences and labels. For each source site $e\in\mathcal E_{\mathrm{tr}}$, we observe $\mathcal D^e=\{(\bm X_i^e,Y_i^e)\}_{i=1}^{N_e}\sim(\mathbb P^e)^{N_e}$, where $\bm X_i^e\in\mathbb R^{T_i^e\times P}$ is the preprocessed BOLD sequence with $T_i^e$ valid time points and $P$ ROIs, and $Y_i^e\in\{0,1\}$ is the label. Sites share the label space, preprocessing pipeline, parcellation, and FC estimator, while $\mathbb P^e$ may differ. The shared FC estimator $\Psi$ yields full-scan graphs $\bm A_i^e=\Psi(\bm X_i^e)\in\mathbb R^{P\times P}$, whose entries encode FC between ROIs. Let $f_{\theta}$ denote a predictor operating on FC graphs. We train it using only source data and evaluate its average risk on unseen sites:
\begin{equation}
    \mathcal R_{\mathrm{OOD}}(f_{\theta})
    =
    \frac{1}{|\mathcal E_{\mathrm{te}}|}
    \sum_{e\in\mathcal E_{\mathrm{te}}}
    \mathbb E_{(\bm X,Y)\sim\mathbb P^{e}}
    \left[
        \ell\!\left(
            f_{\theta}\!\left(\Psi(\bm X)\right),
            Y
        \right)
    \right],
\end{equation}
where $\ell$ denotes the prediction loss.

\textbf{Overview.} As shown in Fig.~\ref{fig:framework}, \method{} uses within-scan FC re-estimation to guide cross-site alignment through three components:
(i) \textbf{Correspondence-Preserving Connectome Factorization} constructs FC re-estimates from each BOLD sequence and maps full-scan and re-estimated graphs into consistently indexed connectome factors using a shared signed GNN and shared ROI assignment.
(ii) \textbf{Task-Calibrated Re-estimation Qualification} evaluates source-balanced predictive relevance and assesses re-estimation support for each site and class by calibrating changes in factor-wise predictive contributions against within-class subject variability and between-class separation.
(iii) \textbf{Qualification-Guided Cross-Site Alignment} combines predictive relevance with support from both sites to form qualifications for each source-site pair and class. These qualifications determine shared factor weights and the overall strength of class-conditional alignment between reweighted full-scan representations.

\vspace{-0.1cm}
\subsection{Correspondence-Preserving Connectome Factorization}
\label{sec:factorization}

FC estimated from a finite BOLD sequence can vary with the temporal samples used~\cite{kudela2017assessing}, potentially changing the predictive contributions of connectivity patterns. To compare these changes at the pattern level, we construct within-scan FC re-estimates and map them together with full-scan graphs into consistently indexed connectome factors.

\textbf{(i) Within-Scan Re-estimation and Shared Connectome Encoding.}
BOLD signals exhibit temporal dependence and share a common time axis across ROIs. With $\bm A_i^{e,(0)}=\bm A_i^e$ denoting the full-scan graph, we construct $K$ FC re-estimates by jointly resampling temporal blocks:
\begin{equation}
\begin{aligned}
    \bm X_i^{e,(k)}
    &=
    \mathcal R_k\!\left(\bm X_i^e\right), \quad
    \bm A_i^{e,(k)}
    =
    \Psi\!\left(\bm X_i^{e,(k)}\right),
    \quad k=1,\ldots,K,
\end{aligned}
\label{eq:fc_reestimation}
\end{equation}
where $\mathcal R_k$ denotes joint circular moving-block resampling with a source-estimated block length $\ell_b$~\cite{shao1993bootstrapping,bellec2008bootstrap}. Resampling retains local temporal dependence within sampled blocks and preserves cross-ROI temporal correspondence through shared sampling indices. 

To encode all FC estimates over the same ROI pairs while retaining their estimate-specific weights, we construct a fixed sparse propagation support $\bm S\in\{0,1\}^{P\times P}$ solely from source full-scan graphs. For $u\neq v$, we compute
\begin{equation}
    \overline a_{uv}
    =
    \operatorname{median}_{e\in\mathcal E_{\mathrm{tr}}}
    \operatorname{median}_{1\leq i\leq N_e}
    \left|
        \left[\bm A_i^{e,(0)}\right]_{uv}
    \right|,
    \quad
    \bm S
    =
    \operatorname{SymTop}_{m}\!\left(\overline{\bm A}\right),
\label{eq:source_support}
\end{equation}
where $\overline{\bm A}=[\overline a_{uv}]$ has zero diagonal and $m=\lceil\rho_g(P-1)\rceil$, with $\rho_g\in(0,1]$. The operator $\operatorname{SymTop}_{m}$ selects the $m$ largest off-diagonal entries per row, sets $S_{uv}=S_{vu}=1$ if either ROI selects the other, and sets all remaining entries to zero. For $k=0,\ldots,K$, we retain each FC estimate's weights on this support and separate positive weights from negative-weight magnitudes:
\begin{equation}
\begin{aligned}
    \bm G_i^{e,(k)}
    =
    \bm S\odot\bm A_i^{e,(k)}, \quad
    \bm G_i^{e,(k),+}
    =
    \left[\bm G_i^{e,(k)}\right]_{+},
    \quad
    \bm G_i^{e,(k),-}
    =
    \left[-\bm G_i^{e,(k)}\right]_{+},
\end{aligned}
\label{eq:signed_graphs}
\end{equation}
where $\odot$ denotes elementwise multiplication and $[\cdot]_{+}$ takes the elementwise positive part.

The shared parcellation provides consistent ROI labels across graphs, while local connectivity on $\bm S$ varies across subjects and FC estimates. Let $\bm e_u^{\mathrm{ROI}}\in\mathbb R^{d_e}$ denote the learnable identity embedding of ROI $u$. We summarize its local connectivity using the descriptor $\bm c_{i,u}^{e,(k)}$:
\begin{equation}
    \bm c_{i,u}^{e,(k)}
    =
    \left[
        \frac{1}{d_u}\sum_v G_{i,uv}^{e,(k),+},
        \frac{1}{d_u}\sum_v G_{i,uv}^{e,(k),-},
        \sqrt{
            \frac{1}{d_u}
            \sum_v
            \left(G_{i,uv}^{e,(k)}\right)^2
        }
    \right]^{\top},
    \quad
    d_u=\sum_v S_{uv}.
\label{eq:node_connectivity_descriptor}
\end{equation}
Here, $G_{i,uv}^{e,(k)}$ and $G_{i,uv}^{e,(k),\pm}$ are entries of the corresponding adjacency matrices, and $d_u$ is the degree of ROI $u$ in $\bm S$. We combine ROI identity and connectivity to initialize node features, then encode them with the shared signed GNN:
\begin{equation}
\begin{aligned}
    \bm h_{i,u}^{e,(k),[0]}
    =
    \phi_{\mathrm{in}}
    \left(
        \bm e_u^{\mathrm{ROI}},
        \bm c_{i,u}^{e,(k)}
    \right), \quad
    \bm H_i^{e,(k)}
    =
    \mathcal G_{\theta_g}
    \left(
        \bm H_i^{e,(k),[0]},
        \bm G_i^{e,(k),+},
        \bm G_i^{e,(k),-}
    \right),
\end{aligned}
\label{eq:shared_signed_encoder}
\end{equation}
where $\phi_{\mathrm{in}}:\mathbb R^{d_e}\times\mathbb R^3\rightarrow\mathbb R^{d_h}$ initializes node features, $\bm H_i^{e,(k),[0]}\in\mathbb R^{P\times d_h}$ stacks them row-wise, and $\mathcal G_{\theta_g}$ is the shared signed GNN. Its output $\bm H_i^{e,(k)}\in\mathbb R^{P\times d_h}$ has $u$-th row $(\bm h_{i,u}^{e,(k)})^{\top}$.

\textbf{(ii) Shared Connectome Factorization and Factor-Wise Scoring.}
Predictive connectivity patterns may involve multiple ROIs~\cite{li2021braingnn}. We therefore pool the encoded ROI representations into connectome factors using shared ROI weighting profiles, so that each factor can be compared consistently across subjects, sites, and FC estimates.

Let $\bm E^{\mathrm{ROI}}=[\bm e_1^{\mathrm{ROI}},\ldots,\bm e_P^{\mathrm{ROI}}]^{\top}\in\mathbb R^{P\times d_e}$ stack the shared ROI identity embeddings, and let $\bm Q=[\bm q_1,\ldots,\bm q_J]^{\top}$ denote $J$ learnable factor queries. The shared assignment $\bm\Pi\in\mathbb R_{+}^{J\times P}$ is defined as
\begin{equation}
\begin{aligned}
    \bm\Pi
    =
    \mathcal A_{\omega}
    \left(
        \bm Q,
        \bm E^{\mathrm{ROI}}
    \right)
    , \quad
    \bm\Pi\bm 1_P
    =
    \frac{1}{J}\bm 1_J,
    \quad
    \bm\Pi^{\top}\bm 1_J
    =
    \frac{1}{P}\bm 1_P,
\end{aligned}
\label{eq:shared_assignment}
\end{equation}
where $\mathcal A_{\omega}$ is a learnable mapping that computes a balanced assignment from similarities between factor queries and ROI identity embeddings, and $\bm1_n$ denotes the length-$n$ all-ones vector. Since $\bm\Pi$ depends only on shared queries and ROI identity embeddings, each factor uses the same ROI weighting profile across subjects, sites, and FC estimates. Pooling the encoded ROI representations under these profiles gives the connectome factors:
\begin{equation}
    \bm r_{i,j}^{e,(k)}
    =
    \mathcal F_{\phi}
    \left(
        J
        \sum_{u=1}^{P}
        \Pi_{j,u}
        \bm h_{i,u}^{e,(k)}
    \right)
    \in\mathbb S^{d_t-1},
    \quad
    j=1,\ldots,J,
\label{eq:connectome_factors}
\end{equation}
where $\mathcal F_{\phi}:\mathbb R^{d_h}\rightarrow\mathbb S^{d_t-1}$ is a shared mapping with unit-norm output. 

To retain the individual factor representations, we concatenate them in a fixed order:
\begin{equation}
    \bm s_i^{e,(k)}
    =
    \frac{1}{\sqrt J}
    \left[
        \bm r_{i,1}^{e,(k)}
        \mathbin\Vert
        \cdots
        \mathbin\Vert
        \bm r_{i,J}^{e,(k)}
    \right].
\label{eq:factorized_representation}
\end{equation}
Here, $\Vert$ denotes concatenation. To attribute prediction changes to individual factors, we use a linear head on the concatenated representation, yielding the prediction logit and additive factor scores:
\begin{equation}
\begin{aligned}
    z_i^{e,(k)}
    =
    \frac{1}{\sqrt J}
    \sum_{j=1}^{J}
    \bm w_j^{\top}\bm r_{i,j}^{e,(k)}
    +b, \quad
    a_{i,j}^{e,(k)}
    =
    \bm w_j^{\top}\bm r_{i,j}^{e,(k)},
\end{aligned}
\label{eq:factorized_task_prediction}
\end{equation}
where $\bm w_j$ is the classifier block for factor $j$, $b$ is the bias, and $a_{i,j}^{e,(k)}$ is the factor's additive contribution to the logit before the common $1/\sqrt J$ scaling.

\subsection{Task-Calibrated Re-estimation Qualification}
\label{sec:qualification}

The factor-wise contributions may distinguish classes in full-scan graphs while varying across FC re-estimates of the same subject. Within-class subject variability and between-class separation provide a reference for assessing these changes at each site. We therefore evaluate source-balanced predictive relevance and task-calibrated re-estimation support.

To assess predictive relevance, we measure class separation in full-scan contributions relative to within-class variability. For source site $e$ and class $c\in\{0,1\}$, let $\mathcal I_{e,c}=\{i\in\{1,\ldots,N_e\}:Y_i^e=c\}$ and $N_{e,c}=|\mathcal I_{e,c}|$, with both classes represented at each source site. Each qualification evaluation computes the full-scan and re-estimated contributions of all subjects in $\mathcal I_{e,c}$ under the same fixed model snapshot. From the full-scan contributions, we define the empirical mean, within-class variance, and class contrast as
\begin{equation}
\begin{aligned}
    \mu_{e,c,j}
    = \frac{1}{N_{e,c}}\sum_{i\in\mathcal I_{e,c}}a_{i,j}^{e,(0)},
    \quad
    V_{e,c,j}^{\mathrm{sub}}
    = \frac{1}{N_{e,c}}\sum_{i\in\mathcal I_{e,c}}
      \left(a_{i,j}^{e,(0)}-\mu_{e,c,j}\right)^2,
    \quad
    \Delta_{e,j}
    = \mu_{e,1,j}-\mu_{e,0,j}.
\end{aligned}
\label{eq:factor_task_statistics}
\end{equation}
We aggregate these within-site variances and signed class contrasts with equal site weights so that larger cohorts do not dominate the relevance estimate:
\begin{equation}
\begin{aligned}
    V_{c,j}^{\mathrm{pool}}
    = \frac{1}{E}\sum_{e\in\mathcal E_{\mathrm{tr}}}V_{e,c,j}^{\mathrm{sub}},
    \quad
    \Delta_j
    = \frac{1}{E}\sum_{e\in\mathcal E_{\mathrm{tr}}}\Delta_{e,j}.
\end{aligned}
\label{eq:source_balanced_statistics}
\end{equation}
Opposing class contrasts offset one another before $\Delta_j$ is squared. We measure relative class separation by normalizing the squared aggregate contrast by within-class variability. To account for contribution scale, we further weight this ratio by the relative squared classifier-block norm $\beta_j$, since unit-norm factors satisfy $|a_{i,j}^{e,(k)}|\leq\|\bm w_j\|_2$. The resulting relevance scores are normalized across factors:
\begin{equation}
    \beta_j
    = \frac{\|\bm w_j\|_2^2}
      {\sum_{j'=1}^{J}\|\bm w_{j'}\|_2^2+\epsilon_w},
    \quad
    R_j^{\mathrm{task}}
    = \frac{\beta_j\Delta_j^2}
      {V_{0,j}^{\mathrm{pool}}+V_{1,j}^{\mathrm{pool}}+\epsilon_s},
    \quad
    \pi_j^{\mathrm{task}}
    = \frac{R_j^{\mathrm{task}}}
      {\sum_{j'=1}^{J}R_{j'}^{\mathrm{task}}},
\label{eq:task_relevance}
\end{equation}
where $\epsilon_w,\epsilon_s>0$ are numerical stabilizers. For $\sum_{j=1}^{J}R_j^{\mathrm{task}}>0$, the normalized scores $\pi_j^{\mathrm{task}}$ form a probability distribution over factors.

We next assess within-scan re-estimation changes by averaging the squared differences between each subject's re-estimated and full-scan contributions over subjects and re-estimates within each site and class:
\begin{equation}
    D_{e,c,j}^{\mathrm{re}}
    = \frac{1}{N_{e,c}K}
      \sum_{i\in\mathcal I_{e,c}}\sum_{k=1}^{K}
      \left(a_{i,j}^{e,(k)}-a_{i,j}^{e,(0)}\right)^2.
\label{eq:reestimation_variation}
\end{equation}
To interpret these deviations at each site, we combine full-scan within-class variability and class separation into a local reference scale:
\begin{equation}
\begin{aligned}
    M_{e,j}^{\mathrm{task}}
    = \frac{1}{4}\Delta_{e,j}^{2},
    \quad
    T_{e,c,j}
    = V_{e,c,j}^{\mathrm{sub}}+M_{e,j}^{\mathrm{task}},
\end{aligned}
\label{eq:task_calibration_scale}
\end{equation}
where $M_{e,j}^{\mathrm{task}}$ is the squared half-separation between the two class means at site $e$. Consequently, $T_{e,c,j}$ equals the empirical mean squared distance of full-scan contributions in class $c$ from the midpoint of these means. We use this reference scale to define re-estimation support for each site and class:
\begin{equation}
    q_{e,c,j}^{\mathrm{re}}
    = \frac{T_{e,c,j}+\epsilon_q}
      {T_{e,c,j}+D_{e,c,j}^{\mathrm{re}}+\epsilon_q}
    \in(0,1],
\label{eq:reestimation_support}
\end{equation}
where $\epsilon_q>0$ is a numerical stabilizer. Higher support indicates smaller mean squared contribution changes relative to the local full-scan reference scale.

\tcolorboxenvironment{proposition}{
  breakable,
  colback=blue!5,
  colframe=blue!50!black,
  boxrule=0.4pt,
  arc=3pt,
  left=3pt, right=3pt,
  top=2pt, bottom=2pt,
  grow to left by=1.4pt,
  grow to right by=1.4pt,
  before skip=3pt, after skip=9pt
}
\begin{proposition}[Class-Contrast Stability under FC Re-estimation]
\label{prop:reestimation_contrast_stability}
Fix the model snapshot, a source site $e$, and a factor $j$, with $N_{e,0},N_{e,1}>0$ and $K\geq1$. Let $\Delta_{e,j}^{\mathrm{re}}$ be the class contrast of contributions averaged over subjects and FC re-estimates, and let $\kappa_{e,c,j}=\Delta_{e,j}^2/\big(4(V_{e,c,j}^{\mathrm{sub}}+\epsilon_q)\big)$ be the full-scan separation ratio of class $c$. Then
\begin{equation}
    \left|\Delta_{e,j}^{\mathrm{re}}-\Delta_{e,j}\right|
    \leq \sum_{c\in\{0,1\}}
    \sqrt{\left(T_{e,c,j}+\epsilon_q\right)\left(\frac{1}{q_{e,c,j}^{\mathrm{re}}}-1\right)},
\label{eq:class_contrast_stability}
\end{equation}
and $\Delta_{e,j}^{\mathrm{re}}$ has the same sign as $\Delta_{e,j}$ whenever
$q_{e,c,j}^{\mathrm{re}}>(1+\kappa_{e,c,j})/(1+2\kappa_{e,c,j})$ for both $c\in\{0,1\}$.
\end{proposition}

Proposition~\ref{prop:reestimation_contrast_stability} bounds the change of class contrast under FC re-estimation by the re-estimation variation, and shows that the support needed to preserve the sign of the contrast decreases with full-scan class separation, from one for weakly separated factors to one half for strongly separated ones. A factor with high relevance but low support, or low relevance but high support, can thus flip its class contrast under re-estimation, so we qualify factors by assessing the two criteria jointly.

\subsection{Qualification-Guided Cross-Site Alignment}
\label{sec:alignment}

Predictive relevance combines class separation with contribution scale across source sites, whereas re-estimation support varies by site and class. For class-conditional alignment, we combine predictive relevance with support from both sites to determine relative factor weights and alignment strength.

For two distinct source sites $e,e'\in\mathcal E_{\mathrm{tr}}$, class $c\in\{0,1\}$, and factor $j$, we combine re-estimation support from both sites and weight it by predictive relevance:
\begin{equation}
    q_{e,e',c,j}^{\mathrm{pair}}
    =
    \sqrt{
        q_{e,c,j}^{\mathrm{re}}
        q_{e',c,j}^{\mathrm{re}}
    },
    \quad
    \alpha_{e,e',c,j}
    =
    \pi_j^{\mathrm{task}}
    q_{e,e',c,j}^{\mathrm{pair}}.
\label{eq:pairwise_factor_qualification}
\end{equation}
The geometric mean treats the two sites symmetrically and decreases when support at either site decreases. We normalize the qualifications to determine relative factor weights and retain their sum as the overall qualification for scaling the alignment loss:
\begin{equation}
\begin{aligned}
    g_{e,e',c}
    =
    \sum_{j=1}^{J}
    \alpha_{e,e',c,j},
    \quad
    \overline{\alpha}_{e,e',c,j}
    =
    \frac{
        \alpha_{e,e',c,j}
    }{
        g_{e,e',c}
    },
    \quad
    \sum_{j=1}^{J}
    \overline{\alpha}_{e,e',c,j}
    =
    1,
\end{aligned}
\label{eq:pairwise_qualification_profile}
\end{equation}
where $0<g_{e,e',c}\leq1$. 

\begin{proposition}[Variational Form of Pairwise Qualification]
\label{prop:pairwise_qualification_variational}
For a probability vector $\bm\pi$ over $J$ factors and supports $q_j\in(0,1]$, define $\alpha_j=\pi_jq_j$, $g=\sum_{j=1}^{J}\alpha_j$, and $\overline\alpha_j=\alpha_j/g$. Let $\mathcal V_{\pi}$ be the set of probability vectors $\bm v$ with $v_j=0$ whenever $\pi_j=0$, and define
\begin{equation}
    \mathcal J(\bm v)
    = \operatorname{KL}\!\left(\bm v\,\middle\|\,\bm\pi\right)
    + \sum_{j=1}^{J}v_j\log\frac{1}{q_j}.
\label{eq:pairwise_qualification_variational_objective}
\end{equation}
Then $\overline{\bm\alpha}$ is the unique minimizer of $\mathcal J$ over $\mathcal V_{\pi}$, with minimum value $-\log g$, and
\begin{equation}
    \overline\alpha_j
    = \frac{\partial\log g}{\partial\log q_j},
    \quad
    -\log g
    \leq \sum_{j=1}^{J}\pi_j\log\frac{1}{q_j},
\label{eq:pairwise_qualification_variational_solution}
\end{equation}
where equality holds if and only if $q_j$ is constant over the factors with $\pi_j>0$.
\end{proposition}

With $\bm\pi=\bm\pi^{\mathrm{task}}$ and $q_j=q^{\mathrm{pair}}_{e,e',c,j}$, Proposition~\ref{prop:pairwise_qualification_variational} recovers $\overline{\bm\alpha}_{e,e',c}$ and $g_{e,e',c}$. The normalized qualifications are the weights closest to predictive relevance under a penalty for weak support, each equal to the sensitivity of alignment strength to its factor's support. Weights and strength thus derive from one objective, and reweighting lowers the support penalty below relevance-only weighting.

We use the normalized qualifications derived from predictive contributions to weight full-scan factor representations identically at both sites. For $d\in\{e,e'\}$ and subject $i\in\mathcal I_{d,c}$, the weighted representation is
\begin{equation}
    \bm\varphi_{i\mid e,e',c}^{Q,d}
    =
    \frac{1}{\sqrt{2}}
    \left[
        \sqrt{\overline{\alpha}_{e,e',c,1}}\bm r_{i,1}^{d,(0)}
        \mathbin\Vert \cdots \mathbin\Vert
        \sqrt{\overline{\alpha}_{e,e',c,J}}\bm r_{i,J}^{d,(0)}
    \right].
\label{eq:pair_specific_representation}
\end{equation}
For subjects $i\in\mathcal I_{e,c}$ and $h\in\mathcal I_{e',c}$, their squared distance in the common weighted space is
\begin{equation}
\begin{aligned}
    \left\|
        \bm\varphi_{i\mid e,e',c}^{Q,e}
        -
        \bm\varphi_{h\mid e,e',c}^{Q,e'}
    \right\|_2^2
    =
    \sum_{j=1}^{J}\overline{\alpha}_{e,e',c,j}
    \left[
        1-\left(\bm r_{i,j}^{e,(0)}\right)^{\top}\bm r_{h,j}^{e',(0)}
    \right].
\end{aligned}
\label{eq:qualified_pairwise_distance}
\end{equation}
Each factor's cosine dissimilarity is weighted by its relative qualification. Since the factors have unit norm and the weights sum to one, the squared distance lies in $[0,2]$. We then construct the empirical class-conditional distributions in this weighted space. Given nonempty source subsets $\mathcal B_{e,c}\subseteq\mathcal I_{e,c}$ and $\mathcal B_{e',c}\subseteq\mathcal I_{e',c}$, we define
\begin{equation}
\begin{aligned}
    \widehat P_{e,c}^{Q[e,e']}
    &=
    \frac{1}{|\mathcal B_{e,c}|}
    \sum_{i\in\mathcal B_{e,c}}
    \delta_{\bm\varphi_{i\mid e,e',c}^{Q,e}},
    \quad
    \widehat P_{e',c}^{Q[e,e']}
    =
    \frac{1}{|\mathcal B_{e',c}|}
    \sum_{h\in\mathcal B_{e',c}}
    \delta_{\bm\varphi_{h\mid e,e',c}^{Q,e'}},
\end{aligned}
\label{eq:qualified_site_distributions}
\end{equation}
where $\delta_{\bm x}$ denotes a point mass at $\bm x$. We compare these distributions using the debiased Sinkhorn divergence $\mathcal S_{\varepsilon}$ with the squared Euclidean ground cost and entropic regularization parameter $\varepsilon>0$~\cite{cuturi2013sinkhorn,feydy2019interpolating}. Details of debiased Sinkhorn divergence $\mathcal S_{\varepsilon}$ are provided in Appendix~\ref{app:sinkhorn}.  Weighting each divergence by $g_{e,e',c}$ and averaging over source-site pairs and classes yields the alignment loss:
\begin{equation}
\begin{aligned}
    \mathcal L_{\mathrm{env}}^{Q}
    =
    \frac{2}{E(E-1)}
    \sum_{\substack{e,e'\in\mathcal E_{\mathrm{tr}}\\e<e'}}
    \frac{1}{2}
    \sum_{c\in\{0,1\}}
    g_{e,e',c}
    \mathcal S_{\varepsilon}
    \left(
        \widehat P_{e,c}^{Q[e,e']},
        \widehat P_{e',c}^{Q[e,e']}
    \right).
\end{aligned}
\label{eq:qualification_guided_alignment}
\end{equation}

\vspace{-0.4cm}
\subsection{Learning Objective}
\label{sec:objective}

For each full-scan graph, the predicted probability is $\widehat p_i^{e,(0)}=\sigma\!\left(z_i^{e,(0)}\right)$, where $\sigma$ denotes the sigmoid function. We define the classification loss with equal weights across source sites and classes:
\begin{equation}
    \mathcal L_{\mathrm{cls}}
    =
    \frac{1}{2E}
    \sum_{e\in\mathcal E_{\mathrm{tr}}}
    \sum_{c\in\{0,1\}}
    \frac{1}{N_{e,c}}
    \sum_{i\in\mathcal I_{e,c}}
    \ell_{\mathrm{BCE}}
    \left(
        \widehat p_i^{e,(0)},Y_i^e
    \right),
\label{eq:classification_loss}
\end{equation}
where $\ell_{\mathrm{BCE}}$ denotes binary cross-entropy. The overall objective is
\begin{equation}
    \mathcal L
    =
    \mathcal L_{\mathrm{cls}}
    +
    \lambda_{\mathrm{env}}\mathcal L_{\mathrm{env}}^{Q},
\label{eq:overall_objective}
\end{equation}
where $\lambda_{\mathrm{env}}\geq0$ controls the strength of cross-site alignment.

%% file: main/5_experiments.tex
\section{Experiments}
\textbf{Datasets.} To evaluate the effectiveness of \method{}, we conduct LOSO experiments on four real-world fMRI datasets: ABIDE~\cite{abide}, REST-meta-MDD~\cite{restmeatmdd}, SRPBS~\cite{srpbs}, and ABCD~\cite{abcd}. These datasets cover prediction tasks related to autism spectrum disorder (ASD), major depressive disorder (MDD), and attention-deficit/hyperactivity disorder (ADHD), with ABCD providing a developmental cohort. We additionally include ABIDE (CC200), constructed using the CC200 functional parcellation, to examine the effect of atlas choice. More details of these datasets are provided in Appendix~\ref{app:data}.

\input{table/OOD}

\textbf{Baselines.} We compare \method{} with competitive baselines as follows:
(1) General GNNs: GCN~\cite{kipf2016gcn}, GAT~\cite{petar2018gat}, and GIN~\cite{xu2018gin};
(2) General OOD methods: CORAL~\cite{sun2016coral} and IRM~\cite{arjovsky2019irm};
(3) Graph OOD methods: GSAT~\cite{miao2022gsat}, DisC~\cite{fan2022disc}, CEPG~\cite{wang2026cepg}, and DiSCO~\cite{sun2026disco};
(4) Brain network methods: BrainNetTF~\cite{kan2022brainnttf}, XG-GNN~\cite{qiu2024towards}, AGMGC~\cite{noman2025agmgc}, FC-HGNN~\cite{gu2025fchgnn}, BrainOOD~\cite{xu2025brainood}, DeCI~\cite{yu2026deci} and CORE~\cite{wang2026brain}. More details of baselines are provided in Appendix~\ref{app:baseline}.

\subsection{Performance Comparison}

We compare \method{} with baseline methods under the LOSO protocol on ABIDE, REST-meta-MDD, SRPBS, and ABCD in Table~\ref{tab:main_results_ood}. We observe that:
(1) General OOD methods do not consistently outperform general GNNs across datasets, suggesting that generic robustness objectives offer limited gains in cross-site brain network generalization.
(2) The best-performing brain network methods consistently outperform graph OOD methods, while the gains from other brain network models vary across datasets. This suggests that cross-site generalization benefits not only from modeling brain connectivity but also from learning task-relevant patterns that transfer across sites.
(3) \method{} consistently outperforms all baselines across datasets. These gains can be attributed to two complementary mechanisms: (i) shared connectome factorization enables consistent comparisons of predictive contributions across FC estimates, while task-calibrated qualification identifies task-relevant factors supported under re-estimation; (ii) qualification-guided alignment prioritizes factors supported by both sites and adjusts the overall alignment strength accordingly, limiting the influence of weakly supported factors on cross-site matching. We further evaluate performance on ABIDE using CC200~\citep{cc200} instead of AAL. Although baseline rankings vary between the two parcellations, \method{} outperforms all baselines on both metrics, indicating that its gains persist across parcellations.

\vspace{-0.1cm}
\subsection{Ablation Study}\label{sec:ablation}
\vspace{-0.1cm}

We evaluate four ablation variants: (1) \method{} w/o RS uses task relevance alone to weight factors and fixes the strength multiplier to one; (2) \method{} w/o TR replaces explicit task relevance weights with uniform weights while retaining task-calibrated re-estimation support; (3) \method{} w/o SM retains qualification-based factor weights but fixes the strength multiplier to one; and (4) \method{} w/o SF uses task relevance alone to weight factors while retaining qualification-based strength modulation. As shown in Fig.~\ref{fig:sensitivity}(a,b), removing RS reduces performance, indicating that assessing predictive contributions under FC re-estimation provides benefits beyond full-scan task relevance. Removing TR degrades performance, highlighting the value of prioritizing discriminative factors rather than relying on re-estimation support alone. The performance drops without SM and SF support the complementary roles of qualification in modulating alignment strength for each site pair and class and prioritizing supported factors in cross-site matching. More results are shown in Appendix~\ref{app:ablation}.

\begin{figure}[t]
    \centering
    \captionsetup[subfigure]{font=small, skip=3pt}

    \begin{subfigure}[b]{0.24\linewidth}
        \centering
        \parbox[c][0.70\linewidth][c]{\linewidth}{%
            \centering
            \includegraphics[
                width=\linewidth,
                height=0.67\linewidth,
                keepaspectratio,
                trim=10bp 46bp 3bp 0bp,
                clip
            ]{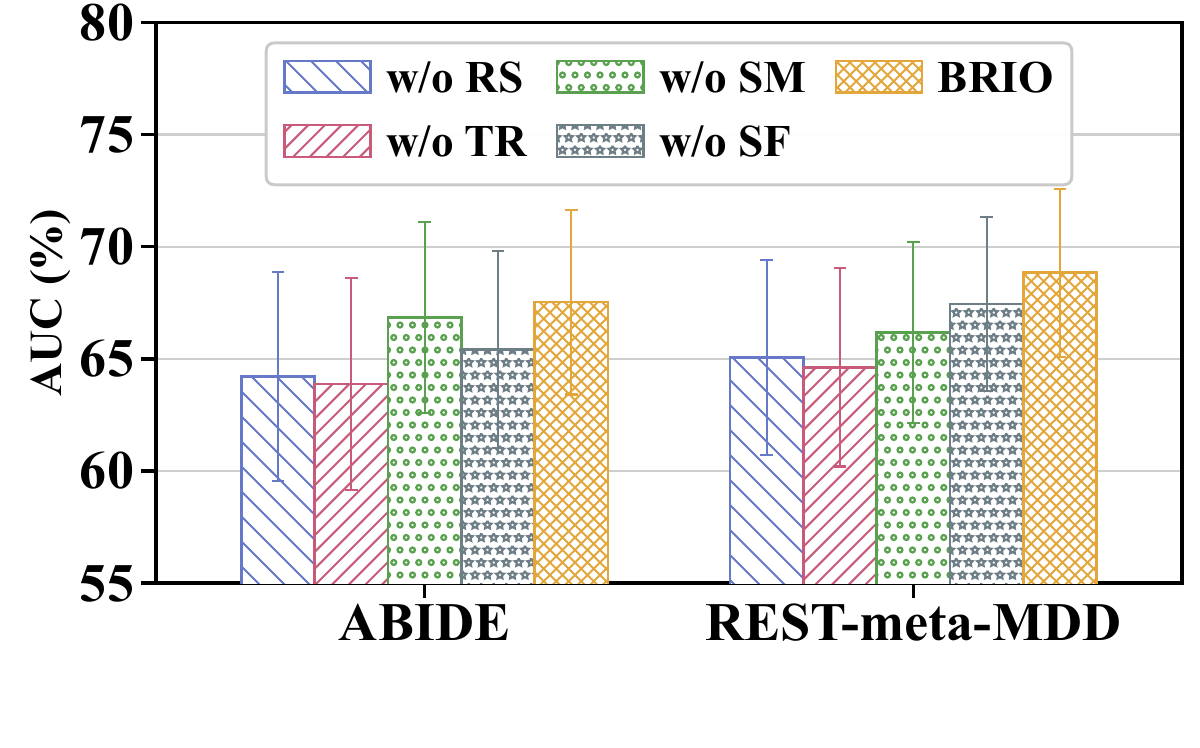}%
        }
        \caption{AUC}
    \end{subfigure}\hfill
    \begin{subfigure}[b]{0.24\linewidth}
        \centering
        \parbox[c][0.70\linewidth][c]{\linewidth}{%
            \centering
            \includegraphics[
                width=\linewidth,
                height=0.67\linewidth,
                keepaspectratio,
                trim=10bp 46bp 3bp 0bp,
                clip
            ]{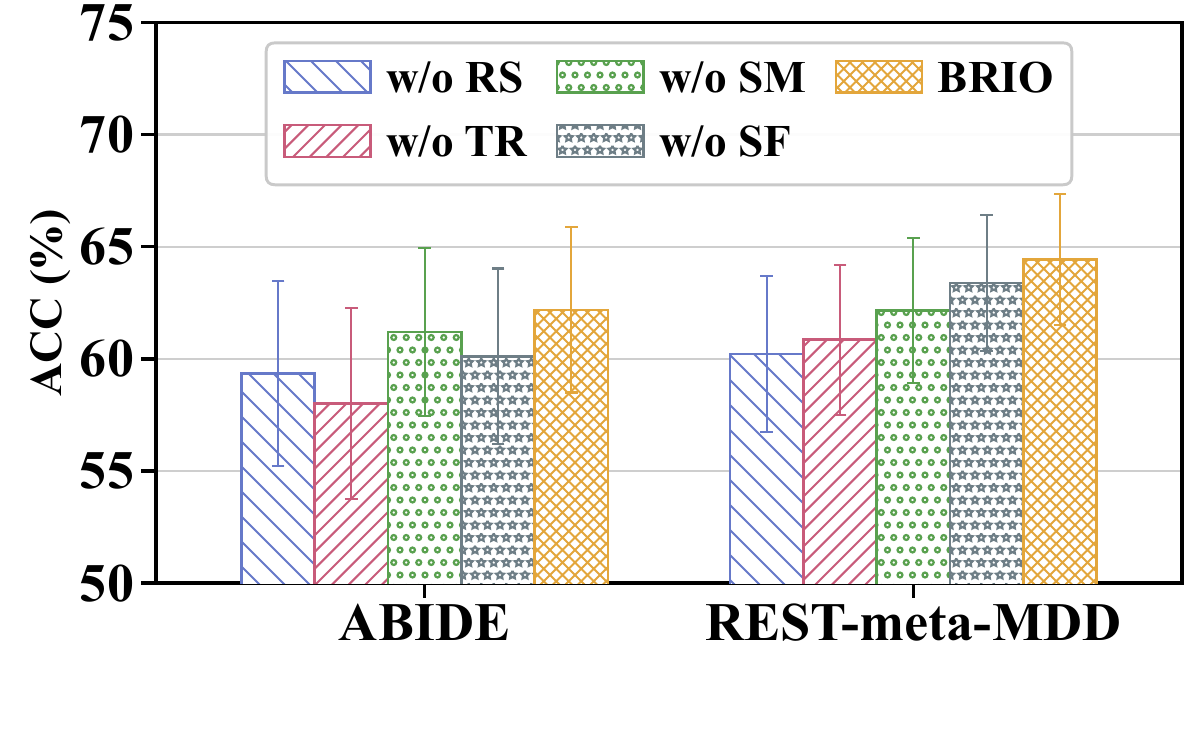}%
        }
        \caption{ACC}
    \end{subfigure}\hfill
    \begin{subfigure}[b]{0.24\linewidth}
        \centering
        \parbox[c][0.70\linewidth][c]{\linewidth}{%
            \centering
            \includegraphics[
                width=\linewidth,
                height=0.67\linewidth,
                keepaspectratio,
                trim=74bp 43bp 80bp 43bp,
                clip
            ]{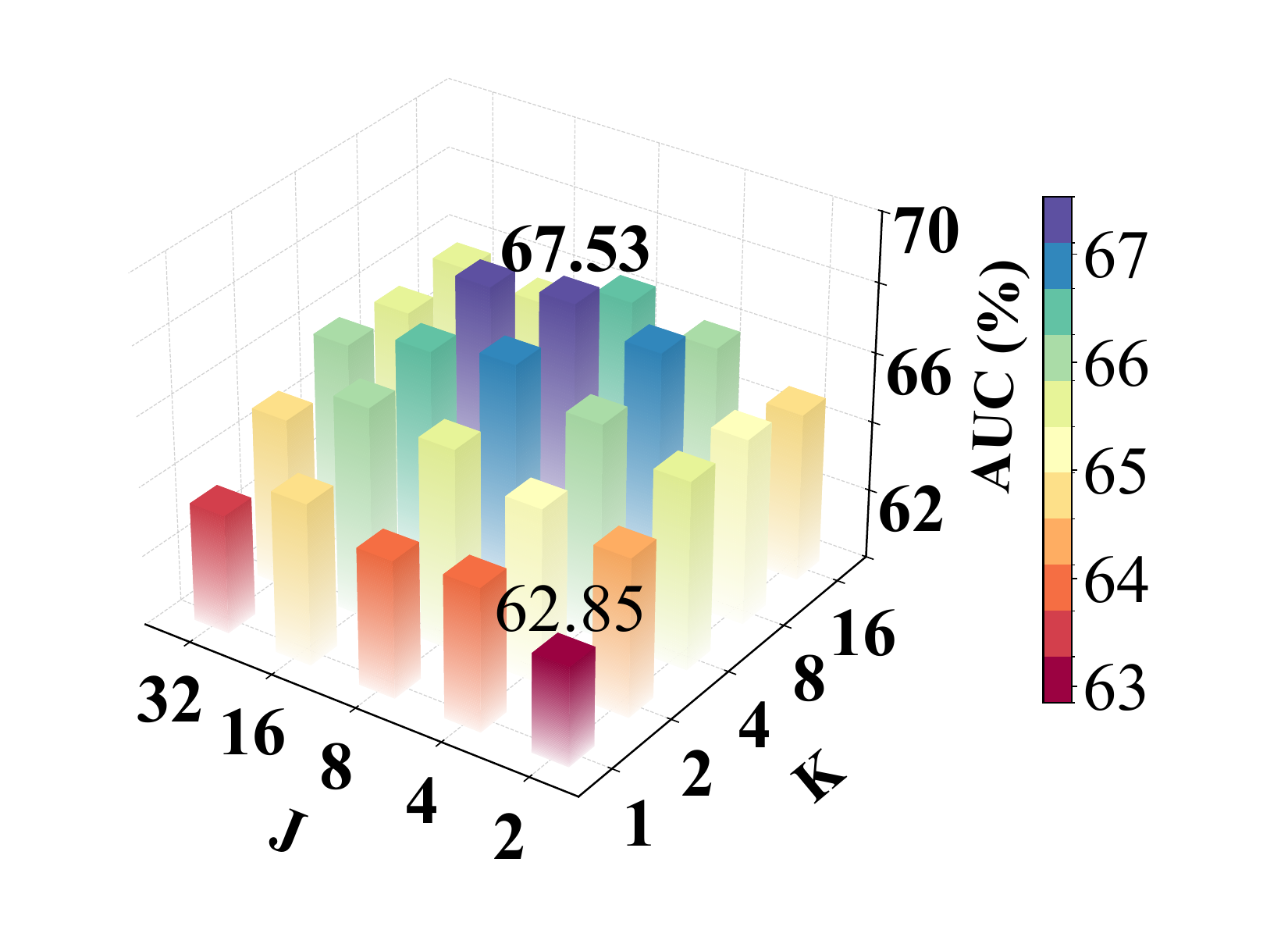}%
        }
        \caption{ABIDE}
    \end{subfigure}\hfill
    \begin{subfigure}[b]{0.24\linewidth}
        \centering
        \parbox[c][0.70\linewidth][c]{\linewidth}{%
            \centering
            \includegraphics[
                width=\linewidth,
                height=0.67\linewidth,
                keepaspectratio,
                trim=74bp 43bp 80bp 43bp,
                clip
            ]{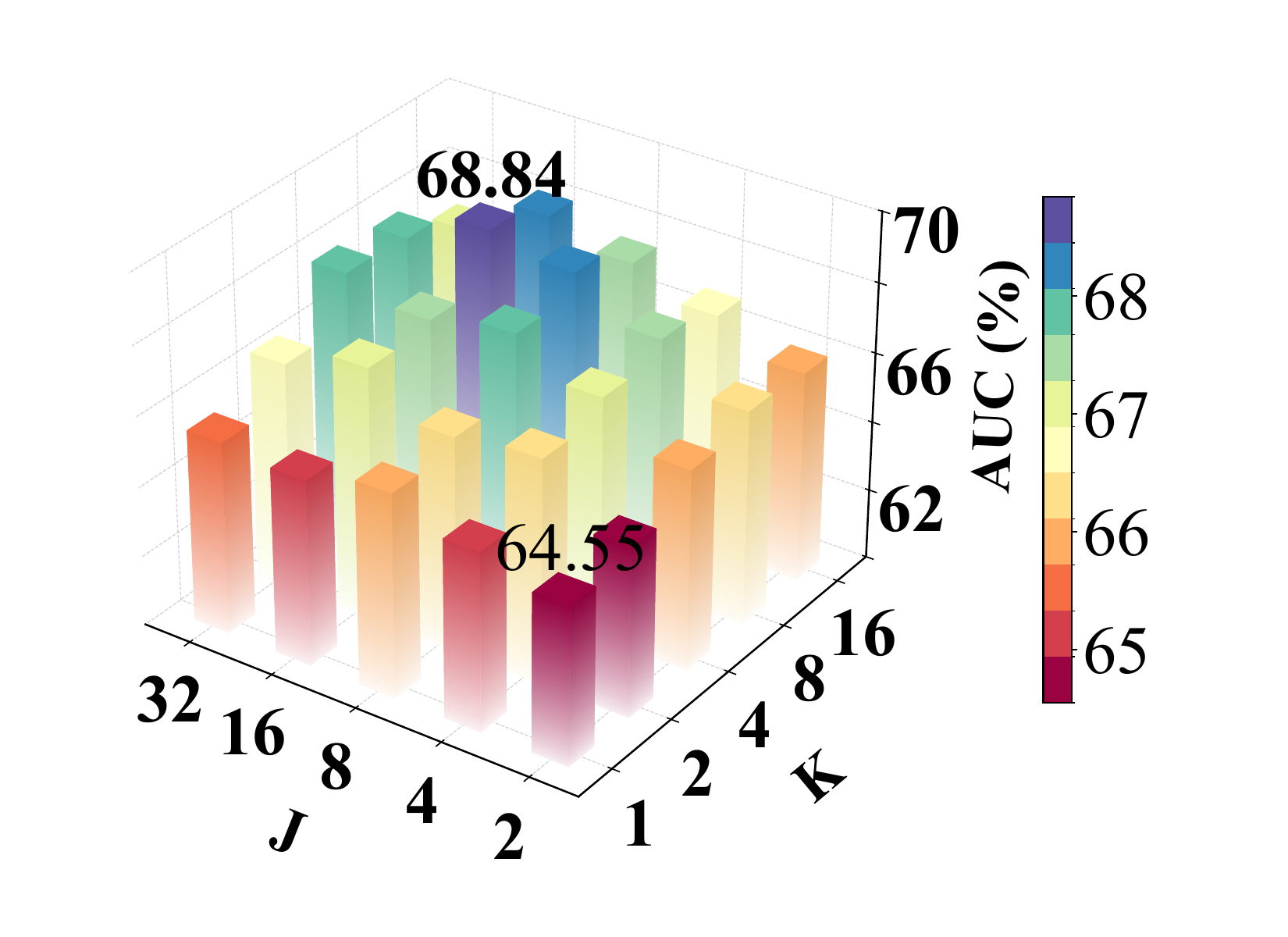}%
        }
        \caption{REST-meta-MDD}
    \end{subfigure}
    \vspace{-0.2cm}
    \caption{Ablation results in (a,b) and sensitivity results in (c,d)
    on ABIDE and REST-meta-MDD.}
    \vspace{-0.3cm}
    \label{fig:sensitivity}
\end{figure}

\begin{figure*}[t]
\centering
\captionsetup[subfigure]{
    font=small,skip=3pt,justification=centering
}
\begin{subfigure}[b]{0.25\linewidth}
    \centering
    \includegraphics[width=\linewidth]{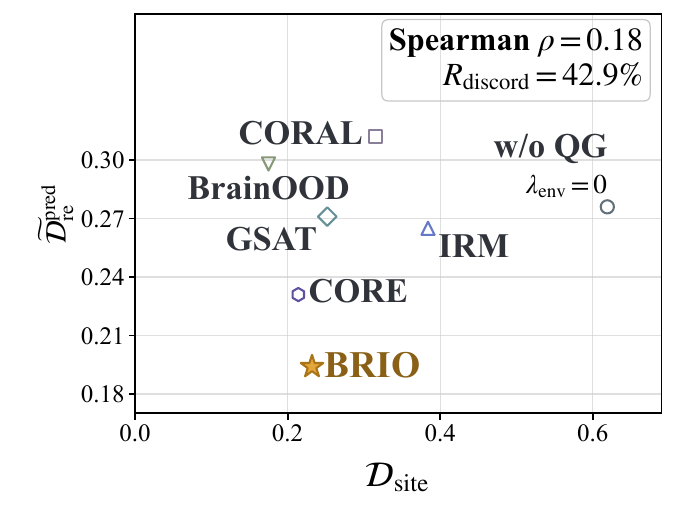}
    \caption{Cross-method}
    \label{fig:cross-method-abide}
\end{subfigure}%
\begin{subfigure}[b]{0.25\linewidth}
    \centering
    \includegraphics[width=\linewidth]{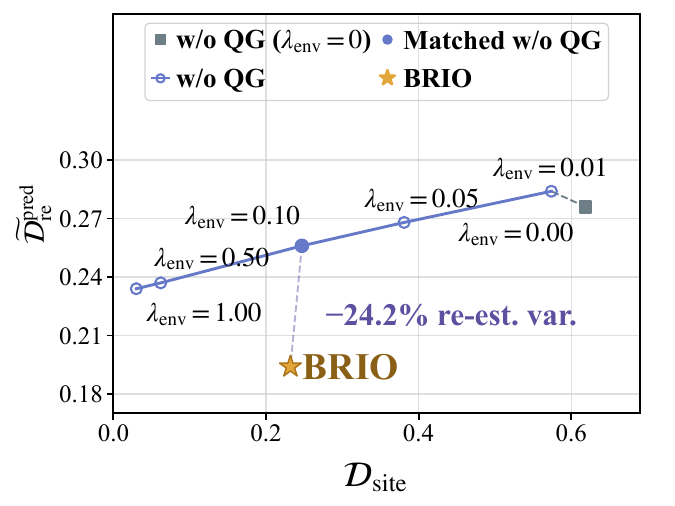}
    \caption{Alignment}
    \label{fig:alignment-abide}
\end{subfigure}%
\begin{subfigure}[b]{0.26\linewidth}
    \centering
    \includegraphics[width=\linewidth]{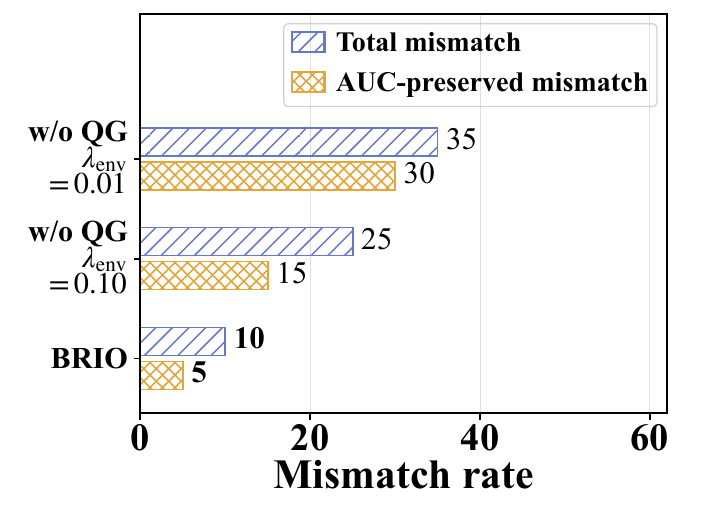}
    \caption{Mismatch}
    \label{fig:mismatch-abide}
\end{subfigure}%
\begin{subfigure}[b]{0.24\linewidth}
    \centering
    \includegraphics[width=\linewidth]{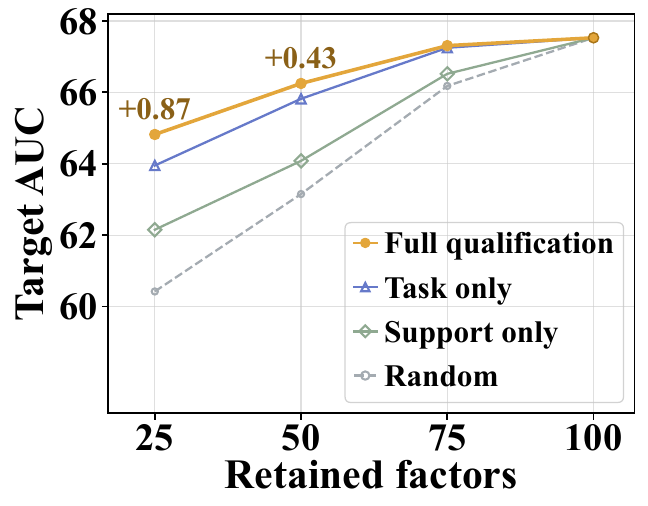}
    \caption{Factor retention}
    \label{fig:factor-retention-abide}
\end{subfigure}
\vspace{-0.6cm}
\caption{Cross-method relationships (a), alignment trajectories (b),
mismatch rates (c), and factor retention (d) on ABIDE.}
\label{fig:diagnostics-retention-abide}
\vspace{-0.2cm}
\end{figure*}

\subsection{Sensitivity Study}
\vspace{-0.1cm}

We examine the sensitivity of \method{} to the number of FC re-estimates $K$ and connectome factors $J$. Fig.~\ref{fig:sensitivity}(c,d) shows that increasing $K$ from small values generally improves AUC on ABIDE and REST-meta-MDD, whereas further increases yield no consistent gains. This suggests that additional re-estimates help characterize variation in predictive contributions, with diminishing benefits at larger budgets. We vary $J$ with the total representation dimension fixed. Intermediate values perform better than very small or large values on both datasets, suggesting a trade-off between separating connectivity patterns and preserving sufficient capacity within each factor. Too few factors may combine patterns with distinct predictive roles and re-estimation responses, whereas too many reduce each factor's dimension and may limit its discriminative capacity. More results are shown in Appendix~\ref{app:sensitivity}.

\vspace{-0.2cm}
\subsection{Case Study}\label{sec:case_study}
\textbf{Site Discrepancy and Re-estimation Variation.} To examine whether lower source-site discrepancy implies lower predictive variation under FC re-estimation, we conduct cross-method comparisons and a controlled alignment-strength sweep, the latter using \method{} w/o QG, which removes qualification guidance by assigning equal weights to all factors and disabling qualification-based strength modulation, trained independently for each tested $\lambda_{\mathrm{env}}$. We measure class-conditional full-scan site discrepancy $\mathcal D_{\mathrm{site}}$ in Eq.~\eqref{eq:app_site_discrepancy} and normalized predictive re-estimation variation $\widetilde{\mathcal D}_{\mathrm{re}}^{\mathrm{pred}}$ in Eq.~\eqref{eq:app_predictive_variation}, and summarize their relation across methods by Spearman $\rho$, the discordance rate $R_{\mathrm{discord}}$, and the total and AUC-preserved mismatch rates defined in Appendix~\ref{app:diagnostics}. Fig.~\ref{fig:diagnostics-retention-abide}(a) shows that lower $\mathcal D_{\mathrm{site}}$ does not consistently correspond to lower $\widetilde{\mathcal D}_{\mathrm{re}}^{\mathrm{pred}}$ across methods, so cross-site agreement alone does not imply stable within-subject predictions. In the controlled sweep in Fig.~\ref{fig:diagnostics-retention-abide}(b), stronger alignment reduces both quantities, while the full \method{} achieves lower $\widetilde{\mathcal D}_{\mathrm{re}}^{\mathrm{pred}}$ than the w/o QG setting at comparable $\mathcal D_{\mathrm{site}}$, indicating that qualification-guided weighting and strength modulation offer benefits beyond simply increasing alignment strength. Beyond these average trends, Fig.~\ref{fig:diagnostics-retention-abide}(c) shows that \method{} achieves lower total and AUC-preserved mismatch rates than the evaluated w/o QG settings, with fewer mismatches even when predictive performance is preserved.

\textbf{Qualification and Factor Transferability.} To examine whether source-side qualification identifies transferable factors, we conduct a factor-retention experiment on ABIDE using a fixed, trained \method{} model. We rank factors by full qualification, task relevance alone, or re-estimation support alone, with random ranking as a reference. Qualification and support scores are averaged over source-site pairs and classes before ranking. For each ranking, we retain varying proportions of the highest-ranked factors and evaluate AUC on held-out sites using the original classifier without retraining, allowing direct comparisons of factor selection strategies. Fig.~\ref{fig:diagnostics-retention-abide}(d) shows that full qualification yields the highest AUC at partial retention, with larger gains when fewer factors are retained. Its advantage over task relevance alone suggests that re-estimation support provides complementary information about factor transferability, while the weaker performance of support alone highlights the importance of considering both criteria jointly. As more factors are retained, the performance gaps narrow and disappear at full retention, where all rankings recover the original model predictions.

%% file: table/OOD.tex
\begin{table*}[t]
    \centering
    \caption{Results on ABIDE, ABIDE (CC200), REST-meta-MDD, SRPBS, and ABCD (ADHD-task)
    over LOSO sites. \textbf{Bold} results indicate the best performance.}
    \vspace{-0.3cm}
    \label{tab:main_results_ood}
    \resizebox{\textwidth}{!}{
    \begin{tabular}{cc |cc |cc |cc |cc |cc}
    \toprule
    \multirow{2}{*}{Type} & \multirow{2}{*}{Method}
    & \multicolumn{2}{c|}{ABIDE}
    & \multicolumn{2}{c|}{ABIDE (CC200)}
    & \multicolumn{2}{c|}{REST-meta-MDD}
    & \multicolumn{2}{c|}{SRPBS}
    & \multicolumn{2}{c}{ABCD (ADHD-task)} \\
    &
    & AUC & ACC
    & AUC & ACC
    & AUC & ACC
    & AUC & ACC
    & AUC & ACC \\
    \midrule

    \multirow{3}{*}{\rotatebox{90}{\fontsize{9}{12}\selectfont GNN}}
    & GCN
    & 61.91$_{\pm 10.32}$ & 56.12$_{\pm 9.66}$
    & 50.21$_{\pm 8.86}$ & 48.38$_{\pm 8.55}$
    & 60.58$_{\pm 4.18}$ & 56.26$_{\pm 2.87}$
    & 65.24$_{\pm 19.47}$ & 61.68$_{\pm 19.17}$
    & 71.75$_{\pm 6.52}$ & 63.14$_{\pm 9.15}$ \\

    & GAT
    & 58.21$_{\pm 8.69}$ & 52.31$_{\pm 6.01}$
    & 57.68$_{\pm 9.37}$ & 54.69$_{\pm 7.74}$
    & 59.97$_{\pm 9.38}$ & 56.17$_{\pm 6.72}$
    & 66.68$_{\pm 15.58}$ & 60.86$_{\pm 14.38}$
    & 70.37$_{\pm 11.32}$ & 63.32$_{\pm 11.20}$ \\

    & GIN
    & 57.73$_{\pm 5.21}$ & 52.08$_{\pm 3.10}$
    & 52.40$_{\pm 8.83}$ & 47.84$_{\pm 10.23}$
    & 56.55$_{\pm 6.72}$ & 53.46$_{\pm 4.72}$
    & 63.74$_{\pm 21.68}$ & 60.19$_{\pm 18.08}$
    & 68.97$_{\pm 8.84}$ & 60.73$_{\pm 7.03}$ \\

    \midrule

    \multirow{2}{*}{\rotatebox{90}{\fontsize{9}{12}\selectfont OOD}}
    & CORAL
    & 55.48$_{\pm 6.12}$ & 50.23$_{\pm 5.49}$
    & 53.75$_{\pm 10.69}$ & 50.06$_{\pm 7.42}$
    & 60.37$_{\pm 8.97}$ & 57.18$_{\pm 7.78}$
    & 63.88$_{\pm 15.48}$ & 59.86$_{\pm 12.68}$
    & 69.43$_{\pm 7.70}$ & 62.55$_{\pm 10.46}$ \\

    & IRM
    & 58.56$_{\pm 6.12}$ & 52.98$_{\pm 3.85}$
    & 51.77$_{\pm 10.64}$ & 50.84$_{\pm 6.99}$
    & 58.56$_{\pm 4.38}$ & 54.18$_{\pm 3.35}$
    & 60.98$_{\pm 21.87}$ & 59.27$_{\pm 16.78}$
    & 70.37$_{\pm 6.60}$ & 59.06$_{\pm 7.57}$ \\

    \midrule

    \multirow{4}{*}{
        \rotatebox{90}{
            \shortstack{\fontsize{9}{12}\selectfont Graph\\OOD}
        }
    }
    & GSAT
    & 58.48$_{\pm 7.62}$ & 53.61$_{\pm 7.98}$
    & 50.81$_{\pm 7.05}$ & 49.05$_{\pm 6.27}$
    & 57.38$_{\pm 5.77}$ & 54.76$_{\pm 4.52}$
    & 62.76$_{\pm 18.97}$ & 59.18$_{\pm 15.78}$
    & 71.27$_{\pm 8.73}$ & 61.16$_{\pm 7.48}$ \\

    & DisC
    & 56.46$_{\pm 7.85}$ & 53.62$_{\pm 8.68}$
    & 52.85$_{\pm 6.48}$ & 50.75$_{\pm 6.41}$
    & 56.19$_{\pm 4.42}$ & 53.27$_{\pm 3.88}$
    & 65.18$_{\pm 19.75}$ & 61.46$_{\pm 16.79}$
    & 69.59$_{\pm 7.70}$ & 58.83$_{\pm 9.45}$ \\

    & CEPG
    & 50.83$_{\pm 8.99}$ & 47.51$_{\pm 6.12}$
    & 44.82$_{\pm 9.74}$ & 45.45$_{\pm 8.93}$
    & 53.89$_{\pm 5.27}$ & 52.96$_{\pm 4.46}$
    & 55.34$_{\pm 8.64}$ & 50.58$_{\pm 6.28}$
    & 63.11$_{\pm 11.06}$ & 57.07$_{\pm 9.59}$ \\

    & DiSCO
    & 56.57$_{\pm 6.72}$ & 52.96$_{\pm 4.83}$
    & 51.37$_{\pm 8.45}$ & 49.35$_{\pm 6.21}$
    & 57.92$_{\pm 3.73}$ & 54.08$_{\pm 2.84}$
    & 65.68$_{\pm 18.16}$ & 59.12$_{\pm 15.46}$
    & 67.51$_{\pm 9.83}$ & 61.98$_{\pm 7.44}$ \\

    \midrule

    \multirow{7}{*}{
        \rotatebox{90}{
            \shortstack{\fontsize{9}{12}\selectfont Brain\\Networks}
        }
    }
    & BrainNetTF
    & 62.64$_{\pm 6.20}$ & 58.38$_{\pm 7.06}$
    & 60.48$_{\pm 8.12}$ & 54.25$_{\pm 7.14}$
    & 61.48$_{\pm 7.38}$ & 58.17$_{\pm 5.58}$
    & 53.04$_{\pm 17.02}$ & 51.89$_{\pm 10.78}$
    & 68.28$_{\pm 4.54}$ & 56.91$_{\pm 10.43}$ \\

    & XG-GNN
    & 60.20$_{\pm 7.73}$ & 51.61$_{\pm 4.09}$
    & 55.94$_{\pm 9.85}$ & 53.64$_{\pm 6.42}$
    & 60.08$_{\pm 7.23}$ & 56.24$_{\pm 6.49}$
    & 73.28$_{\pm 15.19}$ & 66.17$_{\pm 17.27}$
    & 71.20$_{\pm 5.46}$ & 55.66$_{\pm 9.78}$ \\

    & AGMGC
    & 58.47$_{\pm 8.68}$ & 48.97$_{\pm 5.42}$
    & 59.32$_{\pm 6.54}$ & 53.17$_{\pm 5.08}$
    & 59.06$_{\pm 5.07}$ & 54.74$_{\pm 4.12}$
    & 74.98$_{\pm 15.31}$ & 67.34$_{\pm 16.85}$
    & 70.86$_{\pm 8.67}$ & 63.02$_{\pm 11.29}$ \\

    & FC-HGNN
    & 55.05$_{\pm 7.90}$ & 49.71$_{\pm 4.58}$
    & 58.76$_{\pm 7.21}$ & 54.82$_{\pm 5.16}$
    & 58.69$_{\pm 6.69}$ & 56.78$_{\pm 5.71}$
    & 48.78$_{\pm 10.58}$ & 45.28$_{\pm 6.38}$
    & 71.08$_{\pm 5.74}$ & 59.10$_{\pm 12.34}$ \\

    & BrainOOD
    & 63.61$_{\pm 6.73}$ & 57.14$_{\pm 9.70}$
    & 66.15$_{\pm 4.12}$ & 58.21$_{\pm 3.84}$
    & 63.22$_{\pm 6.29}$ & 58.45$_{\pm 5.20}$
    & 81.48$_{\pm 7.31}$ & 69.18$_{\pm 15.21}$
    & 72.77$_{\pm 8.07}$ & 64.36$_{\pm 7.24}$ \\

    & DeCI
    & 55.81$_{\pm 8.91}$ & 58.26$_{\pm 2.89}$
    & 62.18$_{\pm 3.75}$ & 59.48$_{\pm 3.52}$
    & 54.38$_{\pm 2.95}$ & 56.18$_{\pm 2.89}$
    & 82.05$_{\pm 12.28}$ & 76.69$_{\pm 13.66}$
    & 55.03$_{\pm 6.79}$ & 62.81$_{\pm 5.46}$ \\

    & CORE
    & 65.32$_{\pm 6.88}$ & 60.97$_{\pm 5.90}$
    & 68.42$_{\pm 6.35}$ & 59.65$_{\pm 5.42}$
    & 67.44$_{\pm 8.20}$ & 62.81$_{\pm 6.04}$
    & 82.61$_{\pm 8.25}$ & 77.65$_{\pm 9.84}$
    & 74.31$_{\pm 8.34}$ & 67.26$_{\pm 7.53}$ \\

    \midrule

    & \method{}
    & \textbf{67.53$_{\pm 4.12}$} & \textbf{62.18$_{\pm 3.68}$}
    & \textbf{70.86$_{\pm 4.87}$} & \textbf{61.92$_{\pm 3.94}$}
    & \textbf{68.84$_{\pm 3.75}$} & \textbf{64.42$_{\pm 2.91}$}
    & \textbf{83.91$_{\pm 5.63}$} & \textbf{78.82$_{\pm 5.28}$}
    & \textbf{75.98$_{\pm 4.39}$} & \textbf{68.74$_{\pm 3.82}$} \\

    \bottomrule
    \end{tabular}
    }
    \vspace{-0.4cm}
\end{table*}

%% file: main/6_conclusion.tex
\vspace{-0.2cm}
\section{Conclusion}
\vspace{-0.2cm}

To address finite-sample FC estimation variability in cross-site brain network generalization, we proposed \method{}. \method{} qualifies predictive evidence under within-scan FC re-estimation before aligning it across sites. Consistently indexed connectome factors enable comparison of predictive contributions across estimates. Calibrated changes in these contributions, combined with predictive relevance, determine which factors to align and how strongly. Experiments on four real-world datasets show consistent gains over competitive baselines. These findings support estimation stability as a criterion for what to align rather than only how much to align, and motivate future research on adaptive re-estimation budgets and dynamic or multimodal connectomes.

%% file: main/7_appendix.tex
\section{Notation Summary}

\input{table/notations}

As shown in the Table~\ref{tab:notation}, we  summarize the key notations of this paper.

\section{Proof of Proposition~\ref{prop:reestimation_contrast_stability}}
\label{app:reestimation_contrast_stability}

\textbf{Proposition~\ref{prop:reestimation_contrast_stability} (Class-Contrast Stability under FC Re-estimation)} \textit{Fix the model snapshot, a source site $e$, and a factor $j$, with $N_{e,0},N_{e,1}>0$ and $K\geq1$. Let $\Delta_{e,j}^{\mathrm{re}}$ be the class contrast of contributions averaged over subjects and FC re-estimates, and let $\kappa_{e,c,j}=\Delta_{e,j}^2/\big(4(V_{e,c,j}^{\mathrm{sub}}+\epsilon_q)\big)$ be the full-scan separation ratio of class $c$. Then
\begin{equation}
    \left|\Delta_{e,j}^{\mathrm{re}}-\Delta_{e,j}\right|
    \leq \sum_{c\in\{0,1\}}
    \sqrt{\left(T_{e,c,j}+\epsilon_q\right)\left(\frac{1}{q_{e,c,j}^{\mathrm{re}}}-1\right)},
\end{equation}
and $\Delta_{e,j}^{\mathrm{re}}$ has the same sign as $\Delta_{e,j}$ whenever $q_{e,c,j}^{\mathrm{re}}>(1+\kappa_{e,c,j})/(1+2\kappa_{e,c,j})$ for both $c\in\{0,1\}$.}

\textbf{\textit{Proof.}}

For each class $c\in\{0,1\}$, the mean re-estimated contribution and the resulting class contrast are
\begin{equation}
\begin{aligned}
    \mu_{e,c,j}^{\mathrm{re}}
    = \frac{1}{N_{e,c}K}
      \sum_{i\in\mathcal I_{e,c}}\sum_{k=1}^{K}a_{i,j}^{e,(k)},
    \quad
    \Delta_{e,j}^{\mathrm{re}}
    = \mu_{e,1,j}^{\mathrm{re}}-\mu_{e,0,j}^{\mathrm{re}}.
\end{aligned}
\label{eq:reestimated_class_contrast}
\end{equation}
Let $\xi_{i,j}^{e,(k)}=a_{i,j}^{e,(k)}-a_{i,j}^{e,(0)}$ denote the contribution change between the $k$-th FC re-estimate and the full-scan estimate. Repeating each full-scan contribution across the $K$ paired comparisons gives
\begin{equation}
\begin{aligned}
    \mu_{e,c,j}^{\mathrm{re}}-\mu_{e,c,j}
    = \frac{1}{N_{e,c}K}
       \sum_{i\in\mathcal I_{e,c}}\sum_{k=1}^{K}
       \xi_{i,j}^{e,(k)}.
\end{aligned}
\label{eq:app_class_mean_shift}
\end{equation}
Applying the Cauchy--Schwarz inequality to these $N_{e,c}K$ paired changes yields
\begin{equation}
\begin{aligned}
    \left|\mu_{e,c,j}^{\mathrm{re}}-\mu_{e,c,j}\right|^2
    &= \frac{1}{(N_{e,c}K)^2}
       \left(
           \sum_{i\in\mathcal I_{e,c}}\sum_{k=1}^{K}
           \xi_{i,j}^{e,(k)}
       \right)^2 \\
    &\leq \frac{1}{N_{e,c}K}
       \sum_{i\in\mathcal I_{e,c}}\sum_{k=1}^{K}
       \left(\xi_{i,j}^{e,(k)}\right)^2
     = D_{e,c,j}^{\mathrm{re}},
\end{aligned}
\label{eq:app_class_mean_shift_bound}
\end{equation}
where the final equality follows from Eq.~\eqref{eq:reestimation_variation}, and equality holds if and only if $\xi_{i,j}^{e,(k)}$ is constant over $i\in\mathcal I_{e,c}$ and $k$, that is, when re-estimation shifts all class-$c$ contributions by a common constant. Since the change in the class contrast is the difference between the two class-mean shifts, taking square roots and applying the triangle inequality gives
\begin{equation}
\begin{aligned}
    \left|\Delta_{e,j}^{\mathrm{re}}-\Delta_{e,j}\right|
    &= \left|
       \left(\mu_{e,1,j}^{\mathrm{re}}-\mu_{e,1,j}\right)
       -\left(\mu_{e,0,j}^{\mathrm{re}}-\mu_{e,0,j}\right)
       \right| \\
    &\leq
       \left|\mu_{e,1,j}^{\mathrm{re}}-\mu_{e,1,j}\right|
       +\left|\mu_{e,0,j}^{\mathrm{re}}-\mu_{e,0,j}\right|
     \leq \sum_{c\in\{0,1\}}
       \sqrt{D_{e,c,j}^{\mathrm{re}}}.
\end{aligned}
\label{eq:app_class_contrast_deviation_bound}
\end{equation}

To express this bound through re-estimation support, let $m_{e,j}=(\mu_{e,0,j}+\mu_{e,1,j})/2$ denote the midpoint of the full-scan class means, so that $\mu_{e,c,j}-m_{e,j}=(2c-1)\Delta_{e,j}/2$ for either class. Expanding the squared distance from this midpoint around the class mean yields
\begin{equation}
\begin{aligned}
    \frac{1}{N_{e,c}}\sum_{i\in\mathcal I_{e,c}}
       \left(a_{i,j}^{e,(0)}-m_{e,j}\right)^2
    &=
       \frac{1}{N_{e,c}}\sum_{i\in\mathcal I_{e,c}}
       \left[
           \left(a_{i,j}^{e,(0)}-\mu_{e,c,j}\right)
           +\left(\mu_{e,c,j}-m_{e,j}\right)
       \right]^2 \\
    &=
       V_{e,c,j}^{\mathrm{sub}}
       +\frac{1}{4}\Delta_{e,j}^{2}
     = T_{e,c,j},
\end{aligned}
\label{eq:app_task_reference_identity}
\end{equation}
where the cross term vanishes because $\sum_{i\in\mathcal I_{e,c}}(a_{i,j}^{e,(0)}-\mu_{e,c,j})=0$. Thus $T_{e,c,j}\geq0$ is the empirical mean squared distance of class-$c$ full-scan contributions from the midpoint of the two class means. Together with $D_{e,c,j}^{\mathrm{re}}\geq0$ and $\epsilon_q>0$, this ensures $0<q_{e,c,j}^{\mathrm{re}}\leq1$, and rearranging Eq.~\eqref{eq:reestimation_support} gives
\begin{equation}
\begin{aligned}
    \left(T_{e,c,j}+\epsilon_q\right)
    \left(\frac{1}{q_{e,c,j}^{\mathrm{re}}}-1\right)
    = \left(T_{e,c,j}+\epsilon_q\right)
      \frac{D_{e,c,j}^{\mathrm{re}}}{T_{e,c,j}+\epsilon_q}
    = D_{e,c,j}^{\mathrm{re}}.
\end{aligned}
\label{eq:app_support_deviation_identity}
\end{equation}
Substituting Eq.~\eqref{eq:app_support_deviation_identity} into Eq.~\eqref{eq:app_class_contrast_deviation_bound} establishes Eq.~\eqref{eq:class_contrast_stability}.

For sign preservation, we first show that the support threshold is equivalent to a bound on the re-estimation variation. Since $4T_{e,c,j}=4V_{e,c,j}^{\mathrm{sub}}+\Delta_{e,j}^2$ by Eq.~\eqref{eq:app_task_reference_identity},
\begin{equation}
\begin{aligned}
    \frac{1+\kappa_{e,c,j}}{1+2\kappa_{e,c,j}}
    = \frac{4\left(V_{e,c,j}^{\mathrm{sub}}+\epsilon_q\right)+\Delta_{e,j}^2}
           {4\left(V_{e,c,j}^{\mathrm{sub}}+\epsilon_q\right)+2\Delta_{e,j}^2}
    = \frac{T_{e,c,j}+\epsilon_q}
           {T_{e,c,j}+\epsilon_q+\Delta_{e,j}^2/4},
\end{aligned}
\label{eq:app_threshold_identity}
\end{equation}
so by Eq.~\eqref{eq:reestimation_support},
\begin{equation}
\begin{aligned}
    q_{e,c,j}^{\mathrm{re}}>\frac{1+\kappa_{e,c,j}}{1+2\kappa_{e,c,j}}
    \iff
    \frac{T_{e,c,j}+\epsilon_q}{T_{e,c,j}+\epsilon_q+D_{e,c,j}^{\mathrm{re}}}
    >\frac{T_{e,c,j}+\epsilon_q}{T_{e,c,j}+\epsilon_q+\Delta_{e,j}^2/4}
    \iff
    D_{e,c,j}^{\mathrm{re}}<\frac{\Delta_{e,j}^2}{4}.
\end{aligned}
\label{eq:app_support_threshold}
\end{equation}
Suppose the threshold condition holds for both classes. Then $\sqrt{D_{e,c,j}^{\mathrm{re}}}<|\Delta_{e,j}|/2$ for each $c$, hence $\sum_{c}\sqrt{D_{e,c,j}^{\mathrm{re}}}<|\Delta_{e,j}|$ and $\Delta_{e,j}\neq0$. Using Eq.~\eqref{eq:app_class_contrast_deviation_bound}, we obtain
\begin{equation}
\begin{aligned}
    \Delta_{e,j}\Delta_{e,j}^{\mathrm{re}}
    &= \Delta_{e,j}^{2}
       +\Delta_{e,j}
       \left(\Delta_{e,j}^{\mathrm{re}}-\Delta_{e,j}\right) \\
    &\geq \Delta_{e,j}^{2}
       -|\Delta_{e,j}|
        \left|\Delta_{e,j}^{\mathrm{re}}-\Delta_{e,j}\right|
     \geq |\Delta_{e,j}|
       \Big(|\Delta_{e,j}|-\sum_{c\in\{0,1\}}\sqrt{D_{e,c,j}^{\mathrm{re}}}\Big)>0.
\end{aligned}
\label{eq:app_class_contrast_sign_preservation}
\end{equation}
Hence $\Delta_{e,j}^{\mathrm{re}}$ and $\Delta_{e,j}$ have the same sign. The threshold $(1+\kappa_{e,c,j})/(1+2\kappa_{e,c,j})$ lies in $(1/2,1]$ and decreases in $\kappa_{e,c,j}$, so strongly separated factors preserve the sign of their class contrast under moderate support, whereas weakly separated factors require support close to one.

\section{Proof of Proposition~\ref{prop:pairwise_qualification_variational}}
\label{app:pairwise_qualification_variational}

\textbf{Proposition~\ref{prop:pairwise_qualification_variational} (Variational Form of Pairwise Qualification)} \textit{For a probability vector $\bm\pi$ over $J$ factors and supports $q_j\in(0,1]$, define $\alpha_j=\pi_jq_j$, $g=\sum_{j=1}^{J}\alpha_j$, and $\overline\alpha_j=\alpha_j/g$. Let $\mathcal V_{\pi}$ be the set of probability vectors $\bm v$ with $v_j=0$ whenever $\pi_j=0$, and define
\begin{equation}
    \mathcal J(\bm v)
    = \operatorname{KL}\!\left(\bm v\,\middle\|\,\bm\pi\right)
    + \sum_{j=1}^{J}v_j\log\frac{1}{q_j}.
\end{equation}
Then $\overline{\bm\alpha}$ is the unique minimizer of $\mathcal J$ over $\mathcal V_{\pi}$, with minimum value $-\log g$, and
\begin{equation}
    \overline\alpha_j
    = \frac{\partial\log g}{\partial\log q_j},
    \quad
    -\log g
    \leq \sum_{j=1}^{J}\pi_j\log\frac{1}{q_j},
\end{equation}
where equality holds if and only if $q_j$ is constant over the factors with $\pi_j>0$.}

\textbf{\textit{Proof.}}

Let $\mathcal I_{\pi}=\{j\in\{1,\ldots,J\}:\pi_j>0\}$ denote the support of $\bm\pi$, which is nonempty because $\bm\pi$ is a probability vector. Every $\bm v\in\mathcal V_{\pi}$ is supported on $\mathcal I_{\pi}$ and satisfies $\sum_{j\in\mathcal I_{\pi}}v_j=1$. Throughout the proof, we use the convention $0\log(0/b)=0$ for $b>0$.

Since $q_j\in(0,1]$, $\alpha_j$ is strictly positive on $\mathcal I_{\pi}$ and zero outside it. Moreover,
\begin{equation}
    0<g
    = \sum_{j\in\mathcal I_{\pi}}\pi_jq_j
    \leq \sum_{j\in\mathcal I_{\pi}}\pi_j
    = 1.
\label{eq:app_pairwise_qualification_normalizer}
\end{equation}
Thus, $\overline{\bm\alpha}$ is well defined, belongs to $\mathcal V_{\pi}$, and has strictly positive entries on $\mathcal I_{\pi}$.

For every $\bm v\in\mathcal V_{\pi}$, combining the two terms of $\mathcal J$ and substituting $\alpha_j=g\,\overline\alpha_j$ gives
\begin{equation}
\begin{aligned}
    \mathcal J(\bm v)
    &= \sum_{j\in\mathcal I_{\pi}}v_j\log\frac{v_j}{\pi_j}
       +\sum_{j\in\mathcal I_{\pi}}v_j\log\frac{1}{q_j}
     = \sum_{j\in\mathcal I_{\pi}}v_j\log\frac{v_j}{\alpha_j} \\
    &= \sum_{j\in\mathcal I_{\pi}}v_j\log\frac{v_j}{\overline\alpha_j}
       -\left(\sum_{j\in\mathcal I_{\pi}}v_j\right)\log g
     = \operatorname{KL}\!\left(\bm v\,\middle\|\,\overline{\bm\alpha}\right)-\log g,
\end{aligned}
\label{eq:app_qualification_kl_decomposition}
\end{equation}
where the last equality uses $\sum_{j\in\mathcal I_{\pi}}v_j=1$. The objective is therefore a KL divergence to $\overline{\bm\alpha}$ plus a term independent of $\bm v$.

Let $\mathcal I_{\bm v}=\{j\in\mathcal I_{\pi}:v_j>0\}$. Applying $\log t\leq t-1$ for $t>0$ gives
\begin{equation}
\begin{aligned}
    \operatorname{KL}\!\left(\bm v\,\middle\|\,\overline{\bm\alpha}\right)
    &= -\sum_{j\in\mathcal I_{\bm v}}v_j\log\frac{\overline\alpha_j}{v_j}
     \geq \sum_{j\in\mathcal I_{\bm v}}\left(v_j-\overline\alpha_j\right)
     = \sum_{j\in\mathcal I_{\pi}\setminus\mathcal I_{\bm v}}\overline\alpha_j
     \geq0.
\end{aligned}
\label{eq:app_qualification_kl_nonnegativity}
\end{equation}
If $\mathcal I_{\bm v}\neq\mathcal I_{\pi}$, the final sum is strictly positive because $\overline\alpha_j>0$ on $\mathcal I_{\pi}$. Otherwise, equality holds precisely when $\overline\alpha_j/v_j=1$ for every $j\in\mathcal I_{\pi}$, by the equality condition of $\log t\leq t-1$. Hence the KL divergence vanishes if and only if $\bm v=\overline{\bm\alpha}$, and Eq.~\eqref{eq:app_qualification_kl_decomposition} yields
\begin{equation}
    \mathcal J(\bm v)
    \geq-\log g
    =\mathcal J\!\left(\overline{\bm\alpha}\right),
    \qquad \bm v\in\mathcal V_{\pi},
\label{eq:app_qualification_optimal_value}
\end{equation}
with equality if and only if $\bm v=\overline{\bm\alpha}$. Since $\overline{\bm\alpha}$ is feasible, it is the unique minimizer, and the minimum value is $-\log g$.

For the sensitivity identity, $g=\sum_{j'\in\mathcal I_{\pi}}\pi_{j'}q_{j'}$ is linear in each $q_j$ with $\partial g/\partial q_j=\pi_j$, so
\begin{equation}
    \frac{\partial\log g}{\partial\log q_j}
    = \frac{q_j}{g}\,\frac{\partial g}{\partial q_j}
    = \frac{\pi_jq_j}{g}
    = \overline\alpha_j.
\label{eq:app_qualification_sensitivity}
\end{equation}

For the inequality, $\bm\pi\in\mathcal V_{\pi}$ and $\operatorname{KL}(\bm\pi\|\bm\pi)=0$, so Eq.~\eqref{eq:app_qualification_optimal_value} evaluated at $\bm v=\bm\pi$ gives
\begin{equation}
    -\log g
    \leq \mathcal J(\bm\pi)
    = \sum_{j\in\mathcal I_{\pi}}\pi_j\log\frac{1}{q_j},
\label{eq:app_qualification_relevance_only}
\end{equation}
with equality if and only if $\bm\pi=\overline{\bm\alpha}$, that is, $\pi_j=\pi_jq_j/g$ for every $j\in\mathcal I_{\pi}$, which holds if and only if $q_j=g$ for every $j\in\mathcal I_{\pi}$. Equivalently, Eq.~\eqref{eq:app_qualification_relevance_only} is Jensen's inequality $\log\sum_{j}\pi_jq_j\geq\sum_{j}\pi_j\log q_j$ for the concave logarithm, with the same equality condition.
\hfill$\square$

With $\bm\pi=\bm\pi^{\mathrm{task}}$ and $q_j=q^{\mathrm{pair}}_{e,e',c,j}$, the objective specializes to pairwise qualification. By Eq.~\eqref{eq:reestimation_support}, the penalty expands as
\begin{equation}
    \log\frac{1}{q^{\mathrm{pair}}_{e,e',c,j}}
    =
    \frac{1}{2}\log\left(
        1+\frac{D_{e,c,j}^{\mathrm{re}}}{T_{e,c,j}+\epsilon_q}
    \right)
    +\frac{1}{2}\log\left(
        1+\frac{D_{e',c,j}^{\mathrm{re}}}{T_{e',c,j}+\epsilon_q}
    \right),
\label{eq:app_pairwise_support_penalty_calibration}
\end{equation}
so $\mathcal J_{e,e',c}$ trades proximity to $\bm\pi^{\mathrm{task}}$ against re-estimation deviations measured on the local reference scales of both sites. Eq.~\eqref{eq:app_qualification_sensitivity} becomes
\begin{equation}
    \frac{\partial\log g_{e,e',c}}{\partial\log q^{\mathrm{pair}}_{e,e',c,j}}
    = \overline\alpha_{e,e',c,j},
    \quad
    \mathrm{d}\log g_{e,e',c}
    = \sum_{j=1}^{J}\overline\alpha_{e,e',c,j}\,
      \mathrm{d}\log q^{\mathrm{pair}}_{e,e',c,j},
\label{eq:app_alignment_strength_differential}
\end{equation}
so a relative change in the support of factor $j$ moves the alignment strength by its normalized weight times that change. Evaluating Eq.~\eqref{eq:app_qualification_kl_decomposition} at $\bm v=\bm\pi^{\mathrm{task}}$ gives the exact reduction in support penalty obtained by reweighting,
\begin{equation}
    \sum_{j=1}^{J}\pi_j^{\mathrm{task}}\log\frac{1}{q^{\mathrm{pair}}_{e,e',c,j}}
    -\left(-\log g_{e,e',c}\right)
    = \operatorname{KL}\!\left(\bm\pi^{\mathrm{task}}\,\middle\|\,\overline{\bm\alpha}_{e,e',c}\right)
    \geq0,
\label{eq:app_reweighting_gain}
\end{equation}
which is strictly positive unless $q^{\mathrm{pair}}_{e,e',c,j}$ is constant over the factors with $\pi_j^{\mathrm{task}}>0$.

\section{Debiased Sinkhorn Divergence}
\label{app:sinkhorn}

Optimal transport compares empirical distributions by minimizing a ground cost over couplings that preserve their probability masses~\cite{gabriel2019computational}. Entropic regularization enables computation through Sinkhorn scaling, and a self-transport correction removes the regularization bias at identical inputs~\cite{cuturi2013sinkhorn,feydy2019interpolating}. This part gives the construction for the qualification-weighted distributions in Eq.~\eqref{eq:qualification_guided_alignment} and derives how the qualifications enter the alignment gradient.

Fix a source-site pair $(e,e')$, class $c$, and the associated qualification profile. Write $\mu=\widehat P_{e,c}^{Q[e,e']}$ and $\nu=\widehat P_{e',c}^{Q[e,e']}$ for the empirical distributions in Eq.~\eqref{eq:qualified_site_distributions}, and let $\bm x_i$ and $\bm y_h$ denote the weighted representations from Eq.~\eqref{eq:pair_specific_representation}. Then
\begin{equation}
\begin{aligned}
    \mu
    &=\sum_{i=1}^{n}\eta_i\delta_{\bm x_i},
    \quad
    \eta_i=\frac{1}{n},
    \quad
    n=|\mathcal B_{e,c}|, \\
    \nu
    &=\sum_{h=1}^{n'}\zeta_h\delta_{\bm y_h},
    \quad
    \zeta_h=\frac{1}{n'},
    \quad
    n'=|\mathcal B_{e',c}|,
\end{aligned}
\label{eq:app_sinkhorn_empirical_measures}
\end{equation}
with $n,n'>0$, so the marginal vectors $\bm\eta$ and $\bm\zeta$ have positive entries and unit total mass. Both sites use the same normalized qualifications, so the squared Euclidean ground cost coincides with the qualification-weighted dissimilarity in Eq.~\eqref{eq:qualified_pairwise_distance}:
\begin{equation}
    C_{ih}=\|\bm x_i-\bm y_h\|_2^2,
    \quad
    \|\bm x_i\|_2^2=\|\bm y_h\|_2^2=\frac12,
    \quad
    0\leq C_{ih}\leq2.
\label{eq:app_sinkhorn_ground_cost}
\end{equation}

Let $\bm C=[C_{ih}]\in\mathbb R^{n\times n'}$. The admissible couplings are
\begin{equation}
    \mathcal U(\bm\eta,\bm\zeta)
    =
    \left\{
        \bm\Gamma\in\mathbb R_+^{n\times n'}:
        \bm\Gamma\bm1_{n'}=\bm\eta,
        \quad
        \bm\Gamma^\top\bm1_n=\bm\zeta
    \right\},
\label{eq:app_sinkhorn_couplings}
\end{equation}
and the unregularized transport cost is
\begin{equation}
    \operatorname{OT}_0(\mu,\nu)
    =
    \min_{\bm\Gamma\in\mathcal U(\bm\eta,\bm\zeta)}
    \langle\bm\Gamma,\bm C\rangle,
    \quad
    \langle\bm\Gamma,\bm C\rangle
    =
    \sum_{i=1}^{n}\sum_{h=1}^{n'}\Gamma_{ih}C_{ih},
\label{eq:app_sinkhorn_unregularized_ot}
\end{equation}
which equals the squared Wasserstein-2 distance $W_2^2(\mu,\nu)$ for the squared Euclidean ground cost~\cite{gabriel2019computational}.

We regularize the coupling relative to the independent coupling $\bm\eta\bm\zeta^\top$:
\begin{equation}
    \operatorname{OT}_{\varepsilon}(\mu,\nu)
    =
    \min_{\bm\Gamma\in\mathcal U(\bm\eta,\bm\zeta)}
    \left\{
        \langle\bm\Gamma,\bm C\rangle
        +
        \varepsilon\operatorname{KL}\!\left(
            \bm\Gamma\,\middle\|\,\bm\eta\bm\zeta^\top
        \right)
    \right\},
    \quad \varepsilon>0,
\label{eq:app_sinkhorn_regularized_ot}
\end{equation}
where the generalized KL divergence is
\begin{equation}
    \operatorname{KL}\!\left(
        \bm\Gamma\,\middle\|\,\bm\eta\bm\zeta^\top
    \right)
    =
    \sum_{i,h}
    \left[
        \Gamma_{ih}\log\frac{\Gamma_{ih}}{\eta_i\zeta_h}
        -\Gamma_{ih}+\eta_i\zeta_h
    \right],
\label{eq:app_sinkhorn_kl}
\end{equation}
with $0\log(0/b)=0$ for $b>0$. For feasible couplings, the linear terms cancel and the marginal constraints give
\begin{equation}
    \operatorname{KL}\!\left(
        \bm\Gamma\,\middle\|\,\bm\eta\bm\zeta^\top
    \right)
    =
    \sum_{i,h}\Gamma_{ih}\log\Gamma_{ih}
    -\sum_i\eta_i\log\eta_i
    -\sum_h\zeta_h\log\zeta_h,
\label{eq:app_sinkhorn_entropy_expansion}
\end{equation}
so this formulation and the negative-entropy formulation share the same optimal coupling and differ by marginal-dependent constants. The feasible set is compact and the regularizer strictly convex, so the minimizer exists, is unique, and has strictly positive entries~\cite{cuturi2013sinkhorn}.

Introduce dual potentials $\bm\phi\in\mathbb R^n$ and $\bm\psi\in\mathbb R^{n'}$ for the marginal constraints. The Lagrangian is
\begin{equation}
\begin{aligned}
    \mathcal L_{\mathrm{OT}}(\bm\Gamma,\bm\phi,\bm\psi)
    ={}
    \langle\bm\Gamma,\bm C\rangle
    +\varepsilon\operatorname{KL}\!\left(
        \bm\Gamma\,\middle\|\,\bm\eta\bm\zeta^\top
    \right) 
    +\bm\phi^\top(\bm\eta-\bm\Gamma\bm1_{n'})
    +\bm\psi^\top(\bm\zeta-\bm\Gamma^\top\bm1_n),
\end{aligned}
\label{eq:app_sinkhorn_lagrangian}
\end{equation}
and stationarity in each coupling entry yields
\begin{equation}
    C_{ih}
    +\varepsilon\log\frac{\Gamma_{ih}^{\star}}{\eta_i\zeta_h}
    -\phi_i^\star-\psi_h^\star
    =0,
    \quad
    \Gamma_{ih}^{\star}
    =
    \eta_i\zeta_h
    \exp\!\left(
        \frac{\phi_i^\star+\psi_h^\star-C_{ih}}{\varepsilon}
    \right).
\label{eq:app_sinkhorn_stationarity}
\end{equation}
Minimizing the Lagrangian over $\bm\Gamma$ for fixed potentials gives the dual problem
\begin{equation}
    \operatorname{OT}_{\varepsilon}(\mu,\nu)
    =
    \max_{\bm\phi,\bm\psi}
    \Bigg\{
        \bm\eta^\top\bm\phi+\bm\zeta^\top\bm\psi
        -\varepsilon\sum_{i,h}\eta_i\zeta_h
        \left[
            \exp\!\left(
                \frac{\phi_i+\psi_h-C_{ih}}{\varepsilon}
            \right)-1
        \right]
    \Bigg\},
\label{eq:app_sinkhorn_dual}
\end{equation}
where strong duality holds because the strictly positive coupling $\bm\eta\bm\zeta^\top$ is feasible. At optimality, the exponential term recovers the optimal coupling, whose total mass is one, so
\begin{equation}
    \operatorname{OT}_{\varepsilon}(\mu,\nu)
    =
    \bm\eta^\top\bm\phi^\star+\bm\zeta^\top\bm\psi^\star.
\label{eq:app_sinkhorn_dual_value}
\end{equation}

The exponential form leads to Sinkhorn scaling. With
\begin{equation}
    (\bm K_\varepsilon)_{ih}=\exp(-C_{ih}/\varepsilon),
    \quad
    u_i=\eta_i\exp(\phi_i/\varepsilon),
    \quad
    v_h=\zeta_h\exp(\psi_h/\varepsilon),
\label{eq:app_sinkhorn_kernel_scaling}
\end{equation}
the coupling is $\bm\Gamma=\operatorname{diag}(\bm u)\bm K_\varepsilon\operatorname{diag}(\bm v)$, and alternately enforcing the two marginal constraints gives
\begin{equation}
    \bm u^{(t+1)}
    =\bm\eta\oslash(\bm K_\varepsilon\bm v^{(t)}),
    \quad
    \bm v^{(t+1)}
    =\bm\zeta\oslash(\bm K_\varepsilon^\top\bm u^{(t+1)}),
\label{eq:app_sinkhorn_iterations}
\end{equation}
where $\oslash$ denotes elementwise division. For a positive kernel and positive marginals, positive initialization converges to the unique regularized optimum~\cite{sinkhorn1967concerning,cuturi2013sinkhorn}. Substituting Eq.~\eqref{eq:app_sinkhorn_stationarity} into the marginal constraints expresses the same updates in the dual potentials,
\begin{equation}
\begin{aligned}
    \phi_i^{(t+1)}
    =
    -\varepsilon\operatorname{LSE}_{h}
    \left(
        \log\zeta_h+\frac{\psi_h^{(t)}-C_{ih}}{\varepsilon}
    \right), \quad
    \psi_h^{(t+1)}
    =
    -\varepsilon\operatorname{LSE}_{i}
    \left(
        \log\eta_i+\frac{\phi_i^{(t+1)}-C_{ih}}{\varepsilon}
    \right),
\end{aligned}
\label{eq:app_sinkhorn_log_updates}
\end{equation}
where $\operatorname{LSE}_k(z_k)=\log\sum_k\exp(z_k)$. These are alternating maximizations of the dual objective and are evaluated with stable log-sum-exp reductions~\cite{feydy2019interpolating}.

Entropic regularization makes the self-transport cost positive. For $\mu$ with at least two distinct support points, the two terms of the self-transport objective vanish at different couplings,
\begin{equation}
\begin{gathered}
    \langle\bm\Gamma,\bm C^{xx}\rangle=0
    \iff
    \bm\Gamma=\operatorname{diag}(\bm\eta),
    \quad
    \operatorname{KL}\!\left(\bm\Gamma\,\middle\|\,\bm\eta\bm\eta^\top\right)=0
    \iff
    \bm\Gamma=\bm\eta\bm\eta^\top, \\
    \operatorname{KL}\!\left(\operatorname{diag}(\bm\eta)\,\middle\|\,\bm\eta\bm\eta^\top\right)
    =-\sum_i\eta_i\log\eta_i>0,
\end{gathered}
\label{eq:app_sinkhorn_self_bias}
\end{equation}
where $C_{i\ell}^{xx}=\|\bm x_i-\bm x_\ell\|_2^2$, so $\operatorname{OT}_{\varepsilon}(\mu,\mu)>0$. Subtracting half of each self-transport cost removes this bias and defines the debiased Sinkhorn divergence~\cite{feydy2019interpolating},
\begin{equation}
    \mathcal S_{\varepsilon}(\mu,\nu)
    =
    \operatorname{OT}_{\varepsilon}(\mu,\nu)
    -\frac12\operatorname{OT}_{\varepsilon}(\mu,\mu)
    -\frac12\operatorname{OT}_{\varepsilon}(\nu,\nu),
\label{eq:app_sinkhorn_divergence}
\end{equation}
where all three terms use the same $\varepsilon$ and the same weighted representation map, and $C_{hk}^{yy}=\|\bm y_h-\bm y_k\|_2^2$. By Eq.~\eqref{eq:app_sinkhorn_entropy_expansion}, with $H(\bm\eta)=-\sum_i\eta_i\log\eta_i$ and $\operatorname{OT}_{\varepsilon}^{\mathrm{ent}}$ the negative-entropy formulation,
\begin{equation}
    \operatorname{OT}_{\varepsilon}(\mu,\nu)
    =\operatorname{OT}_{\varepsilon}^{\mathrm{ent}}(\mu,\nu)
    +\varepsilon H(\bm\eta)+\varepsilon H(\bm\zeta),
    \quad
    \operatorname{OT}_{\varepsilon}(\mu,\mu)
    =\operatorname{OT}_{\varepsilon}^{\mathrm{ent}}(\mu,\mu)
    +2\varepsilon H(\bm\eta),
\label{eq:app_sinkhorn_constant_shift}
\end{equation}
so the constants cancel in Eq.~\eqref{eq:app_sinkhorn_divergence} and both formulations yield the same $\mathcal S_{\varepsilon}$. Transposition exchanges the marginals of a coupling without changing its cost or penalty, which gives
\begin{equation}
    \operatorname{OT}_{\varepsilon}(\mu,\nu)=\operatorname{OT}_{\varepsilon}(\nu,\mu),
    \quad
    \mathcal S_{\varepsilon}(\mu,\nu)=\mathcal S_{\varepsilon}(\nu,\mu),
    \quad
    \mathcal S_{\varepsilon}(\mu,\mu)=0.
\label{eq:app_sinkhorn_symmetry}
\end{equation}

For nonnegativity and definiteness, we use the optimal self-transport potentials~\cite{feydy2019interpolating}. Let $\mathcal D_{\mu\mu}$ denote the self-transport dual objective in Eq.~\eqref{eq:app_sinkhorn_dual} with $\nu=\mu$, which is concave and satisfies $\mathcal D_{\mu\mu}(\bm\phi,\bm\psi)=\mathcal D_{\mu\mu}(\bm\psi,\bm\phi)$. For an optimal pair $(\bm\phi^\star,\bm\psi^\star)$, the average $\bm\phi^\mu=\tfrac12(\bm\phi^\star+\bm\psi^\star)$ satisfies
\begin{equation}
    \mathcal D_{\mu\mu}(\bm\phi^\mu,\bm\phi^\mu)
    \geq
    \tfrac12\mathcal D_{\mu\mu}(\bm\phi^\star,\bm\psi^\star)
    +\tfrac12\mathcal D_{\mu\mu}(\bm\psi^\star,\bm\phi^\star)
    =\operatorname{OT}_{\varepsilon}(\mu,\mu),
\label{eq:app_sinkhorn_symmetric_potentials}
\end{equation}
so $(\bm\phi^\mu,\bm\phi^\mu)$ is optimal; define $\bm\phi^\nu$ likewise. Let $\{\bm z_\ell\}_{\ell=1}^{n_z}$ be the distinct points in the union of the two supports, and define
\begin{equation}
\begin{aligned}
    \Omega_{\ell r}
    =
    \exp\!\left(
        -\frac{\|\bm z_\ell-\bm z_r\|_2^2}{\varepsilon}
    \right), \quad
    \xi_\ell
    =
    \sum_{\{i:\bm x_i=\bm z_\ell\}}
    \eta_i\exp(\phi_i^\mu/\varepsilon),
    \quad
    \chi_\ell
    =
    \sum_{\{h:\bm y_h=\bm z_\ell\}}
    \zeta_h\exp(\phi_h^\nu/\varepsilon),
\end{aligned}
\label{eq:app_sinkhorn_common_support}
\end{equation}
where empty sums are zero. The Gaussian Gram matrix $\bm\Omega$ is strictly positive definite on distinct points. Let $\bm\Gamma^{xx,\star}$ and $\bm\Gamma^{yy,\star}$ be the optimal self-couplings. Their unit total masses and Eq.~\eqref{eq:app_sinkhorn_dual_value} give
\begin{equation}
\begin{aligned}
    \bm\xi^\top\bm\Omega\bm\xi
    &=\sum_{i,\ell}\Gamma_{i\ell}^{xx,\star}=1,
    &\quad
    \bm\chi^\top\bm\Omega\bm\chi
    &=\sum_{h,k}\Gamma_{hk}^{yy,\star}=1, \\
    \operatorname{OT}_{\varepsilon}(\mu,\mu)
    &=2\bm\eta^\top\bm\phi^\mu,
    &\quad
    \operatorname{OT}_{\varepsilon}(\nu,\nu)
    &=2\bm\zeta^\top\bm\phi^\nu.
\end{aligned}
\label{eq:app_sinkhorn_self_normalization}
\end{equation}
Evaluating the cross-transport dual at $(\bm\phi^\mu,\bm\phi^\nu)$ gives
\begin{equation}
\begin{aligned}
    \operatorname{OT}_{\varepsilon}(\mu,\nu)
    &\geq
    \bm\eta^\top\bm\phi^\mu+\bm\zeta^\top\bm\phi^\nu
    -\varepsilon
    \left[
        \sum_{i,h}\eta_i\zeta_h
        \exp\!\left(
            \frac{\phi_i^\mu+\phi_h^\nu-C_{ih}}{\varepsilon}
        \right)-1
    \right] \\
    &=
    \bm\eta^\top\bm\phi^\mu+\bm\zeta^\top\bm\phi^\nu
    +\varepsilon(1-\bm\xi^\top\bm\Omega\bm\chi),
\end{aligned}
\label{eq:app_sinkhorn_cross_lower_bound}
\end{equation}
and subtracting half of each self-cost from Eq.~\eqref{eq:app_sinkhorn_self_normalization} yields
\begin{equation}
\begin{aligned}
    \mathcal S_{\varepsilon}(\mu,\nu)
    &\geq
    \varepsilon(1-\bm\xi^\top\bm\Omega\bm\chi)
    =
    \frac{\varepsilon}{2}
    \left(
        \bm\xi^\top\bm\Omega\bm\xi
        +\bm\chi^\top\bm\Omega\bm\chi
        -2\bm\xi^\top\bm\Omega\bm\chi
    \right) \\
    &=
    \frac{\varepsilon}{2}
    (\bm\xi-\bm\chi)^\top\bm\Omega(\bm\xi-\bm\chi)
    \geq0.
\end{aligned}
\label{eq:app_sinkhorn_positivity}
\end{equation}
If $\mathcal S_{\varepsilon}(\mu,\nu)=0$, strict positive definiteness of $\bm\Omega$ gives $\bm\xi=\bm\chi$, and the marginal constraints of the self-couplings recover the masses at each distinct point,
\begin{equation}
    \mu(\{\bm z_\ell\})
    =
    \sum_{\{i:\bm x_i=\bm z_\ell\}}\eta_i
    =\xi_\ell(\bm\Omega\bm\xi)_\ell,
    \quad
    \nu(\{\bm z_\ell\})
    =
    \sum_{\{h:\bm y_h=\bm z_\ell\}}\zeta_h
    =\chi_\ell(\bm\Omega\bm\chi)_\ell,
\label{eq:app_sinkhorn_recover_masses}
\end{equation}
so $\bm\xi=\bm\chi$ implies $\mu=\nu$. Together with Eq.~\eqref{eq:app_sinkhorn_symmetry}, this establishes
\begin{equation}
    \mathcal S_{\varepsilon}(\mu,\nu)\geq0,
    \quad
    \mathcal S_{\varepsilon}(\mu,\nu)=0
    \iff
    \mu=\nu.
\label{eq:app_sinkhorn_separation}
\end{equation}

The same formulation provides the gradients used for representation learning. With fixed marginal weights, the envelope theorem gives $\partial\operatorname{OT}_{\varepsilon}/\partial C_{ih}=\Gamma_{ih}^{xy,\star}$ for the optimal cross-coupling~\cite{gabriel2019computational}, so
\begin{equation}
    \nabla_{\bm x_i}\operatorname{OT}_{\varepsilon}(\mu,\nu)
    =
    2\sum_h\Gamma_{ih}^{xy,\star}(\bm x_i-\bm y_h).
\label{eq:app_sinkhorn_cross_gradient}
\end{equation}
For self-transport, transposition preserves feasibility and objective, so uniqueness makes the optimal self-coupling symmetric; since each representation appears in both arguments of the self-cost,
\begin{equation}
\begin{aligned}
    \nabla_{\bm x_i}\operatorname{OT}_{\varepsilon}(\mu,\mu)
    &=
    2\sum_\ell\Gamma_{i\ell}^{xx,\star}(\bm x_i-\bm x_\ell)
    +2\sum_\ell\Gamma_{\ell i}^{xx,\star}(\bm x_i-\bm x_\ell) \\
    &=
    4\sum_\ell\Gamma_{i\ell}^{xx,\star}(\bm x_i-\bm x_\ell).
\end{aligned}
\label{eq:app_sinkhorn_self_gradient}
\end{equation}
Combining the two gives, at the converged transport solutions~\cite{feydy2019interpolating},
\begin{equation}
\begin{aligned}
    \nabla_{\bm x_i}\mathcal S_{\varepsilon}(\mu,\nu)
    &=
    2\sum_h\Gamma_{ih}^{xy,\star}(\bm x_i-\bm y_h)
    -2\sum_\ell\Gamma_{i\ell}^{xx,\star}(\bm x_i-\bm x_\ell), \\
    \nabla_{\bm y_h}\mathcal S_{\varepsilon}(\mu,\nu)
    &=
    2\sum_i\Gamma_{ih}^{xy,\star}(\bm y_h-\bm x_i)
    -2\sum_k\Gamma_{hk}^{yy,\star}(\bm y_h-\bm y_k).
\end{aligned}
\label{eq:app_sinkhorn_representation_gradients}
\end{equation}

Finally, the alignment contribution for the fixed source-site pair and class is $\ell_{e,e',c}^{Q}=g_{e,e',c}\mathcal S_{\varepsilon}(\mu,\nu)$. By Eq.~\eqref{eq:pair_specific_representation}, the $j$-th block of $\bm x_i$ is $\sqrt{\overline\alpha_{e,e',c,j}/2}\,\bm r_{i,j}^{e,(0)}$, so the $j$-th block of Eq.~\eqref{eq:app_sinkhorn_representation_gradients} equals $\sqrt{\overline\alpha_{e,e',c,j}/2}$ times the transport residual
\begin{equation}
    \bm\Delta_{i,j}^{e,e',c}
    =
    \sum_h\Gamma_{ih}^{xy,\star}\big(\bm r_{i,j}^{e,(0)}-\bm r_{h,j}^{e',(0)}\big)
    -\sum_\ell\Gamma_{i\ell}^{xx,\star}\big(\bm r_{i,j}^{e,(0)}-\bm r_{\ell,j}^{e,(0)}\big).
\label{eq:app_sinkhorn_transport_residual}
\end{equation}
Applying the chain rule with the qualifications held fixed gives
\begin{equation}
\begin{gathered}
    \nabla_{\bm r_{i,j}^{e,(0)}}\ell_{e,e',c}^{Q}
    =
    g_{e,e',c}\,\overline\alpha_{e,e',c,j}\,\bm\Delta_{i,j}^{e,e',c}
    =
    \alpha_{e,e',c,j}\,\bm\Delta_{i,j}^{e,e',c}, \\
    \alpha_{e,e',c,j}
    =\pi_j^{\mathrm{task}}q^{\mathrm{pair}}_{e,e',c,j},
    \quad
    \sum_{j=1}^{J}\alpha_{e,e',c,j}=g_{e,e',c}.
\end{gathered}
\label{eq:app_sinkhorn_factor_gradient}
\end{equation}
The transport plans are shared across factors and set by the normalized qualifications through the ground cost, while the gradient reaching factor $j$ scales with its unnormalized qualification, and the scales sum to the overall qualification. Averaging over unordered pairs and classes recovers Eq.~\eqref{eq:qualification_guided_alignment}:
\begin{equation}
    \mathcal L_{\mathrm{env}}^{Q}
    =
    \frac{2}{E(E-1)}
    \sum_{e<e'}\frac12\sum_{c\in\{0,1\}}
    \ell_{e,e',c}^{Q}.
\label{eq:app_sinkhorn_loss_recovery}
\end{equation}

\section{Datasets and Preprocessing}~\label{app:data}
\vspace{-0.5cm}

\subsection{Dataset description}

We evaluate \method{} on four multi-site rs-fMRI datasets, retaining samples with diagnostic labels, site information, and usable ROI-level time series. Table~\ref{tab:dataset_statistics} summarizes the final analysis cohorts, with class counts listed in the order of the corresponding task labels.

\begin{table*}[t]
    \centering
    \small
    \setlength{\tabcolsep}{5pt}
    \renewcommand{\arraystretch}{1.10}
    \caption{Statistics of the datasets.}
    \label{tab:dataset_statistics}
    \vspace{-0.15cm}
    \resizebox{0.85\textwidth}{!}{
    \begin{tabular}{l |ccc |cc}
        \toprule
        Dataset & Task & Samples & Class counts & Parcellation & Time points \\
        \midrule
        ABIDE
        & ASD vs.\ TD
        & 1,025
        & 488 / 537
        & AAL116; CC200
        & $194.51 \pm 58.58$ \\

        REST-meta-MDD
        & MDD vs.\ HC
        & 2,428
        & 1,300 / 1,128
        & AAL116
        & $207.54 \pm 29.29$ \\

        SRPBS
        & MDD vs.\ HC
        & 1,000
        & 500 / 500
        & AAL116
        & 107--284 \\

        ABCD
        & ADHD vs.\ HC
        & 425
        & 213 / 212
        & AAL116
        & 383 \\
        \bottomrule
    \end{tabular}
    }
    \vspace{-0.2cm}
\end{table*}

\textbf{ABIDE.}
The Autism Brain Imaging Data Exchange (ABIDE)~\cite{abide} provides multi-site rs-fMRI data for studying autism spectrum disorder (ASD). We use the quality-controlled ABIDE I cohort with available phenotypic information and ROI-level time series. The final cohort contains 1,025 subjects, including 488 individuals with ASD and 537 typically developing (TD) controls from 20 sites. Acquisition sites differ in demographic composition and imaging protocols. Primary experiments use the AAL116 atlas~\cite{aal}, while an additional evaluation uses the CC200 parcellation~\cite{cc200} to examine performance under an alternative ROI definition.

\textbf{REST-meta-MDD.}
The REST-meta-MDD consortium~\cite{restmeatmdd} provides multi-site rs-fMRI data for studying major depressive disorder (MDD), collected across clinical centers in China. The final cohort comprises 2,428 participants from 25 acquisition sites, including 1,300 individuals with MDD and 1,128 healthy controls (HC). Each retained site includes both diagnostic groups. Site sizes range from 24 to 533 subjects, with the largest site accounting for approximately 22.0\% of the cohort. Among participants with available age information, the mean age is $36.24 \pm 15.07$ years, ranging from 12 to 82 years, with group means of $36.23 \pm 14.62$ years for MDD and $36.25 \pm 15.58$ years for HC. Records with available sex information include 1,454 females and 925 males. All subjects are represented using AAL116, with variable numbers of available time points.

\textbf{SRPBS.}
The Strategic Research Program for Brain Sciences (SRPBS) dataset~\cite{srpbs} is a multi-site neuroimaging resource collected in Japan, with demographic, diagnostic, and acquisition information. We use the MDD cohort for binary classification against HC. After quality control and cohort balancing, the analysis set contains 1,000 rs-fMRI subjects, comprising 500 MDD and 500 HC subjects from 11 acquisition environments. Ages range from 18 to 80 years, with an overall mean of $42.93 \pm 13.59$ years and group means of $42.66 \pm 12.09$ years for MDD and $43.21 \pm 14.95$ years for HC. The cohort includes 541 females and 459 males: 237 females and 263 males in the MDD group, and 304 females and 196 males in the HC group. ROI sequences use AAL116, with lengths ranging from 107 to 284 time points across acquisition environments.

\textbf{ABCD.}
The Adolescent Brain Cognitive Development (ABCD) Study~\cite{abcd} is a multi-site longitudinal study of brain development and mental health. We use baseline rs-fMRI data from \texttt{baseline\_year\_1\_arm\_1} for binary classification of attention-deficit/hyperactivity disorder (ADHD) versus HC. After quality control and site filtering, the final cohort contains 425 subjects from 10 imaging sites, including 213 individuals with ADHD and 212 HC. The mean age is $9.47 \pm 0.50$ years. All retained subjects are represented using AAL116 and have 383 retained time points.

\subsection{Data Preprocessing}
\label{app:preprocessing}

We apply a common FC estimation procedure across datasets and methods that take FC graphs as input. For each subject, frames containing non-finite values in any ROI are removed jointly across ROIs, and each retained ROI signal is standardized within subject. From the resulting sequence $\bm X_i^e\in\mathbb R^{T_i^e\times P}$, we compute pairwise Pearson correlations and apply the Fisher-$z$ transformation to obtain
\begin{equation}
    \bm A_i^{e,(0)}
    =
    \Psi\!\left(\bm X_i^e\right)
    \in\mathbb R^{P\times P},
\label{eq:app_full_scan_fc}
\end{equation}
where $T_i^e$ denotes the number of retained time points. Diagonal entries are set to zero. We use $P=116$ for AAL116 and $P=200$ for CC200.

To encode these graphs over a common set of ROI pairs, \method{} and variants sharing its backbone use a binary propagation support $\bm S$ constructed exclusively from training-site full-scan FC matrices. For each ROI pair, we compute the median absolute FC weight across subjects within each training site, followed by the median across training sites. Each ROI selects the $m=\lceil\rho_g(P-1)\rceil$ strongest off-diagonal connections, with $\rho_g=0.20$. We symmetrize the support by retaining an undirected connection whenever either endpoint selects the other. Constructed separately within each LOSO fold, $\bm S$ remains fixed across training, validation, and test subjects and all FC re-estimates. Each graph retains its own weights on this support, with positive weights and negative-weight magnitudes processed separately.

While full-scan graphs are used for parameter updates and inference, \method{} additionally constructs FC re-estimates from training subjects for qualification evaluation. We apply joint circular moving-block resampling with identical temporal indices across ROIs, preserving cross-ROI temporal correspondence and local temporal dependence within sampled blocks. Each resampled sequence $\bm X_i^{e,(k)}$ contains $T_i^e$ time points and is processed using the same Pearson-correlation and Fisher-$z$ procedure:
\begin{equation}
    \bm A_i^{e,(k)}
    =
    \Psi\!\left(\bm X_i^{e,(k)}\right),
    \qquad k=1,\ldots,K.
\label{eq:app_reestimated_fc}
\end{equation}
The block length is estimated once per LOSO fold from training subjects only. We draw a site- and class-balanced subset of at most 128 training subjects and, for each subject, randomly select up to 24 ROI signals and up to 48 pairwise ROI products. For each series, the autocorrelation horizon is the smallest lag at which the absolute autocorrelation stays below $\tau_{\mathrm{ac}}=0.10$ for two consecutive lags, searched over lags up to $\min(30,\lfloor T_i^e/2\rfloor)$. Horizons are aggregated by taking the $0.75$ quantile over series within each subject, the median over subjects within each site, and the median over training sites, so that every source site contributes equally to the fold-level value $\ell_b$. The block length applied to sequence $i$ is
\begin{equation}
    \ell_{b,i}^{e}
    =
    \operatorname{clip}\!\Big(
        \operatorname{round}(\ell_b),\;
        \ell_{\min},\;
        \min\!\big(\ell_{\max},\lfloor T_i^e/2\rfloor\big)
    \Big),
    \qquad
    \ell_{\min}=2,\quad \ell_{\max}=30.
\label{eq:app_block_length_range}
\end{equation}
The same estimation and clipping rules are applied across datasets.

\section{Baselines}~\label{app:baseline}

We compare \method{} with different competitive baselines across four categories:

\textbf{General GNNs.}
\begin{itemize}
    \item \textbf{GCN}~\citep{kipf2016gcn} updates node representations through degree-normalized neighborhood propagation, combining local features with information from adjacent nodes.
    \item \textbf{GAT}~\citep{petar2018gat} computes feature-dependent attention coefficients for neighboring nodes and combines multiple attention heads to learn node representations.
    \item \textbf{GIN}~\citep{xu2018gin} applies MLPs to summed neighborhood features and pools node embeddings to capture structural differences between graphs.
\end{itemize}

\textbf{General OOD methods.}
\begin{itemize}
    \item \textbf{CORAL}~\citep{sun2016coral} reduces covariance differences between learned representations. In our source-only setting, this regularization is applied across training sites.
    \item \textbf{IRM}~\citep{arjovsky2019irm} encourages a representation for which a shared classifier is simultaneously optimal across training environments, reducing reliance on environment-dependent associations.
\end{itemize}

\textbf{Graph OOD methods.}
\begin{itemize}
    \item \textbf{GSAT}~\citep{miao2022gsat} learns probabilistic edge selection through an information-bottleneck objective, suppressing graph information that is unnecessary for prediction. The resulting edge probabilities also provide interpretable subgraph explanations.
    \item \textbf{DisC}~\citep{fan2022disc} separately encodes causal and bias subgraphs, then recombines their representations to weaken spurious associations during training, particularly when graph labels correlate strongly with bias substructures.
    \item \textbf{CEPG}~\citep{wang2026cepg} leverages environmental information to guide invariant graph representation learning and improve prediction under distribution shifts by emphasizing predictive features shared across different training environments.
    \item \textbf{DiSCO}~\citep{sun2026disco} trains diverse subgraph experts and uses sparse gating to select their contributions for each graph, accommodating heterogeneous predictive structures. Different experts capture complementary causal patterns across graph instances.
\end{itemize}

\textbf{Brain network methods.}
\begin{itemize}
    \item \textbf{BrainNetTF}~\citep{kan2022brainnttf} learns interactions among ROI connectivity profiles through Transformer attention and forms graph embeddings using orthonormal clustering readout, which groups ROIs according to their learned functional representations.
    \item \textbf{XG-GNN}~\citep{qiu2024towards} constructs task-oriented nonlinear functional networks and uses meta-learning with explanation regularization to preserve diagnostic patterns across acquisition sites, jointly considering diagnostic accuracy and the consistency of explanations.
    \item \textbf{AGMGC}~\citep{noman2025agmgc} combines learned connectivity with correlation-based graphs, integrating spline and multi-graph convolutions with contrastive regularization to capture local connectivity patterns and broader interregional dependencies.
    \item \textbf{FC-HGNN}~\citep{gu2025fchgnn} first encodes individual connectomes using hemisphere-aware convolutions, then combines imaging and non-imaging information through heterogeneous population-graph aggregation. This connects individual connectivity patterns with relationships among subjects.
    \item \textbf{BrainOOD}~\citep{xu2025brainood} filters node features and regularizes extracted graph structures through information-bottleneck and structure-consistency objectives for cross-site prediction. The consistency objective encourages similar structure selection across subjects.
    \item \textbf{DeCI}~\citep{yu2026deci} separates ROI time series into oscillatory and slowly varying components, processes each ROI independently, and aggregates the resulting predictions, allowing temporal dynamics to inform diagnosis beyond static correlations.
    \item \textbf{CORE}~\citep{wang2026brain} extracts a reproducible connectivity scaffold after site-specific deconfounding and models pathway dynamics on its line graph using subject-adaptive gating to combine population-level connectivity priors with individual temporal characteristics.
\end{itemize}

\section{Algorithm}

The overall training and inference process of the proposed \method{} is shown in Algorithm~\ref{alg:overall}. 

\input{table/ID}

\section{Implementation Details}
\label{app:implementation}

We implement \method{} in PyTorch~\footnote{https://pytorch.org/}. The encoder is a two-layer signed GNN with hidden dimension 64 and 16-dimensional ROI identity embeddings. Positive and negative-magnitude adjacency matrices are symmetrically normalized separately. Each layer sums separate bias-free transformations of self-features and messages from both channels, followed by GELU, residual addition, and LayerNorm. We use $J=4$ connectome factors with dimension $d_t=16$. The shared ROI assignment uses 32-dimensional queries and keys, scaled dot-product similarities with temperature $0.10$, and 60 log-domain Sinkhorn iterations targeting row and column masses of $1/J$ and $1/P$, respectively. The factor mapping $\mathcal F_{\phi}$ consists of a bias-free $64\!\rightarrow\!16$ projection followed by a $16\!\rightarrow\!32\!\rightarrow\!16$ MLP with GELU between its linear layers. Both the projection and MLP outputs are $\ell_2$-normalized. We train the model using AdamW with learning rate $3\times10^{-4}$ and weight decay $10^{-4}$. Each LOSO fold reserves one site for testing and two for validation, while model fitting and estimation of training statistics use only the remaining sites. Results are reported as mean $\pm$ standard deviation across held-out test sites.

\section{Complexity Analysis}

We analyze the per-step cost of \method{} after supervised warm-up, where the support $\bm S$ and assignment $\overline{\bm\Pi}$ are fixed. Let $B$ be the batch size, $E$ the number of source sites, $N_{\mathrm{src}}=\sum_e N_e$ the number of source training subjects, $P$ the number of ROIs, $T$ the number of time points, $s=\|\bm S\|_0$ the number of retained adjacency entries, $J$ the number of factors, $L$ the number of GNN layers, and $d=\max(d_h,d_t)$. With sparse propagation, encoding and factorization cost $C_{\mathrm{enc}}=L(sd+Pd^2)+PJd+Jd^2$ per graph. For batches balanced across sites and classes, forming pair-specific representations and cost matrices takes $\mathcal O(B^2Jd)$, and $I_{\mathrm{sk}}$ Sinkhorn iterations over all cross- and self-transport problems take $\mathcal O(B^2I_{\mathrm{sk}})$. Every $M$ steps, the qualification refresh generates $K$ re-estimates per source subject at $C_{\mathrm{FC}}=\mathcal O(P^2T)$ each, encodes $K+1$ graphs per subject, and forms qualifications over site pairs in $\mathcal O(E^2J)$. The amortized time per step is therefore
$\mathcal O\!\left((B+N_{\mathrm{src}}(K+1)/M)\,C_{\mathrm{enc}}+B^2(Jd+I_{\mathrm{sk}})+(N_{\mathrm{src}}KC_{\mathrm{FC}}+E^2J)/M\right)$.
Relative to class-conditional alignment without qualification, the overhead is confined to the $1/M$ terms. Since the refresh runs without gradients and discards re-estimates after use, peak memory is that of a training batch, $\mathcal O(BL(Pd+s)+B^2)$, independent of $N_{\mathrm{src}}$ and $K$. 

\section{More results}

\begin{figure*}[t]
\centering
\captionsetup[subfigure]{
    font=small,
    skip=3pt,
    justification=centering
}

\begin{subfigure}[b]{0.245\linewidth}
    \centering
    \includegraphics[width=\linewidth]{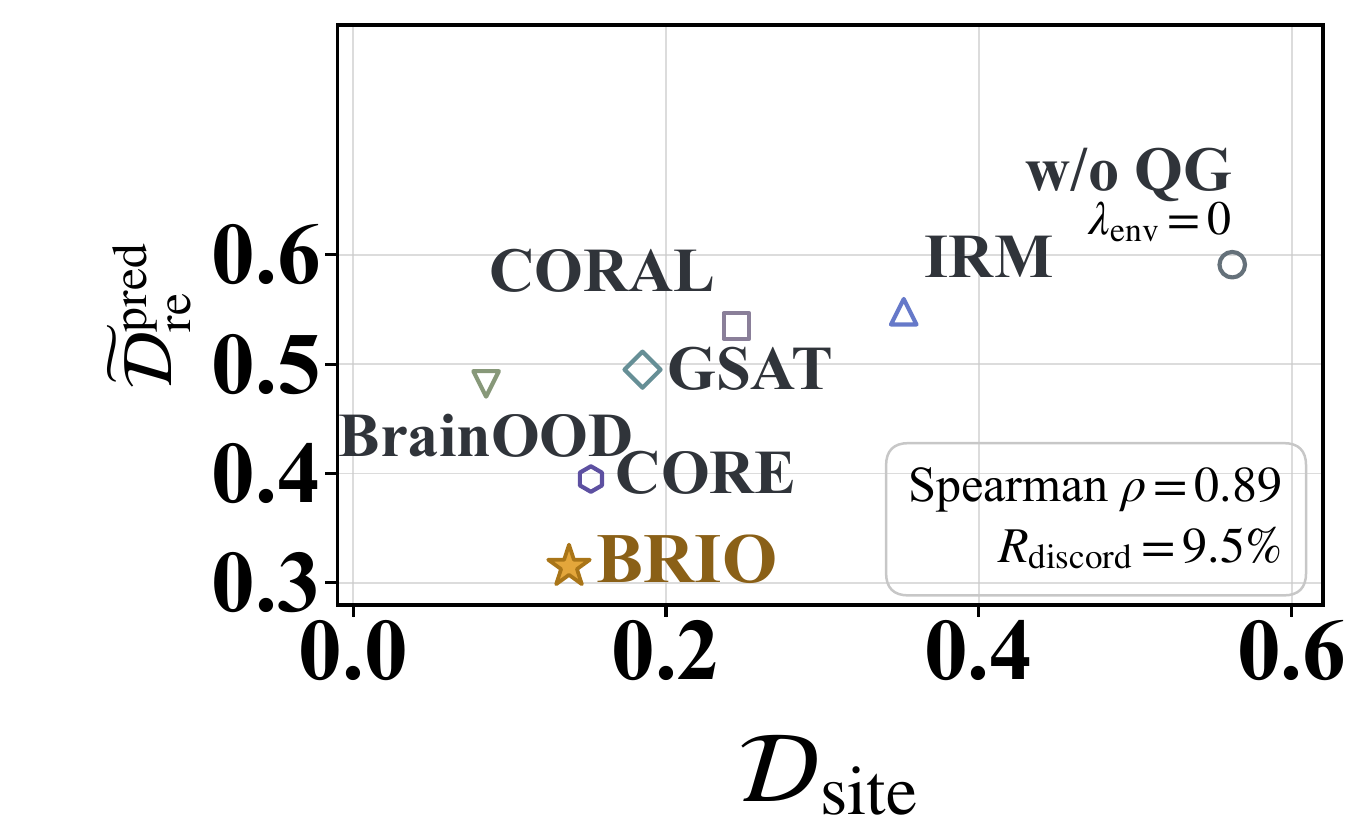}
    \caption{Cross-method}
    
\end{subfigure}\hfill
\begin{subfigure}[b]{0.245\linewidth}
    \centering
    \includegraphics[width=\linewidth]{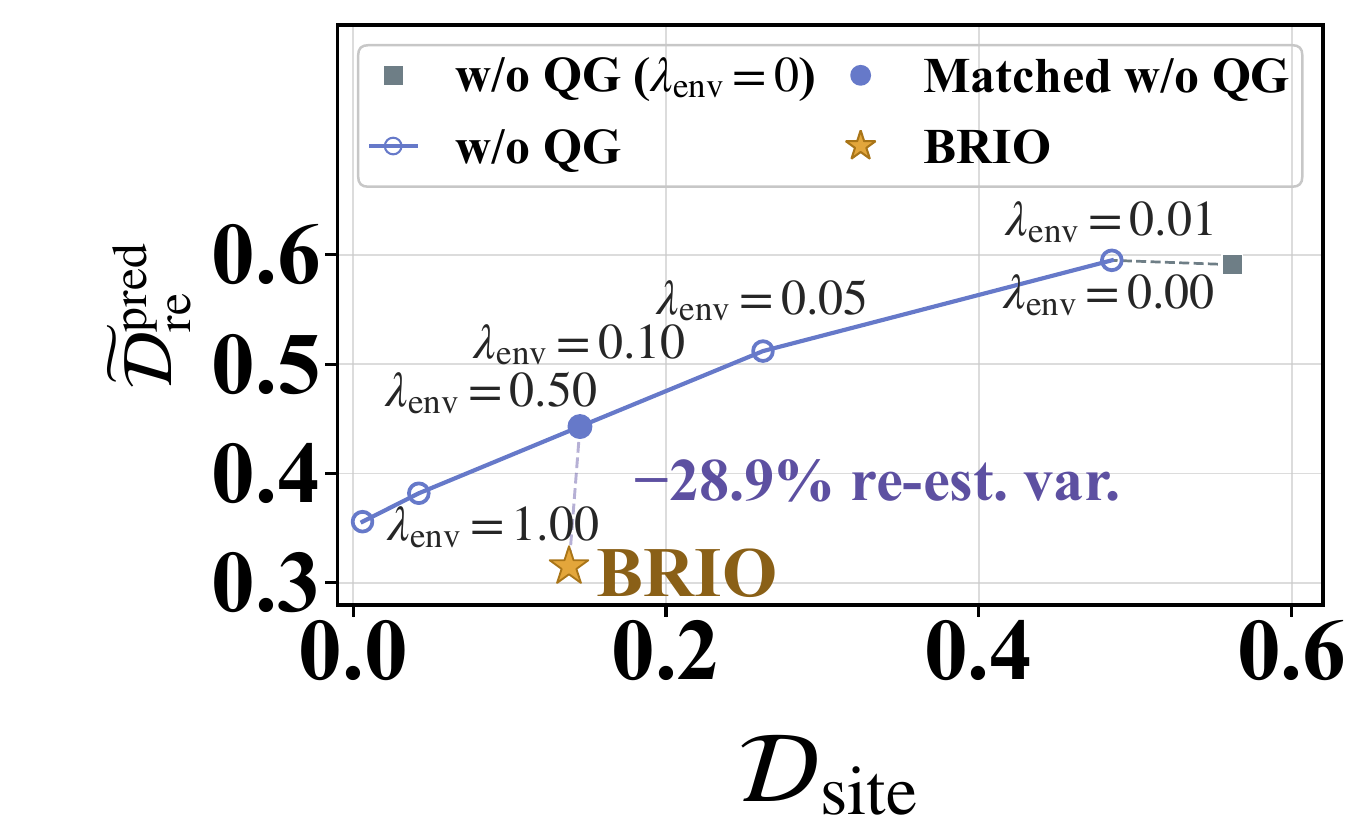}
    \caption{Alignment}
    
\end{subfigure}\hfill
\begin{subfigure}[b]{0.245\linewidth}
    \centering
    \includegraphics[width=\linewidth]{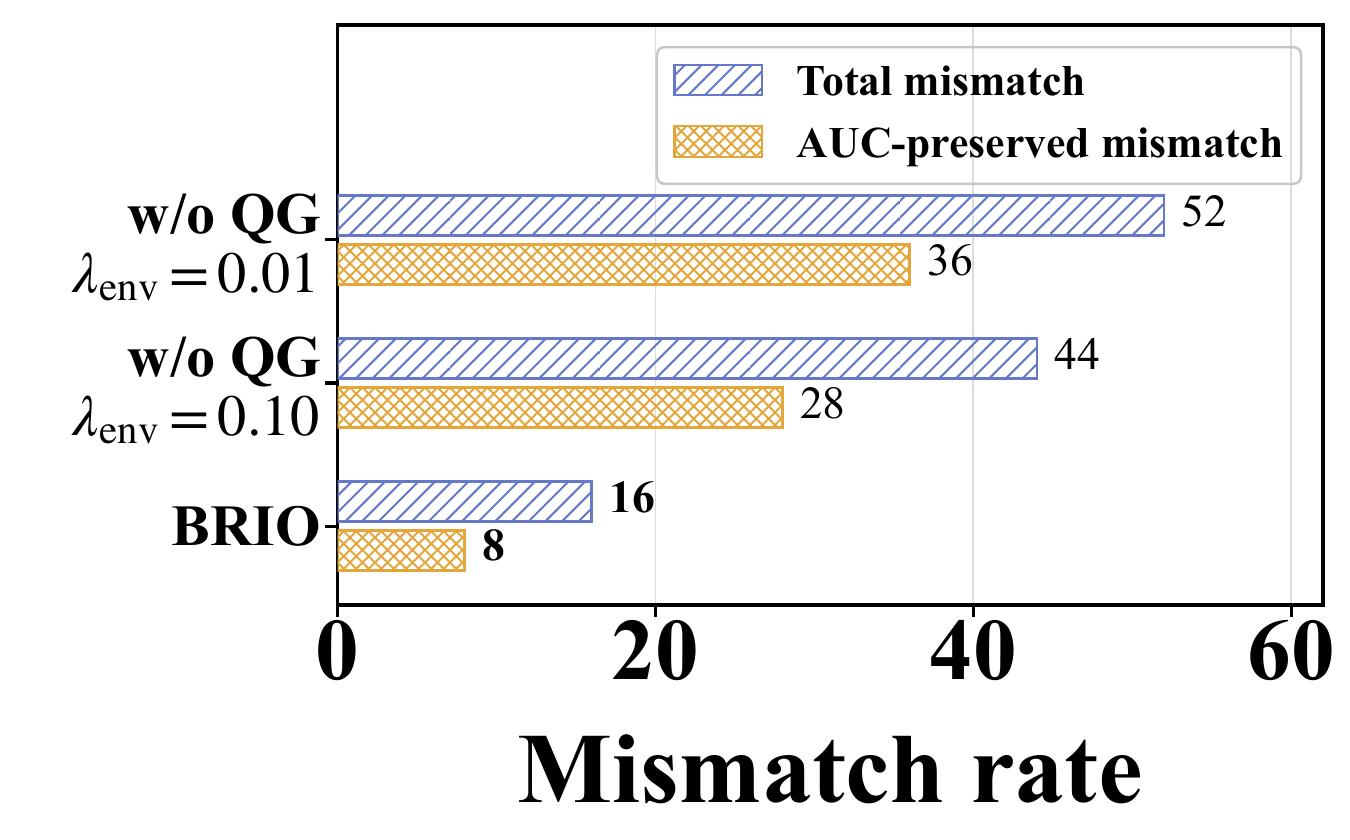}
    \caption{Mismatch}
    
\end{subfigure}\hfill
\begin{subfigure}[b]{0.245\linewidth}
    \centering
    \includegraphics[width=\linewidth]{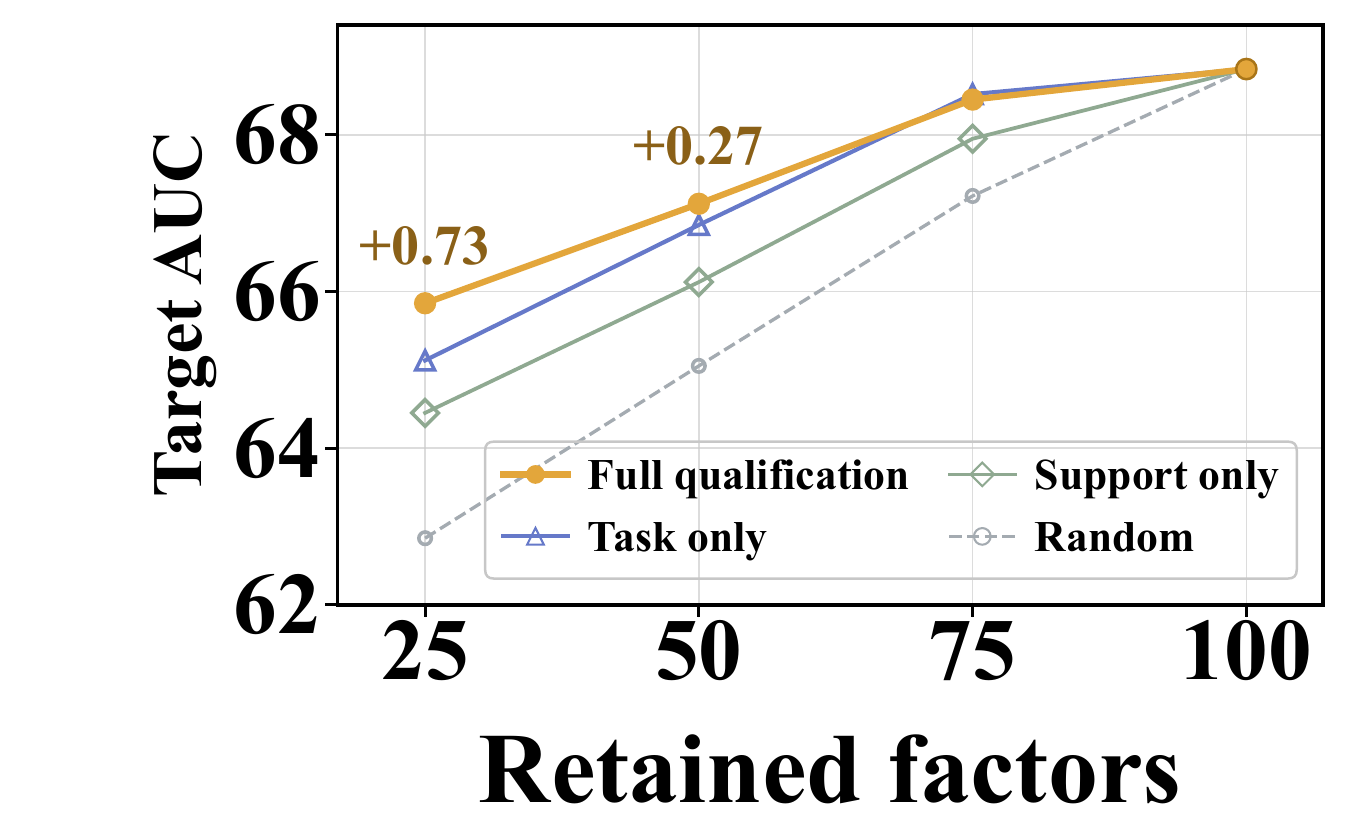}
    \caption{Factor retention}
    
\end{subfigure}
\vspace{-0.2cm}
\caption{Cross-method relationships (a), alignment trajectories (b),
mismatch rates (c), and factor retention (d) on REST-meta-MDD.}
\vspace{-0.2cm}
\label{fig:diagnostics-retention-mdd}
\end{figure*}

\begin{figure*}[t]
\centering
\captionsetup[subfigure]{
    font=small,
    skip=3pt,
    justification=centering
}

\begin{subfigure}[b]{0.245\linewidth}
    \centering
    \includegraphics[width=\linewidth]{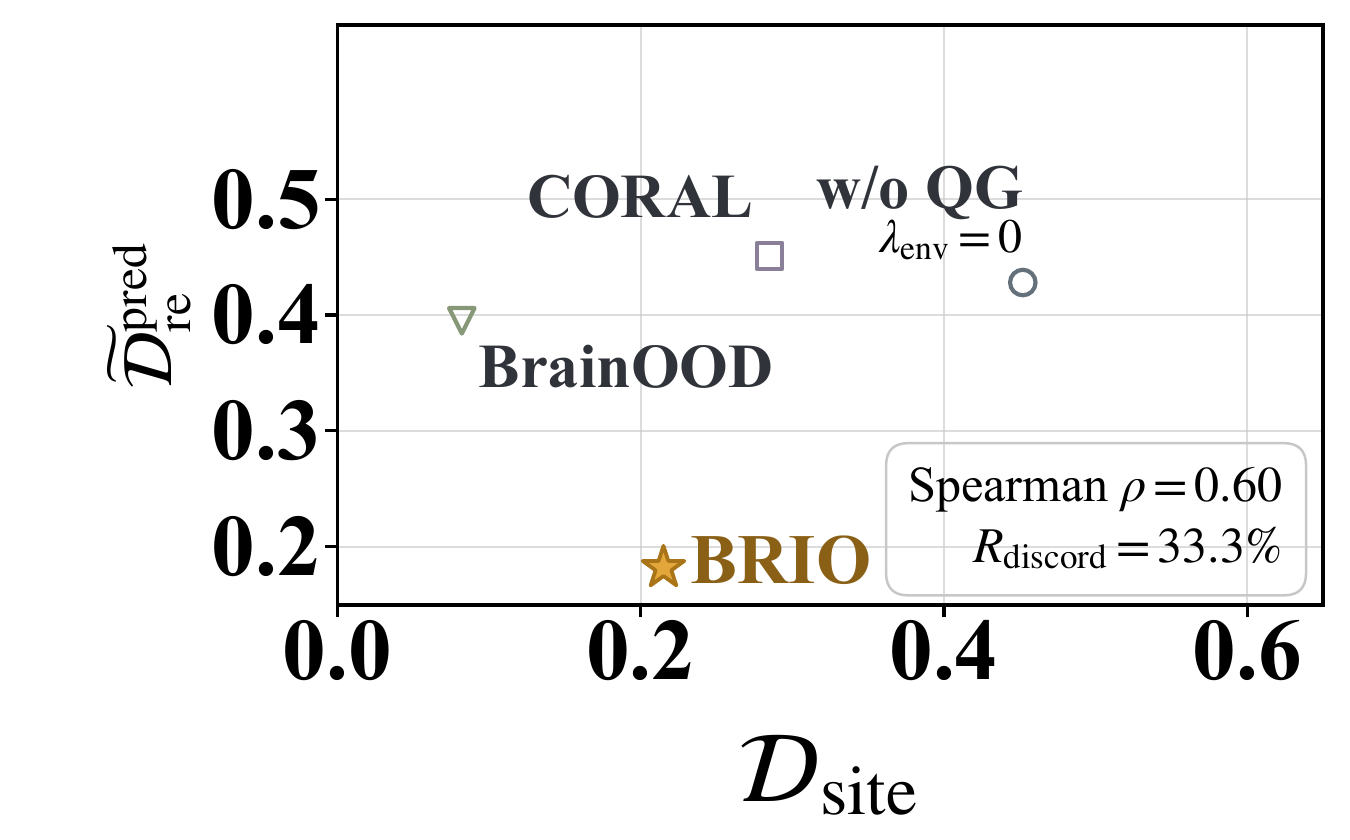}
    \caption{Cross-method}
   
\end{subfigure}\hfill
\begin{subfigure}[b]{0.245\linewidth}
    \centering
    \includegraphics[width=\linewidth]{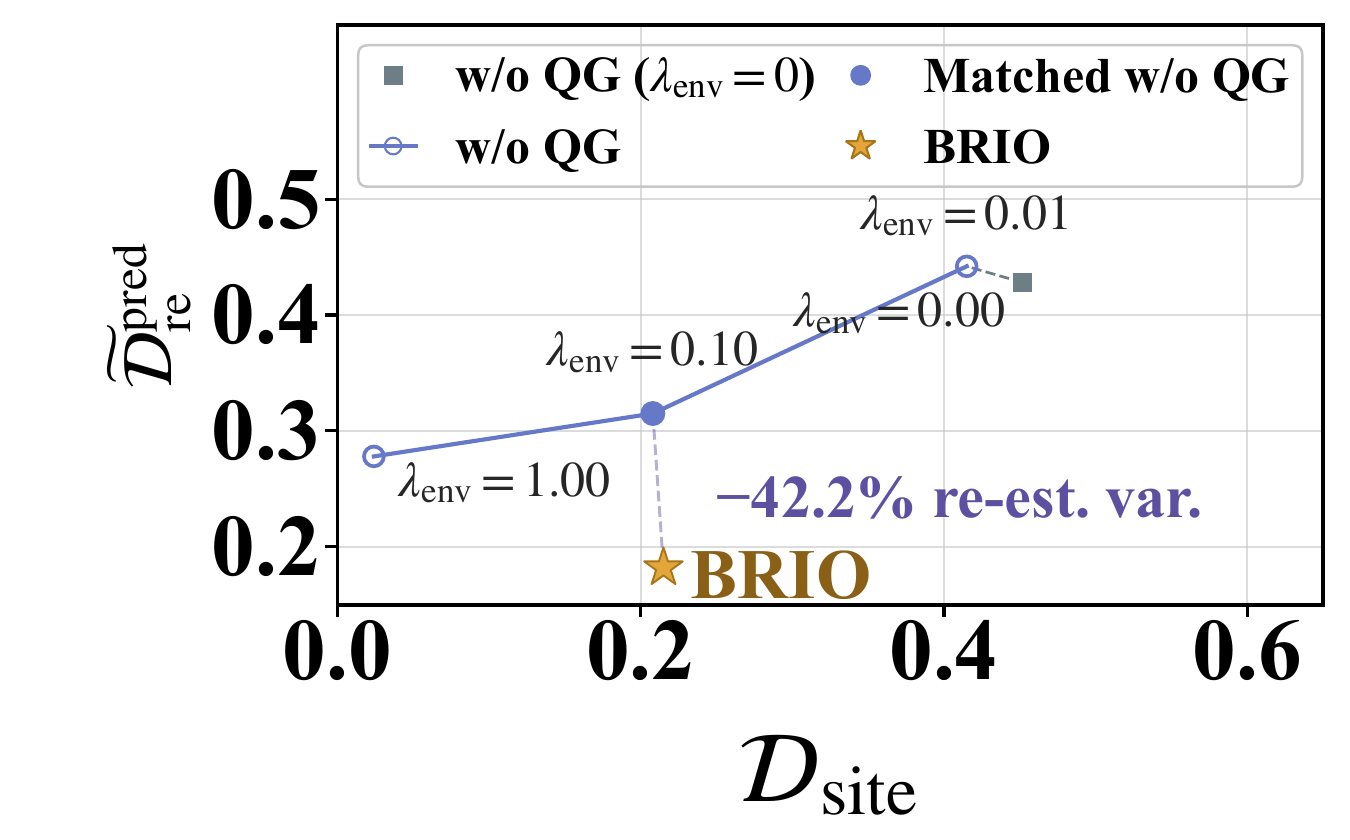}
    \caption{Alignment}
    
\end{subfigure}\hfill
\begin{subfigure}[b]{0.245\linewidth}
    \centering
    \includegraphics[width=\linewidth]{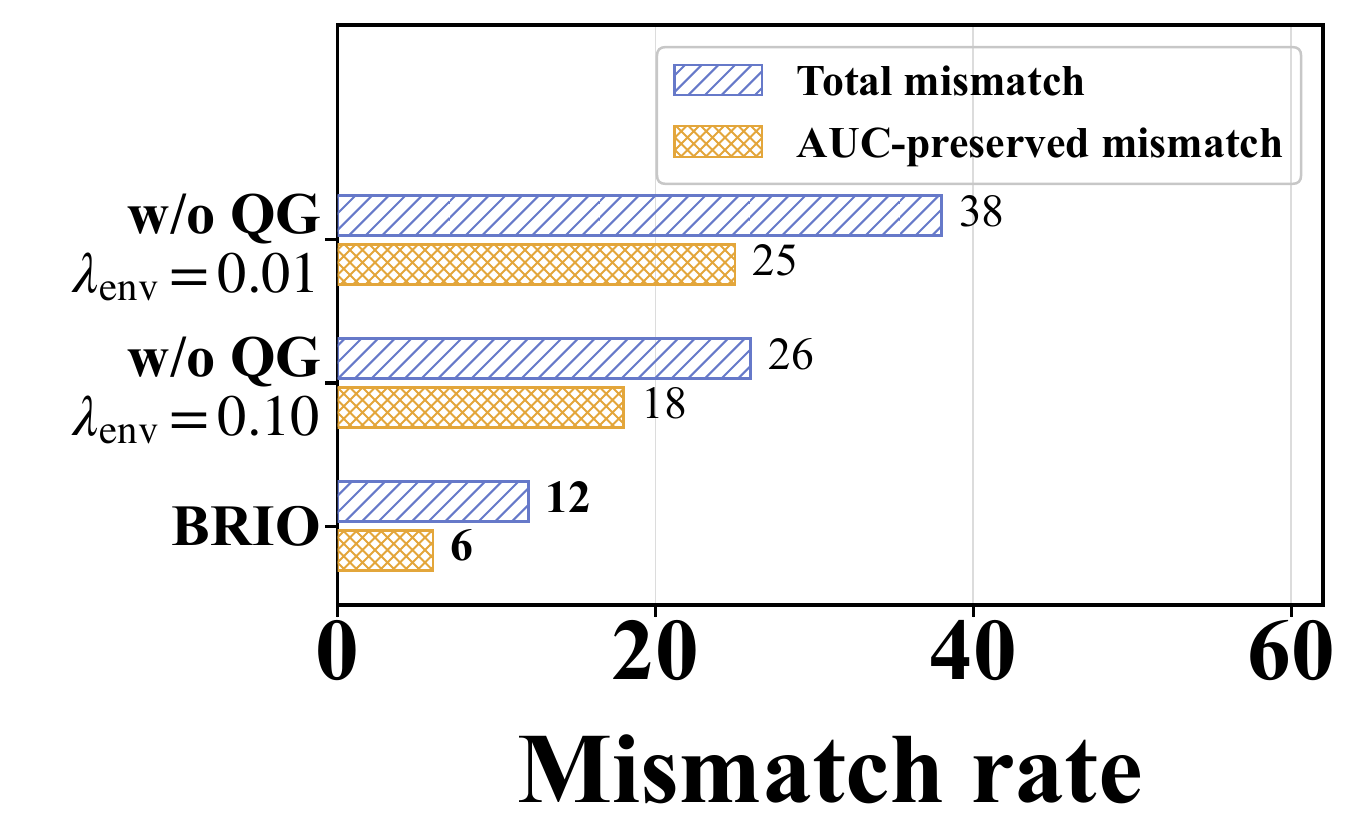}
    \caption{Mismatch}
    
\end{subfigure}\hfill
\begin{subfigure}[b]{0.245\linewidth}
    \centering
    \includegraphics[width=\linewidth]{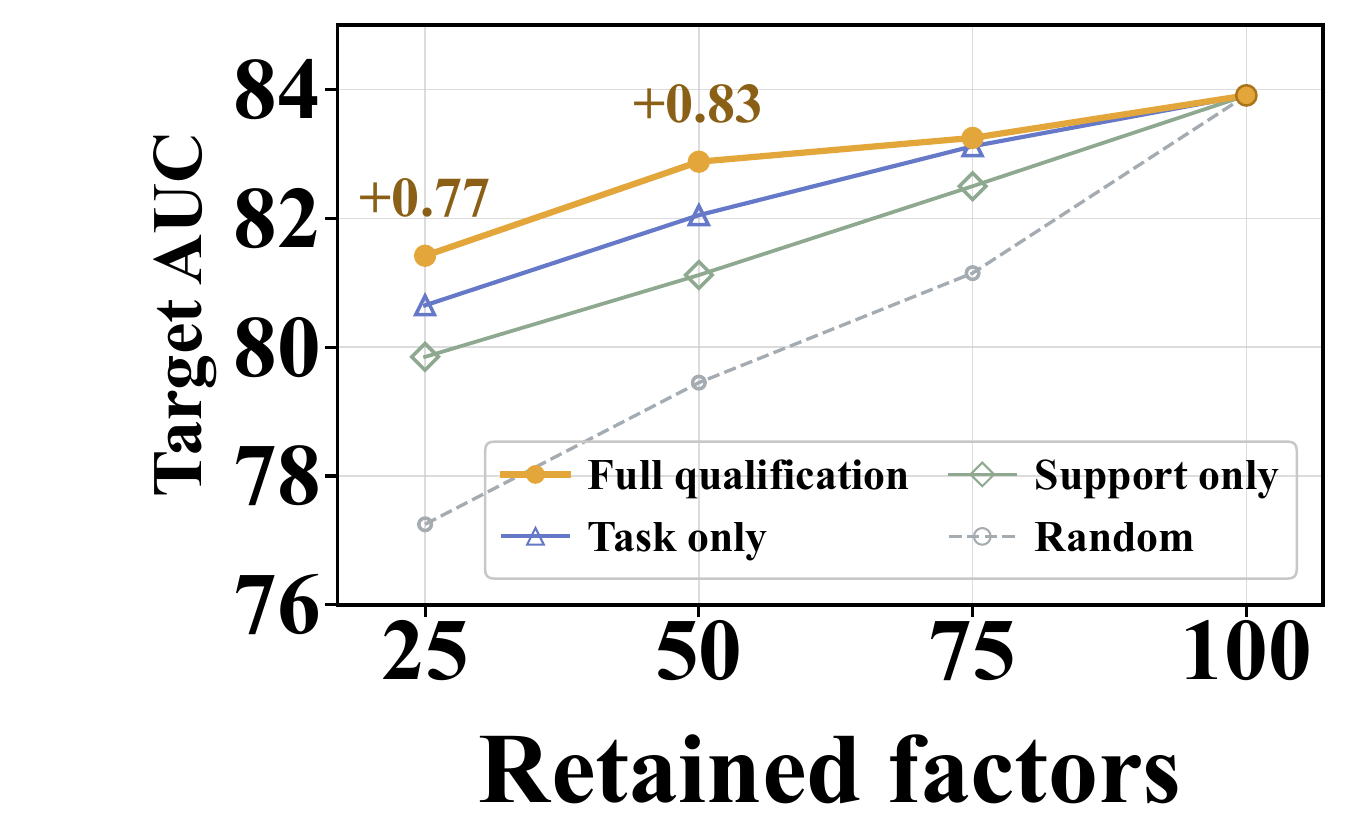}
    \caption{Factor retention}
    
\end{subfigure}
\vspace{-0.2cm}
\caption{Cross-method relationships (a), alignment trajectories (b),
mismatch rates (c), and factor retention (d) on SRPBS.}
\vspace{-0.6cm}
\label{fig:diagnostics-retention-srpbs}
\end{figure*}

\begin{figure*}[t]
\centering
\captionsetup[subfigure]{
    font=small,
    skip=3pt,
    justification=centering
}

\begin{subfigure}[b]{0.245\linewidth}
    \centering
    \includegraphics[width=\linewidth]{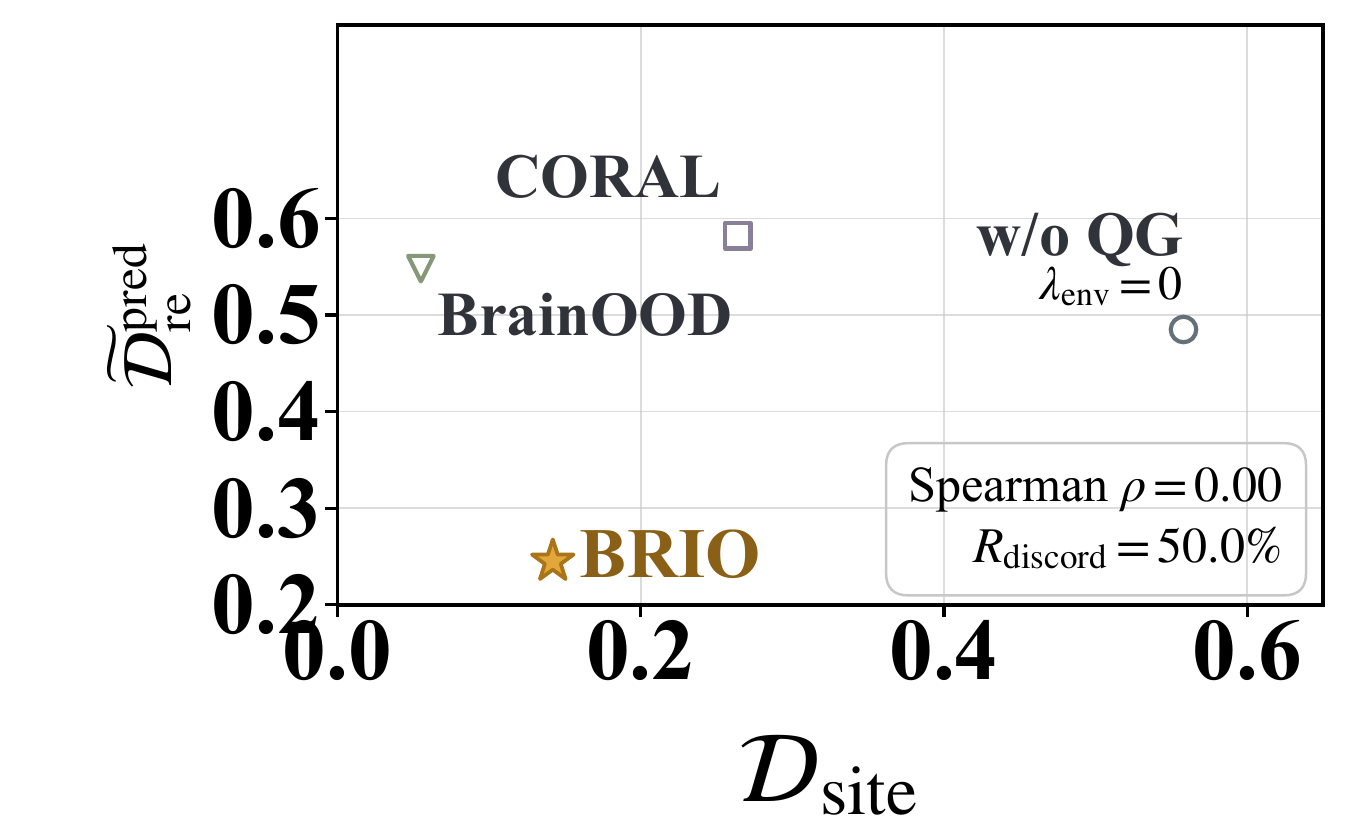}
    \caption{Cross-method}
    
\end{subfigure}\hfill
\begin{subfigure}[b]{0.245\linewidth}
    \centering
    \includegraphics[width=\linewidth]{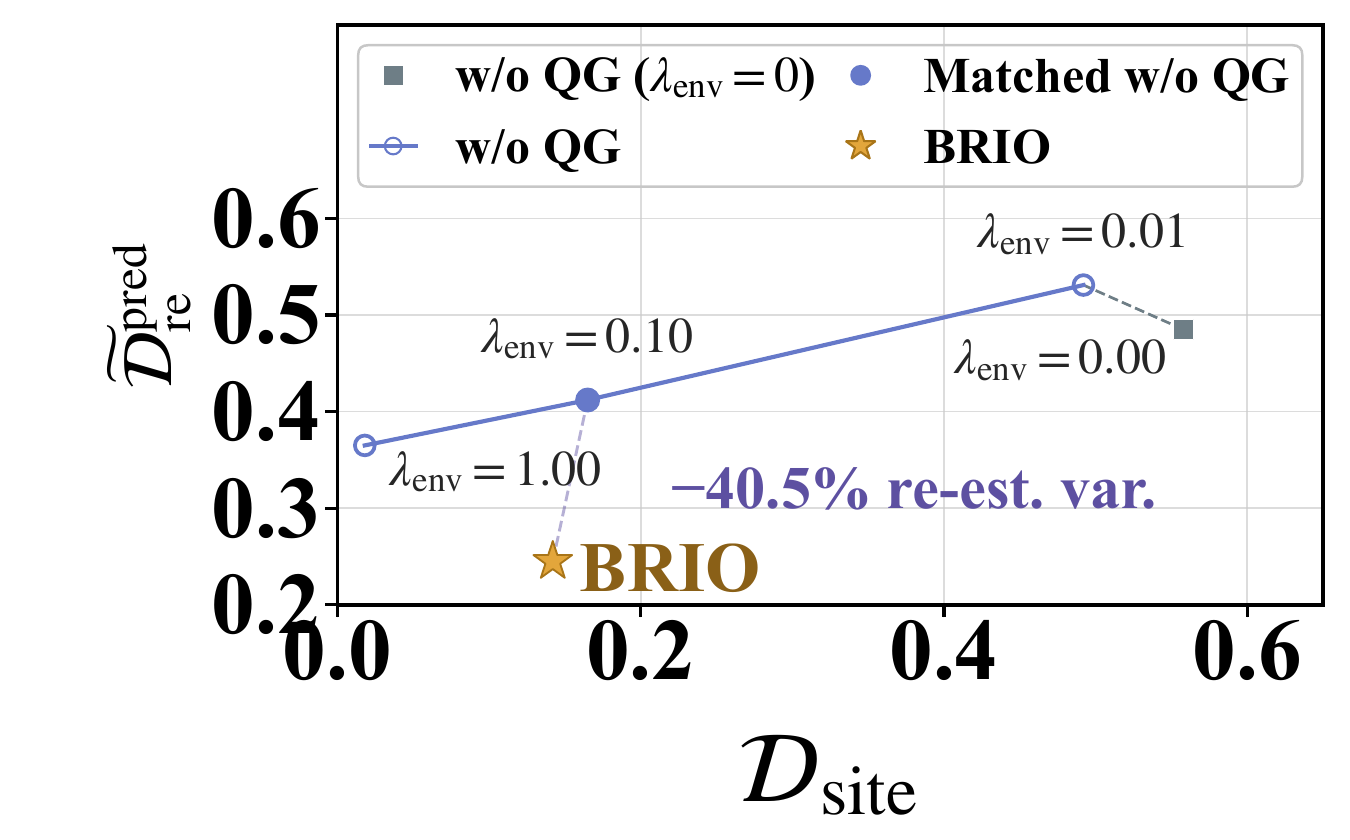}
    \caption{Alignment}

\end{subfigure}\hfill
\begin{subfigure}[b]{0.245\linewidth}
    \centering
    \includegraphics[width=\linewidth]{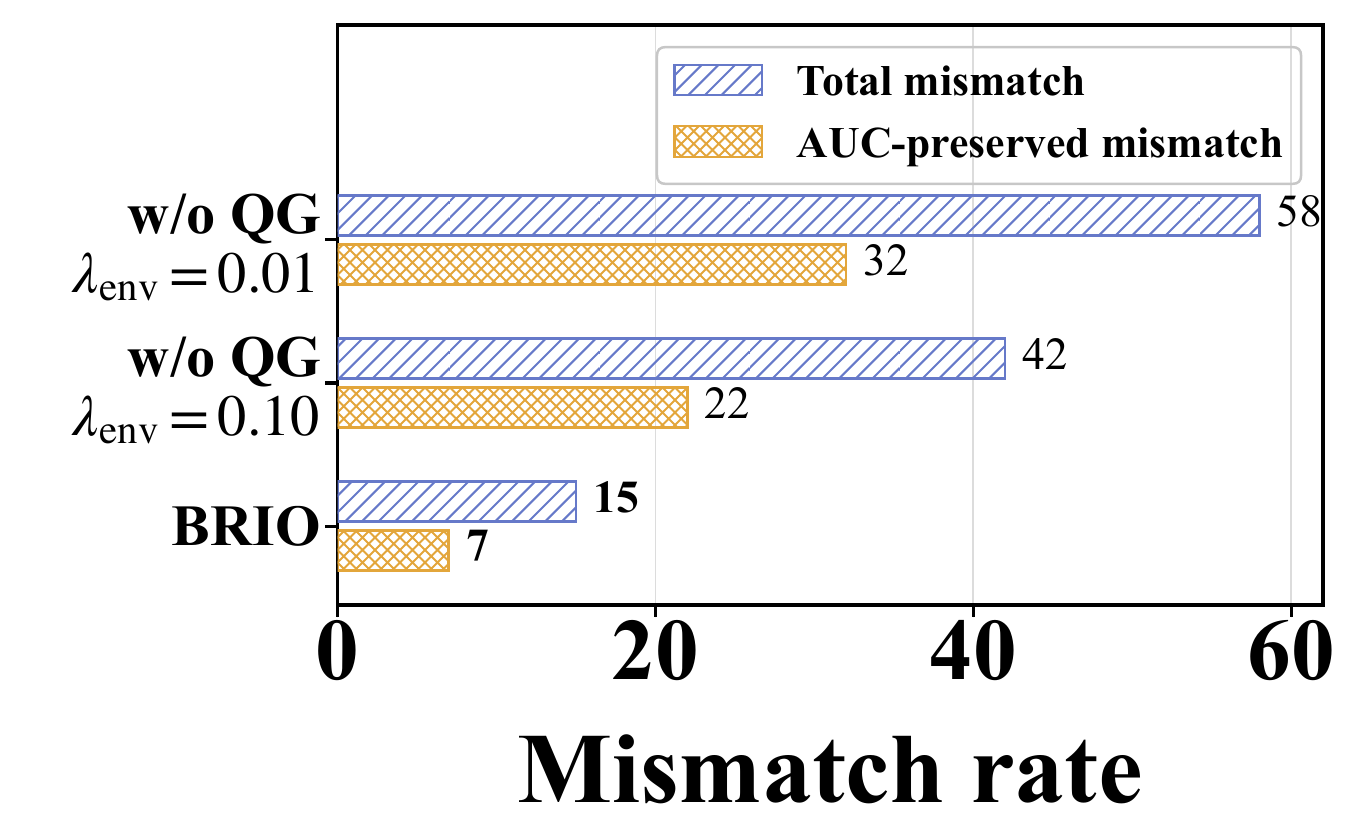}
    \caption{Mismatch}

\end{subfigure}\hfill
\begin{subfigure}[b]{0.245\linewidth}
    \centering
    \includegraphics[width=\linewidth]{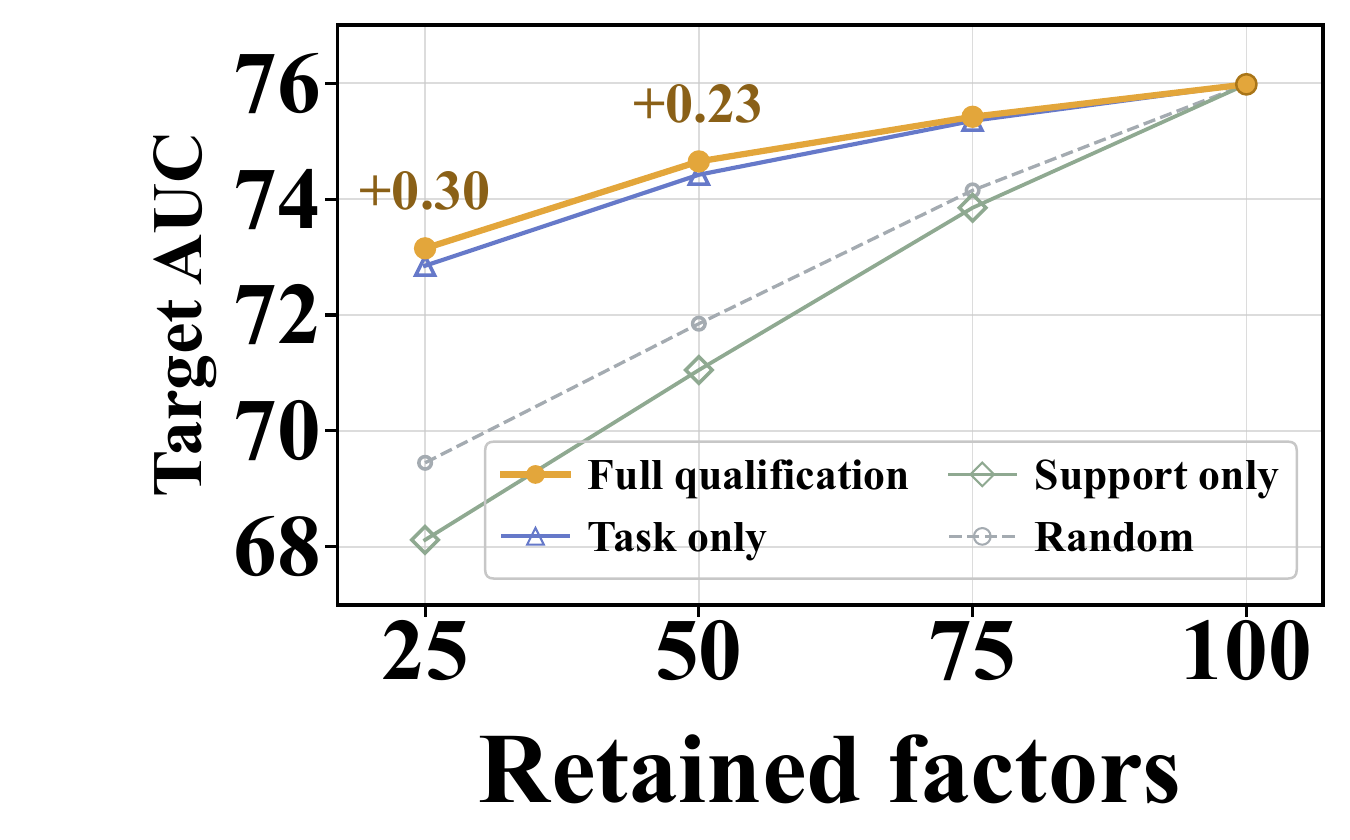}
    \caption{Factor retention}
    
\end{subfigure}
\vspace{-0.2cm}
\caption{Cross-method relationships (a), alignment trajectories (b),
mismatch rates (c), and factor retention (d) on ABCD.}
\vspace{-0.6cm}
\label{fig:diagnostics-retention-abcd}
\end{figure*}

\subsection{ID Performance Comparison}

In this part, we assess whether \method{} also performs well under in-distribution (ID) settings. Subjects from all sites are pooled and partitioned into ten folds stratified by site and class, so that every fold preserves the site and class composition of the full cohort. Each fold in turn serves as the test set, the next two folds serve as the validation set, and the remaining seven folds form the training set, giving a 7:2:1 rotation in which every subject is tested exactly once. All sites appear in training and every training site contains both classes, so source sites continue to serve as environments for alignment. As shown in Table~\ref{tab:main_results_id}, \method{} ranks first or second on both metrics across all five settings. It achieves the highest mean AUC on ABIDE, ABIDE (CC200), REST-meta-MDD, and ABCD, and the highest mean ACC on REST-meta-MDD and ABCD. BrainNetTF and DeCI obtain the highest ACC on ABIDE and ABIDE (CC200), respectively, while CORE leads both metrics on SRPBS, with \method{} closely following in each case. These results show that the predictive value of re-estimation-informed learning extends beyond unseen-site evaluation and supports competitive ID performance alongside the cross-site OOD results.

\subsection{More Site Discrepancy and Re-estimation Variation}
\label{app:diagnostics}

In this part, we define the two evaluation metrics used in Sec.~\ref{sec:case_study} and extend the cross-method comparison and alignment-strength sweep to REST-meta-MDD, SRPBS, and ABCD. The cross-method comparison includes CORAL, IRM, GSAT, BrainOOD, CORE, \method{} w/o QG with $\lambda_{\mathrm{env}}=0$, and \method{} on ABIDE and REST-meta-MDD, and CORAL, BrainOOD, \method{} w/o QG, and \method{} on SRPBS and ABCD. Each method uses the checkpoint selected in the main experiment and is evaluated with all trained parameters frozen and stochastic operations disabled. All methods receive the same source training subjects, full-scan FC graphs, and FC re-estimates, and each observation corresponds to one method, held-out site, and random seed.

For each method, let $\widehat P_{e,c}$ denote the empirical distribution of unit-normalized full-scan representations from source site $e$ and class $c$, taken from the layer preceding the classifier. The class-conditional site discrepancy is
\begin{equation}
    \mathcal D_{\mathrm{site}}
    =
    \frac{1}{E(E-1)}
    \sum_{1\leq e<e'\leq E}
    \sum_{c\in\{0,1\}}
    \mathcal S_{\varepsilon}
    \left(\widehat P_{e,c},\widehat P_{e',c}\right),
    \label{eq:app_site_discrepancy}
\end{equation}
where $E$ is the number of source sites. The debiased Sinkhorn divergence $\mathcal S_{\varepsilon}$ uses squared Euclidean distances and identical solver settings across methods, with the same representation in cross- and self-transport terms and without qualification-based weighting. Lower values indicate closer full-scan representations across source sites under a common geometry.

To measure predictive changes, let $\ell_i^{(k)}$ be the binary logit margin, given by the pre-sigmoid logit or the difference between the two softmax logits; a common shift of two logits leaves the predicted probability unchanged and is thus excluded. With $k=0$ denoting the full-scan prediction, the normalized predictive re-estimation variation is
\begin{equation}
    \widetilde{\mathcal D}_{\mathrm{re}}^{\mathrm{pred}}
    =
    \frac{1}{2E}
    \sum_{e=1}^{E}\sum_{c\in\{0,1\}}
    \frac{
        \widehat{\mathbb E}_{i,k\mid e,c}
        \left[
            \left(\ell_i^{(k)}-\ell_i^{(0)}\right)^2
        \right]
    }{
        V_{e,c}^{\mathrm{pred}}
        +\frac14
        \left(\mu_{e,1}^{\mathrm{pred}}-\mu_{e,0}^{\mathrm{pred}}\right)^2
        +\epsilon
    },
    \label{eq:app_predictive_variation}
\end{equation}
where the empirical expectation averages over source subjects and their re-estimates with $k\geq1$, $\mu_{e,c}^{\mathrm{pred}}$ and $V_{e,c}^{\mathrm{pred}}$ are the full-scan logit mean and variance within each site and class, and $\epsilon>0$ ensures numerical stability. Numerator and denominator scale identically under a global rescaling of the logits, so the ratio measures re-estimation change relative to the within-site spread and class separation of full-scan predictions.

This output-level variation and the factor-level re-estimation support in Sec.~\ref{sec:qualification} capture different aspects of stability. Since $z_i=J^{-1/2}\sum_j a_{i,j}+b$, a re-estimation change of the logit satisfies
\begin{equation}
    \mathbb E\!\left[(\delta z_i)^2\right]
    =
    \frac{1}{J}\Big[
        \sum_{j}\mathbb E\!\left[(\delta a_{i,j})^2\right]
        +2\sum_{j<l}\mathbb E\!\left[\delta a_{i,j}\,\delta a_{i,l}\right]
    \Big],
\end{equation}
so changes of individual factors may add or cancel at the output, and the two quantities are reported separately.

The cross-method plots show method-level means, with Spearman $\rho$ computed across individual runs. $R_{\mathrm{discord}}$ is the percentage of method pairs within the same fold and seed for which smaller $\mathcal D_{\mathrm{site}}$ corresponds to larger $\widetilde{\mathcal D}_{\mathrm{re}}^{\mathrm{pred}}$, excluding ties. A run is a mismatch when its $\mathcal D_{\mathrm{site}}$ falls below and its $\widetilde{\mathcal D}_{\mathrm{re}}^{\mathrm{pred}}$ falls above the pooled dataset medians. The total mismatch rate is the percentage of mismatches among all evaluated runs, and the AUC-preserved mismatch rate counts only mismatches whose source-validation AUC is no more than a tolerance $\delta_{\mathrm{AUC}}$ below that of \method{} w/o QG with $\lambda_{\mathrm{env}}=0$, again over all evaluated runs.

\begin{figure*}
\begin{subfigure}[b]{0.33\linewidth}
    \centering
    \includegraphics[width=\linewidth]{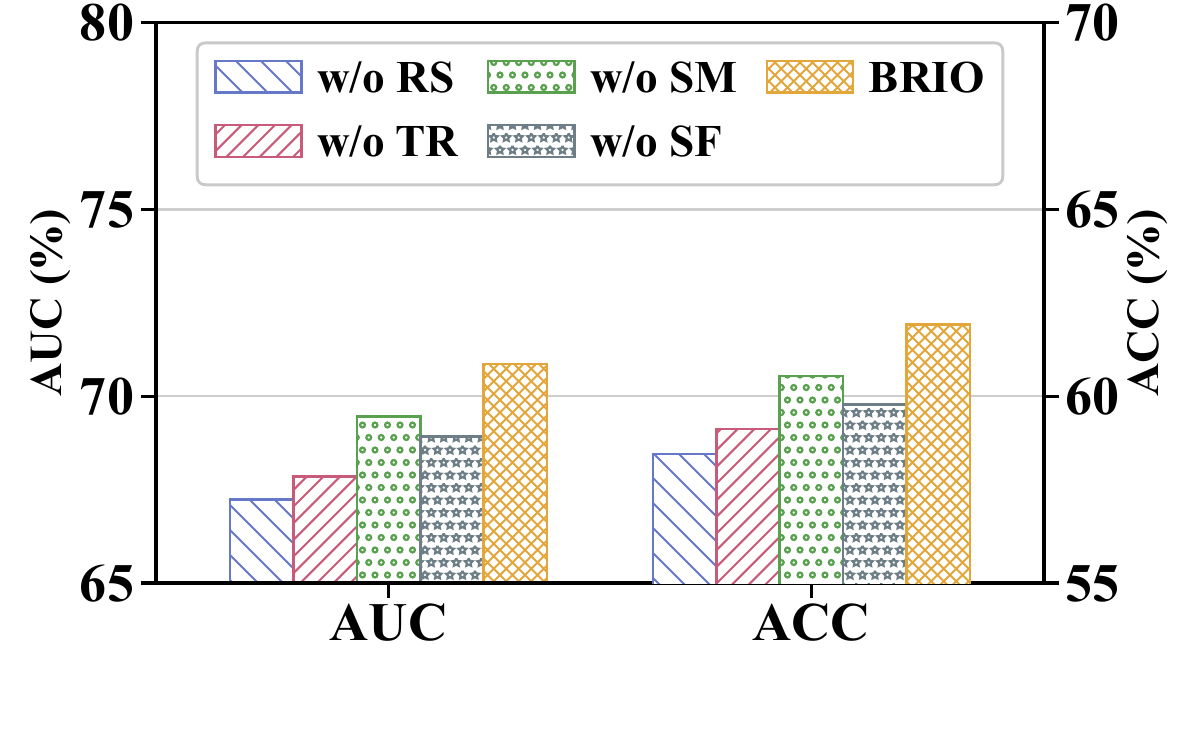}

    \caption{ABIDE (CC200)}
    \label{fig:ablation_cc200}
\end{subfigure}\hfill
\begin{subfigure}[b]{0.33\linewidth}
    \centering
    \includegraphics[width=\linewidth]{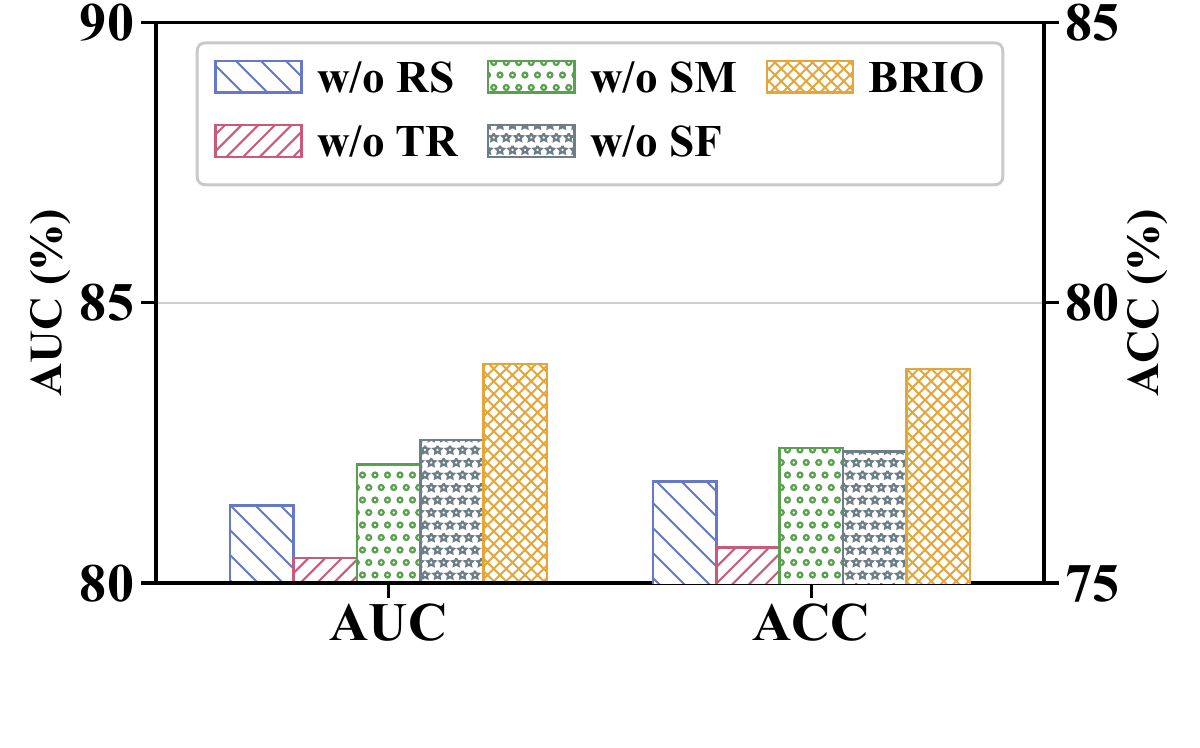}

    \caption{SRPBS}
    \label{fig:ablation_srbps}
\end{subfigure}\hfill
\begin{subfigure}[b]{0.33\linewidth}
    \centering
    \includegraphics[width=\linewidth]{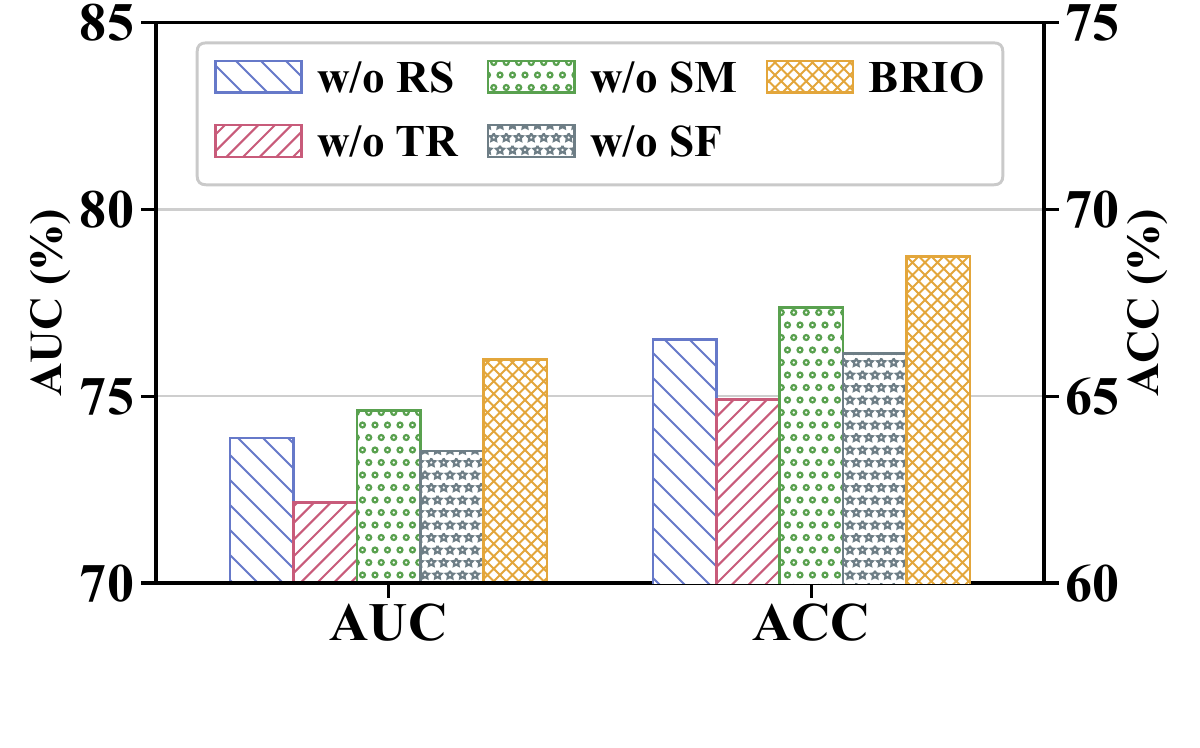}

    \caption{ABCD}
    \label{fig:ablation_abcd}
\end{subfigure}
\vspace{-0.2cm}
\caption{Ablation study on ABIDE (CC200), SRPBS, and ABCD.}
\vspace{-0.6cm}
\label{fig:more_ablation}
\end{figure*}

\begin{figure*}[t]
\begin{subfigure}[t]{0.33\linewidth}
    \centering
    \includegraphics[width=\linewidth]{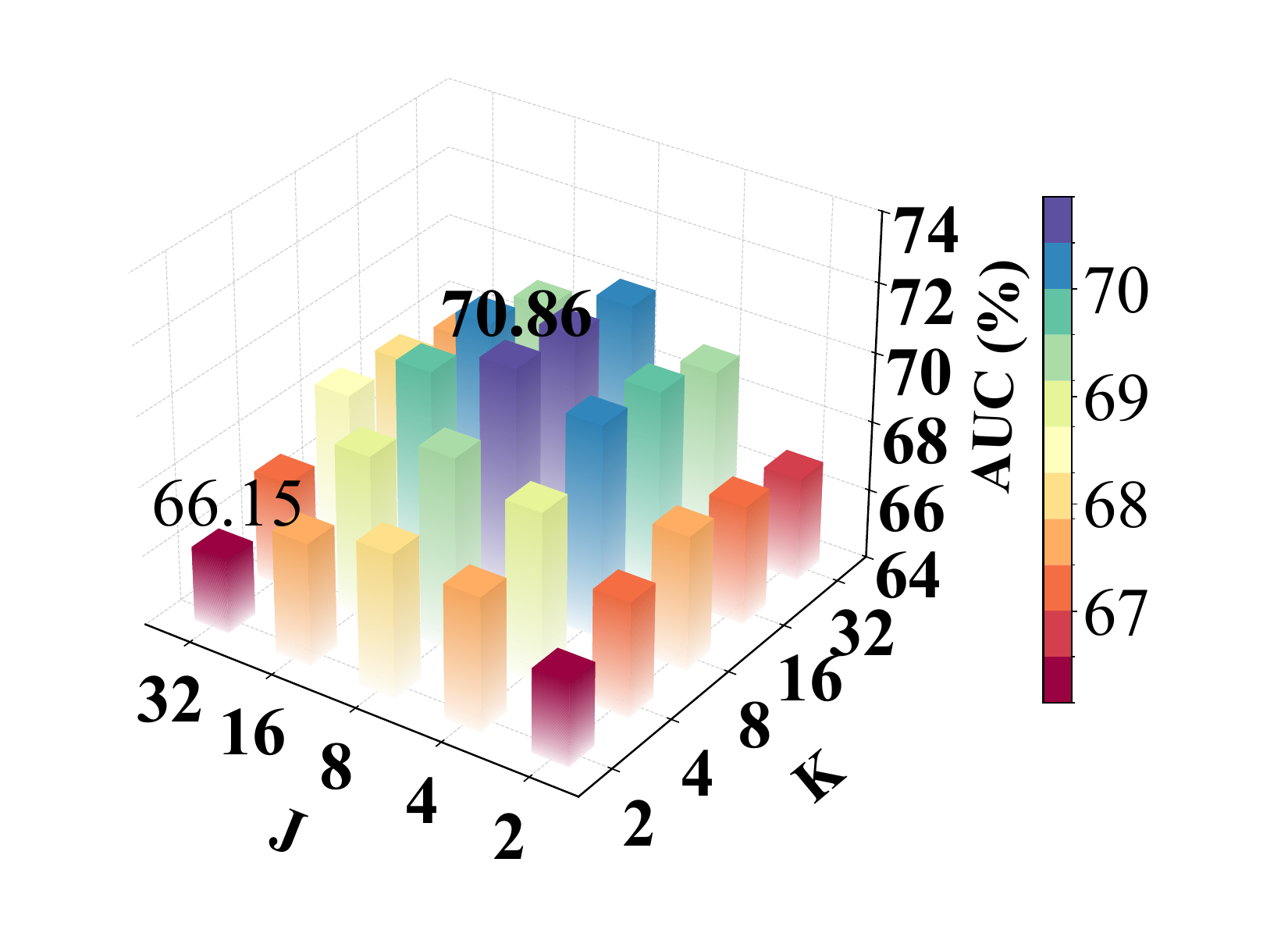}
    \caption{ABIDE (CC200)}
    \label{fig:sen_cc200}
\end{subfigure}\hfill
\begin{subfigure}[t]{0.33\linewidth}
    \centering
    \includegraphics[width=\linewidth]{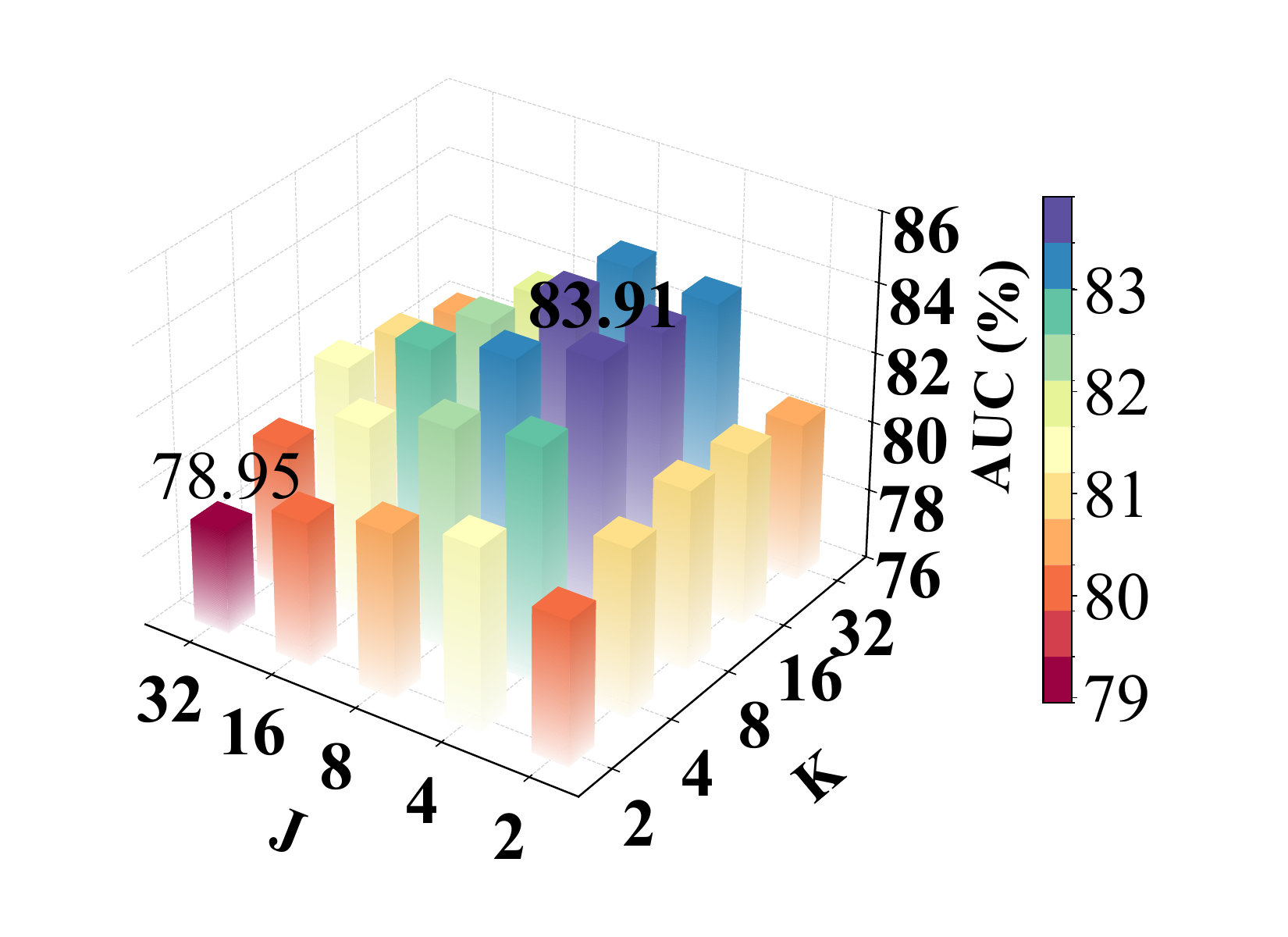}
    \caption{SRPBS}
    \label{fig:sen_srbps}
\end{subfigure}\hfill
\begin{subfigure}[t]{0.33\linewidth}
    \centering
    \includegraphics[width=\linewidth]{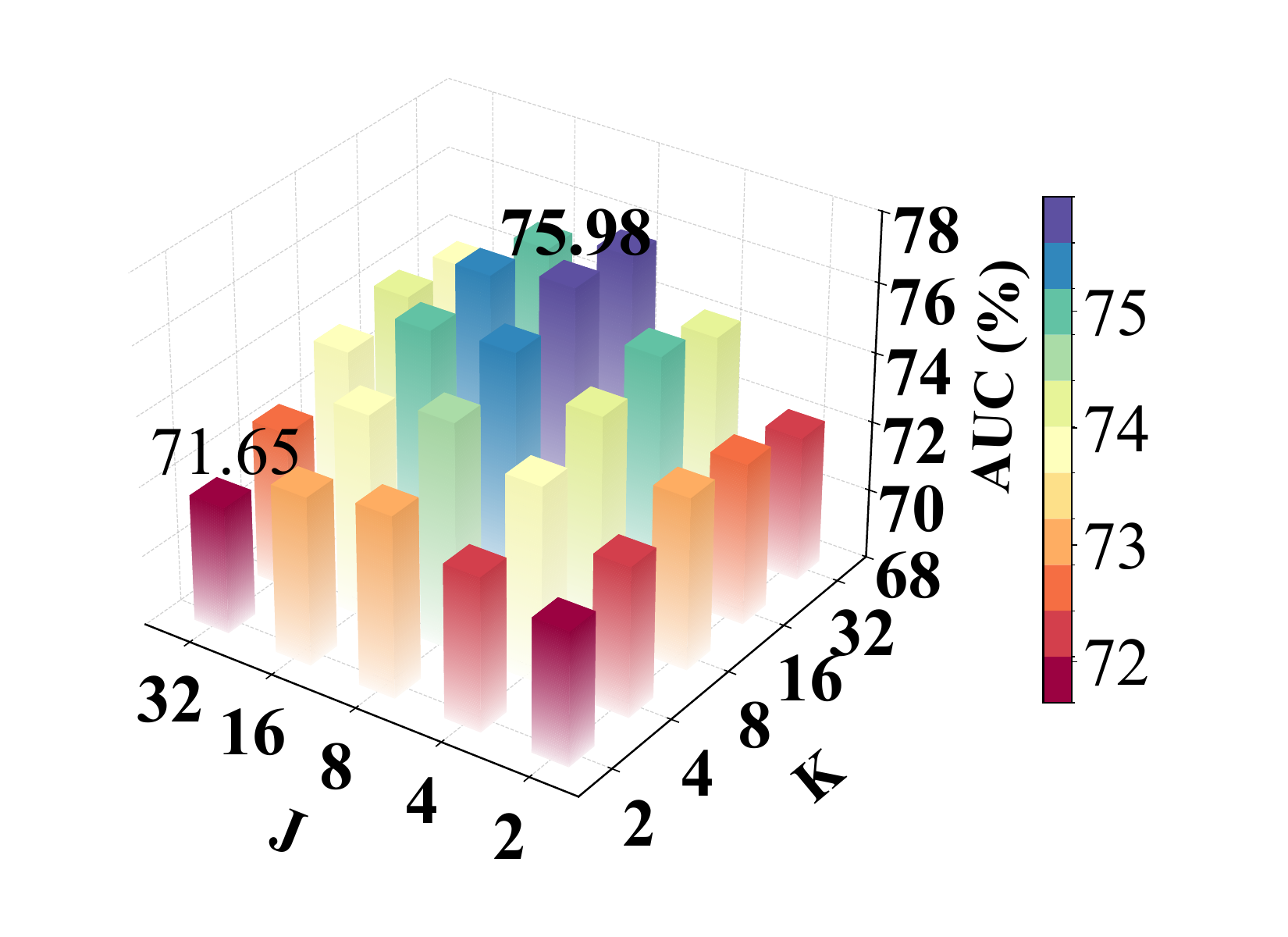}
    \caption{ABCD}
    \label{fig:sen_abcd}
\end{subfigure}
\vspace{-0.2cm}
\caption{Sensitivity study on ABIDE (CC200), SRPBS, and ABCD.}
\vspace{-0.3cm}
\label{fig:more_sen}
\end{figure*}

For the controlled sweep, \method{} w/o QG shares the graph encoder, factor representation, classifier, and training procedure of \method{}, and removes qualification guidance by assigning equal weights to all factors and disabling qualification-based strength modulation. We train this variant independently for each tested value of $\lambda_{\mathrm{env}}$ and select its checkpoint by the same source-validation rule as the main experiment. Runs are paired by fold and seed, and the trajectories show means connected in increasing $\lambda_{\mathrm{env}}$. The full \method{} is added as a separate reference point, and Matched w/o QG denotes the setting with positive $\lambda_{\mathrm{env}}$ selected by source validation.

Panels (a) of Figs.~\ref{fig:diagnostics-retention-mdd}, \ref{fig:diagnostics-retention-srpbs}, and \ref{fig:diagnostics-retention-abcd} show positive rank correlations between the two metrics on REST-meta-MDD and SRPBS but zero rank correlation on ABCD, and $R_{\mathrm{discord}}$ is nonzero on every dataset, so lower site discrepancy does not consistently imply lower predictive variation under FC re-estimation. The controlled sweeps in panels (b) show that stronger alignment reduces both quantities, while the full \method{} achieves lower variation than \method{} w/o QG at comparable site discrepancy, indicating that qualification-guided weighting and strength modulation offer benefits beyond simply increasing alignment strength. Panels (c) further show that \method{} achieves lower total and AUC-preserved mismatch rates than the evaluated w/o QG settings on all three datasets, indicating fewer mismatches even when predictive performance is preserved.

\vspace{-0.2cm}
\subsection{More Qualification and Factor Transferability}
\vspace{-0.2cm}

In this part, we extend the factor-retention analysis to REST-meta-MDD, SRPBS, and ABCD. As shown in panels (d) of Figs.~\ref{fig:diagnostics-retention-mdd}, \ref{fig:diagnostics-retention-srpbs}, and \ref{fig:diagnostics-retention-abcd}, full qualification outperforms the other rankings when retaining 25\% or 50\% of factors across all three datasets. Its advantage over task relevance alone shows that re-estimation support provides complementary information for prioritizing factors with predictive value on unseen sites. Both rankings outperform support alone at partial retention, and support alone falls below random ordering on ABCD, so re-estimation support identifies transferable factors only together with task relevance. Performance gaps narrow as more factors are retained and vanish at full retention, where all rankings recover the original predictions.

\begin{table*}[t]
    \centering
    \small
    \setlength{\tabcolsep}{4pt}
    \renewcommand{\arraystretch}{1.1}
    \caption{Sensitivity of \method{} to moving-block length $\ell_b$. \textbf{Bold} indicates the highest mean.}
    \vspace{-0.15cm}
    \label{tab:block_length}
    \resizebox{\textwidth}{!}{
    \begin{tabular}{l |ccc |ccc |ccc}
        \toprule
        \multirow{2}{*}{Dataset}
        & \multicolumn{3}{c|}{$0.5\times$}
        & \multicolumn{3}{c|}{$1\times$ (Default)}
        & \multicolumn{3}{c}{$2\times$} \\
        & $\ell_b$ & AUC & ACC
        & $\ell_b$ & AUC & ACC
        & $\ell_b$ & AUC & ACC \\
        \midrule

        ABIDE
        & 3 (6.0)
        & 65.84$_{\pm 4.67}$
        & 60.25$_{\pm 4.15}$
        & 6 (12.0)
        & \textbf{67.53}$_{\pm 4.12}$
        & \textbf{62.18}$_{\pm 3.68}$
        & 12 (24.0)
        & 66.71$_{\pm 4.28}$
        & 61.54$_{\pm 3.92}$ \\

        ABIDE (CC200)
        & 3 (6.0)
        & 69.45$_{\pm 5.25}$
        & 60.12$_{\pm 4.58}$
        & 6 (12.0)
        & \textbf{70.86}$_{\pm 4.87}$
        & \textbf{61.92}$_{\pm 3.94}$
        & 12 (24.0)
        & 68.72$_{\pm 5.15}$
        & 59.88$_{\pm 4.25}$ \\

        REST-meta-MDD
        & 4 (8.0)
        & 67.85$_{\pm 4.22}$
        & 63.38$_{\pm 3.42}$
        & 7 (14.0)
        & \textbf{68.84}$_{\pm 3.75}$
        & \textbf{64.42}$_{\pm 2.91}$
        & 14 (28.0)
        & 67.62$_{\pm 3.82}$
        & 63.85$_{\pm 3.18}$ \\

        SRPBS
        & 3 (6.0)
        & 81.35$_{\pm 6.35}$
        & 76.54$_{\pm 5.82}$
        & 6 (12.0)
        & \textbf{83.91}$_{\pm 5.63}$
        & \textbf{78.82}$_{\pm 5.28}$
        & 12 (24.0)
        & 83.72$_{\pm 5.58}$
        & 78.45$_{\pm 5.15}$ \\

        ABCD (ADHD-task)
        & 7 (5.6)
        & 72.15$_{\pm 5.45}$
        & 65.12$_{\pm 4.85}$
        & 15 (12.0)
        & \textbf{75.98}$_{\pm 4.39}$
        & \textbf{68.74}$_{\pm 3.82}$
        & 30 (24.0)
        & 75.45$_{\pm 4.52}$
        & 67.92$_{\pm 3.95}$ \\

        \bottomrule
    \end{tabular}
    }
    \vspace{-0.2cm}
\end{table*}

\subsection{More ablation studies}\label{app:ablation}

In this part, we extend the ablation study to ABIDE (CC200), SRPBS, and ABCD using the same four variants as in Sec.\ref{sec:ablation}. As shown in Fig.~\ref{fig:more_ablation}, \method{} achieves the highest AUC and ACC across all three settings. \method{} w/o RS shows the largest performance drop on ABIDE (CC200), whereas \method{} w/o TR yields the lowest performance on SRPBS and ABCD. These differences suggest that the relative contributions of re-estimation support and task relevance vary across datasets, while their joint use consistently benefits cross-site prediction. \method{} w/o SM and \method{} w/o SF also underperform the full model on both metrics, supporting the complementary roles of modulating overall alignment strength and prioritizing supported factors in cross-site matching.

\subsection{More sensitivity analysis}\label{app:sensitivity}

First, we extend the sensitivity analysis of the number of FC re-estimates $K$ and connectome factors $J$ to ABIDE (CC200), SRPBS, and ABCD. As shown in Fig.~\ref{fig:more_sen}, increasing $K$ from small values generally improves AUC, whereas further increases yield no consistent gains. This suggests that additional re-estimates help characterize variability in predictive contributions, with diminishing benefits at larger budgets. Intermediate values of $J$ also tend to outperform very small or large factor counts, indicating that finer factorization does not necessarily improve cross-site prediction. 

Additionally, we examine the sensitivity of \method{} to the temporal block length $\ell_b$ used for within-scan FC re-estimation. Whereas $K$ controls the number of re-estimates, $\ell_b$ determines the length of contiguous temporal segments retained during resampling. We compare shorter, default, and longer blocks at nominal scales of $0.5\times$, $1\times$, and $2\times$: the $1\times$ setting is the fold-level estimate rounded to the nearest integer, and the $0.5\times$ and $2\times$ settings scale the unrounded estimate before rounding and are used as fixed block lengths. Table~\ref{tab:block_length} reports each block length in time points together with its duration in seconds. The default block length achieves the highest mean AUC and ACC in all five settings, and the estimated defaults correspond to a similar physical duration across datasets despite their different repetition times, indicating that the autocorrelation-based estimate captures a temporal scale that transfers across acquisition protocols. Both halving and doubling the block length degrade performance, with halving being more harmful on most datasets; blocks that are too short break the temporal dependence that FC re-estimation is meant to preserve, whereas overly long blocks reduce the diversity of re-estimates.

\subsection{Effect of Scan Length on Predictive Performance}
\label{app:scan_data}

To examine how reduced temporal data for target FC estimation affects cross-site prediction, we evaluate \method{}, \method{} w/o RS, and \method{} w/o QG on ABIDE, REST-meta-MDD, SRPBS, and ABCD. We reconstruct FC graphs using 75\% or 50\% of each subject's preprocessed BOLD sequence, keeping all source-selected checkpoints fixed. At each fraction, AUC is computed separately for beginning, middle, and end windows and then averaged across positions. We report $\Delta\mathrm{AUC}$ relative to each model's full-scan performance. As shown in Fig.~\ref{fig:scan_degradation}, all three models exhibit larger AUC reductions when fewer time points are retained, but \method{} consistently shows the smallest decline. \method{} w/o RS incurs smaller drops than \method{} w/o QG, while remaining more sensitive than the full model. The differences in degradation widen at the lower scan fraction, suggesting that re-estimation support provides benefits beyond task relevance alone. These results support the value of re-estimation-informed learning for maintaining cross-site prediction performance when target FC graphs are estimated from fewer temporal observations.

\begin{figure*}[t]
    \centering
    \captionsetup{skip=4pt}
    \captionsetup[subfigure]{
        font=small,
        skip=3pt,
        justification=centering
    }

    \begin{subfigure}[t]{0.24\linewidth}
        \centering
        \includegraphics[width=\linewidth]{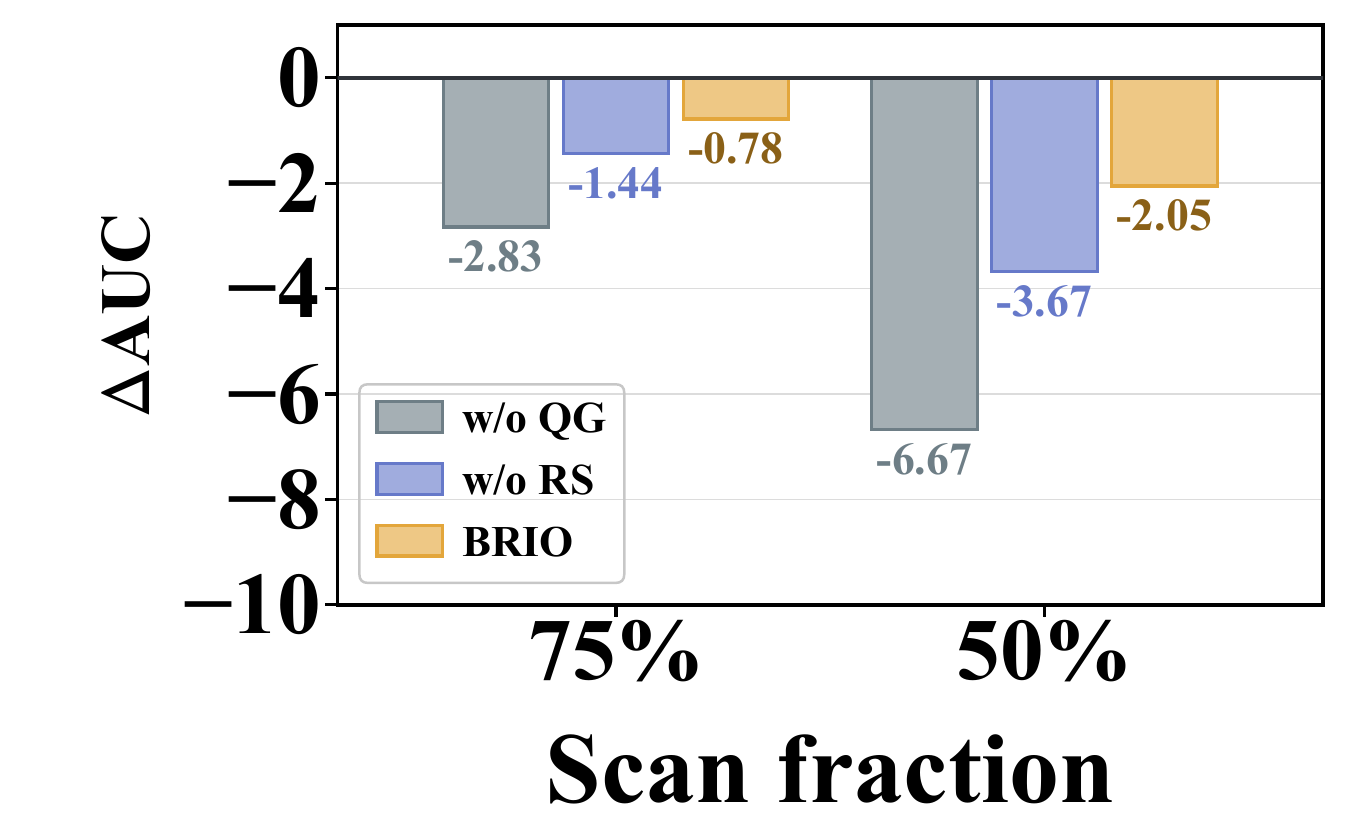}
        \caption{ABIDE}
        \label{fig:scan_degradation_abide}
    \end{subfigure}\hfill%
    \begin{subfigure}[t]{0.24\linewidth}
        \centering
        \includegraphics[width=\linewidth]{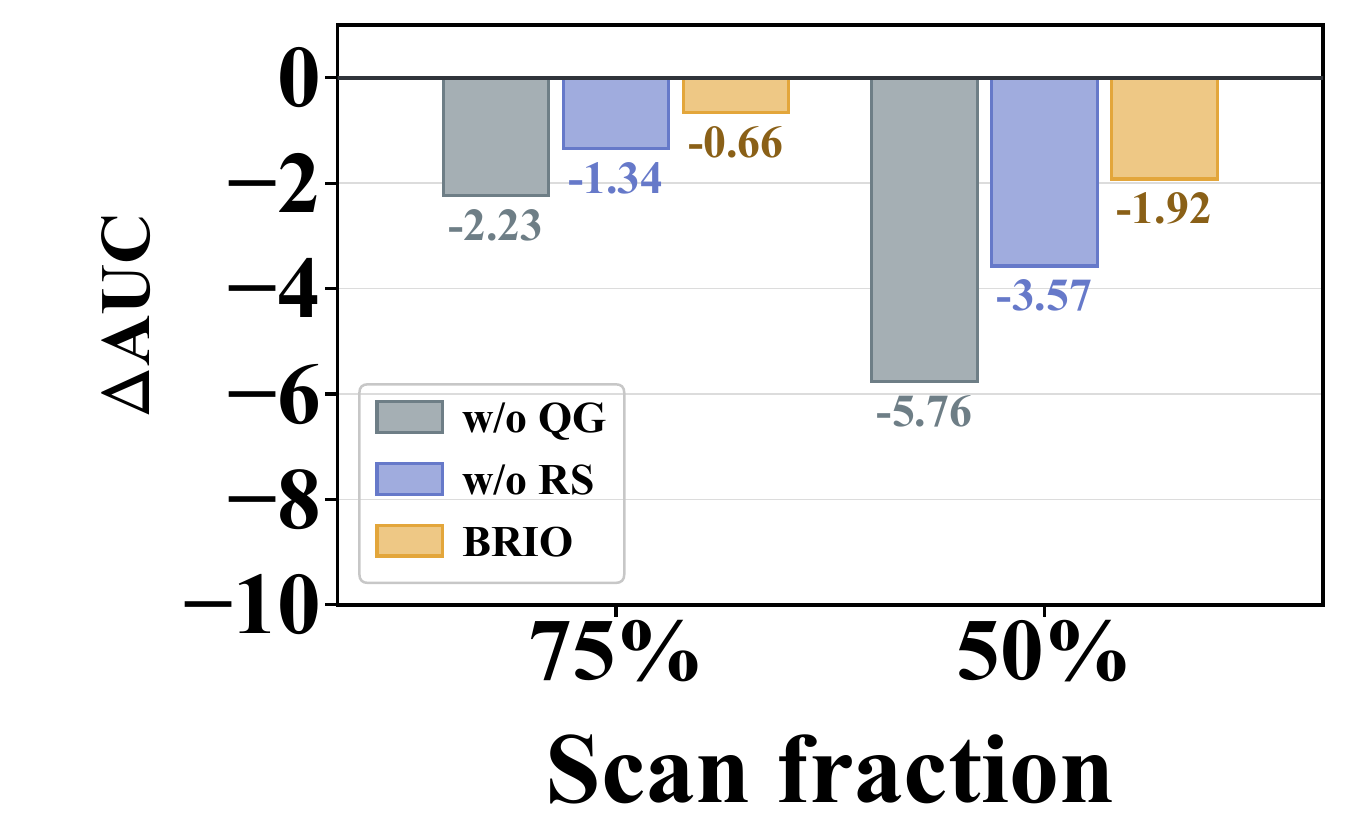}
        \caption{REST-meta-MDD}
        \label{fig:scan_degradation_mdd}
    \end{subfigure}\hfill%
    \begin{subfigure}[t]{0.24\linewidth}
        \centering
        \includegraphics[width=\linewidth]{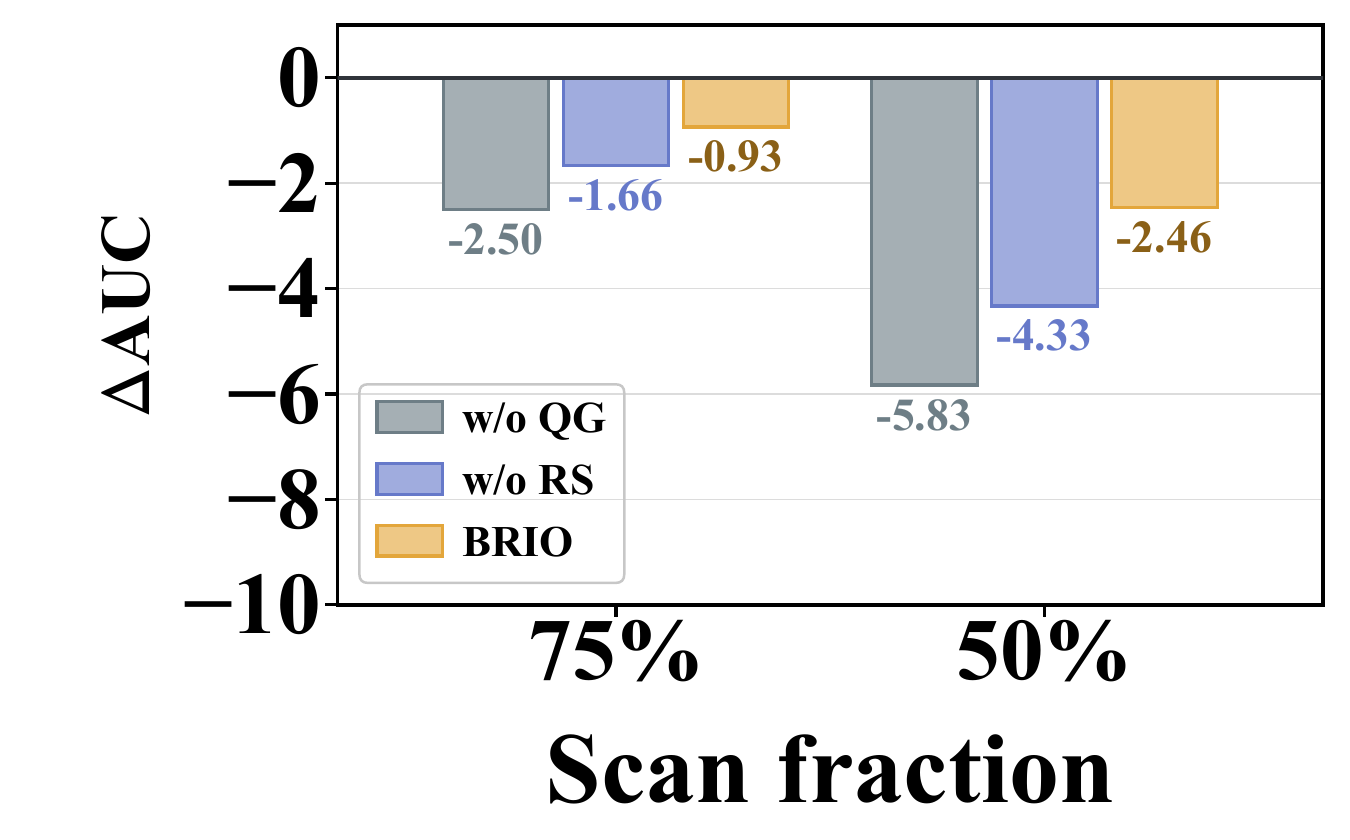}
        \caption{SRPBS}
        \label{fig:scan_degradation_srpbs}
    \end{subfigure}\hfill%
    \begin{subfigure}[t]{0.24\linewidth}
        \centering
        \includegraphics[width=\linewidth]{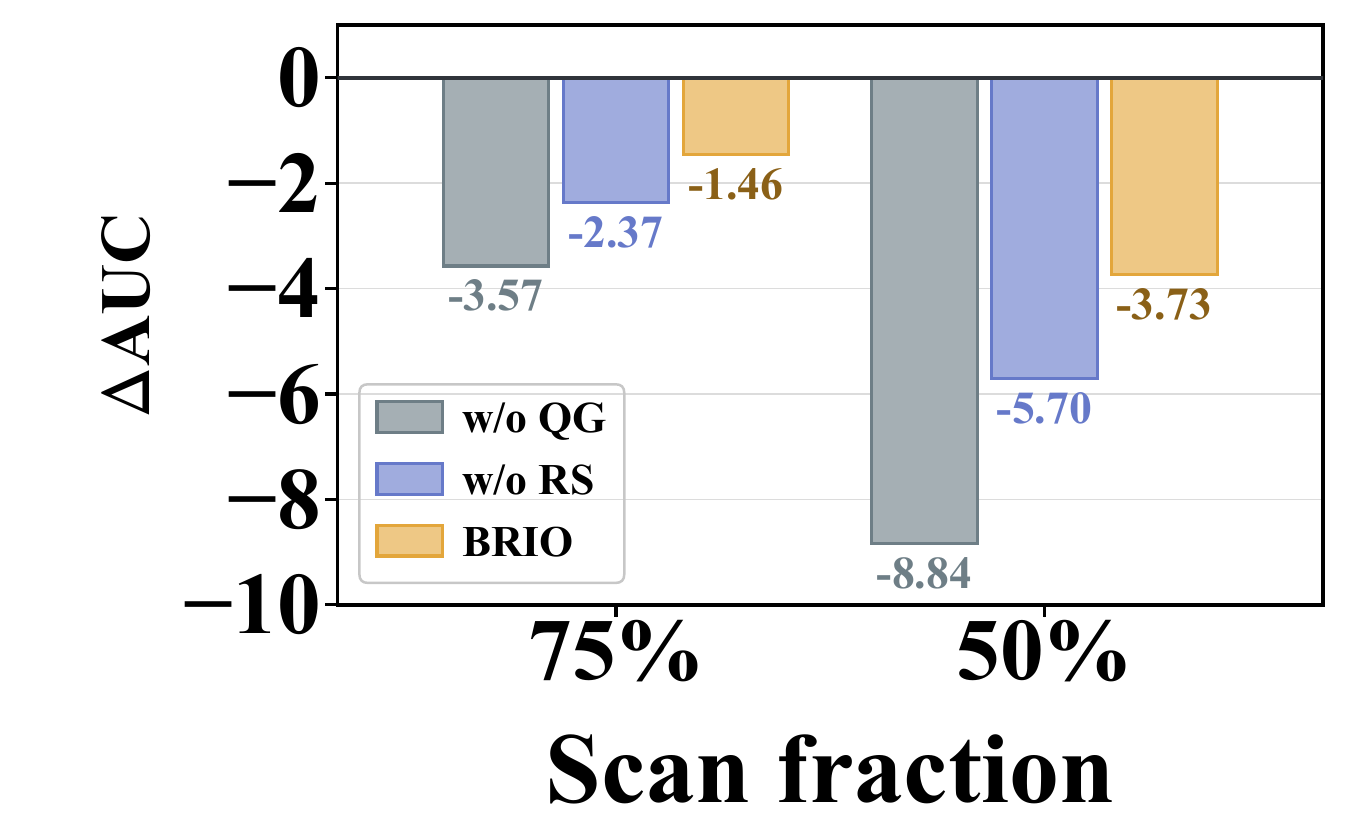}
        \caption{ABCD}
        \label{fig:scan_degradation_abcd}
    \end{subfigure}

    \caption{Performance under reduced scan data on different datasets.}
    \label{fig:scan_degradation}
    \vspace{-0.2cm}
\end{figure*}

\subsection{Training Efficiency and Resource Consumption}
\label{app:efficiency}

In this part, we compare the training efficiency of \method{} and representative baselines on four datasets. Table~\ref{tab:efficiency} reports per-epoch training time in seconds and GPU memory consumption in GB. We observe that \method{} requires more training time and GPU memory than lightweight baselines such as FC-HGNN and CORE. Compared with BrainOOD, it achieves shorter per-epoch training times while using slightly more memory, reflecting a trade-off between runtime and memory consumption. In contrast, \method{} requires less time and memory than DiSCO. These relative resource-use patterns remain consistent across all evaluated datasets.

\begin{table*}[t]
    \centering
    \small
    \setlength{\tabcolsep}{6pt}
    \renewcommand{\arraystretch}{1.1}
    \caption{Training time (in seconds) and GPU memory consumption (in GB) for each training epoch.}
    \vspace{-0.15cm}
    \label{tab:efficiency}
    \resizebox{0.8\textwidth}{!}{
    \begin{tabular}{l |cc |cc |cc |cc}
        \toprule
        \multirow{2}{*}{Method}
        & \multicolumn{2}{c|}{ABIDE}
        & \multicolumn{2}{c|}{REST-meta-MDD}
        & \multicolumn{2}{c|}{SRPBS}
        & \multicolumn{2}{c}{ABCD} \\
        & Time & Memory
        & Time & Memory
        & Time & Memory
        & Time & Memory \\
        \midrule
        GSAT
        & 0.68 & 2.1
        & 1.52 & 3.3
        & 0.48 & 2.0
        & 0.34 & 1.7 \\

        DisC
        & 1.05 & 2.4
        & 2.38 & 3.8
        & 0.74 & 2.2
        & 0.51 & 1.9 \\

        CEPG
        & 0.58 & 2.0
        & 1.34 & 3.1
        & 0.41 & 1.9
        & 0.28 & 1.6 \\

        DiSCO
        & 1.65 & 2.9
        & 3.85 & 4.6
        & 1.18 & 2.7
        & 0.78 & 2.3 \\

        BrainNetTF
        & 0.78 & 2.0
        & 1.65 & 3.2
        & 0.49 & 1.9
        & 0.32 & 1.6 \\

        AGMGC
        & 0.71 & 1.8
        & 1.58 & 3.0
        & 0.46 & 1.8
        & 0.30 & 1.5 \\

        FC-HGNN
        & 0.34 & 1.8
        & 0.48 & 3.1
        & 0.19 & 1.8
        & 0.12 & 1.5 \\

        XG-GNN
        & 0.63 & 2.1
        & 1.31 & 3.4
        & 0.38 & 2.0
        & 0.25 & 1.7 \\

        DeCI
        & 0.88 & 2.3
        & 1.95 & 3.6
        & 0.62 & 2.1
        & 0.44 & 1.8 \\

        BrainOOD
        & 1.45 & 2.5
        & 3.32 & 4.1
        & 0.88 & 2.3
        & 0.58 & 2.0 \\

        CORE
        & 0.35 & 2.0
        & 0.52 & 3.7
        & 0.22 & 1.9
        & 0.12 & 1.6 \\
        \midrule

        \textbf{\method{}}
        & 1.22 & 2.6
        & 2.85 & 4.2
        & 0.86 & 2.4
        & 0.56 & 2.1 \\
        \bottomrule
    \end{tabular}
    }
    \vspace{-0.2cm}
\end{table*}

\section{Interpretability Analysis}
\label{app:interpretability}

To examine how qualification-based factor weighting translates into regional pooling emphasis, we analyze the models trained on ABIDE and REST-meta-MDD using AAL116. For each model, we first summarize the normalized qualifications $\overline{\alpha}_{e,e',c,j}$ over all unordered source-site pairs and both classes:
\begin{equation}
    \overline{\alpha}_{j}^{\mathrm{global}}
    =
    \frac{1}{E(E-1)}
    \sum_{\substack{e,e'\in\mathcal E_{\mathrm{tr}}\\e<e'}}
    \sum_{c\in\{0,1\}}
    \overline{\alpha}_{e,e',c,j},
    \qquad
    \sum_{j=1}^{J}
    \overline{\alpha}_{j}^{\mathrm{global}}
    =1.
\label{eq:app_global_factor_qualification}
\end{equation}
This equal-weight summary captures relative factor allocation across source-site pairs and classes, separately from the overall alignment strength controlled by $g_{e,e',c}$. We then project these qualifications through the factor--ROI assignment $\overline{\bm\Pi}$ retained after supervised warm-up. For ROI $u$, the resulting score and its normalization are
\begin{equation}
    S_u
    =
    \sum_{j=1}^{J}
    \overline{\alpha}_{j}^{\mathrm{global}}
    \overline{\Pi}_{j,u},
    \qquad
    \widehat S_u
    =
    \frac{S_u}{\sum_{v=1}^{P}S_v}
    =
    JS_u.
\label{eq:app_qualification_roi_score}
\end{equation}
The assignment row masses of $1/J$ give $\sum_uS_u=1/J$, so $\widehat S_u$ defines a normalized ROI profile. Higher scores indicate greater pooling emphasis in factors with higher average source qualification. Under uniform factor weights, the column-mass constraint yields $\widehat S_u=1/P$ for every ROI. Regional differences therefore reflect how qualification-based factor weighting interacts with the learned ROI assignment.

To summarize these profiles across LOSO folds and random seeds, we compute $\widehat S_u^{(r,s)}$ separately for each fold $r$ and seed $s$, using the qualifications and retained assignment of that model. Since factor indices need not correspond across independently trained models, we aggregate the scores in the common AAL116 ROI space:
\begin{equation}
    S_u^{\mathrm{dataset}}
    =
    \frac{1}{|\mathcal R||\mathcal S|}
    \sum_{r\in\mathcal R}
    \sum_{s\in\mathcal S}
    \widehat S_u^{(r,s)},
    \qquad
    \sum_{u=1}^{P}S_u^{\mathrm{dataset}}=1,
\label{eq:app_dataset_roi_score}
\end{equation}
where $\mathcal R$ and $\mathcal S$ denote the evaluated folds and seeds. This procedure preserves each model's factor--ROI correspondence during projection and yields a dataset-level summary of regional pooling emphasis.

Fig.~\ref{fig:roi_interpretability} visualizes the five ROIs with the highest dataset-averaged scores. On ABIDE, these are the left precuneus, right superior temporal gyrus, right posterior cingulate gyrus, left medial superior frontal gyrus, and right fusiform gyrus. On REST-meta-MDD, the selected regions are the left medial superior frontal gyrus, left amygdala, right posterior cingulate gyrus, right hippocampus, and right anterior cingulate gyrus. The right posterior cingulate and left medial superior frontal regions appear in both profiles, while the remaining selected ROIs differ between datasets. These shared and differing regions illustrate how the factors prioritized by source-side qualification distribute pooling emphasis across the anatomical space. The profiles complement the ablation and factor-retention analyses by connecting the relative factor weights used for alignment to their regional pooling structure.

\begin{figure*}[t]
    \centering
    \captionsetup{
        font=small,
        skip=5pt
    }
    \captionsetup[subfigure]{
        font=small,
        skip=3pt,
        justification=centering
    }

    \begin{subfigure}[t]{0.49\linewidth}
        \vspace{0pt}
        \centering
        \includegraphics[width=\linewidth]{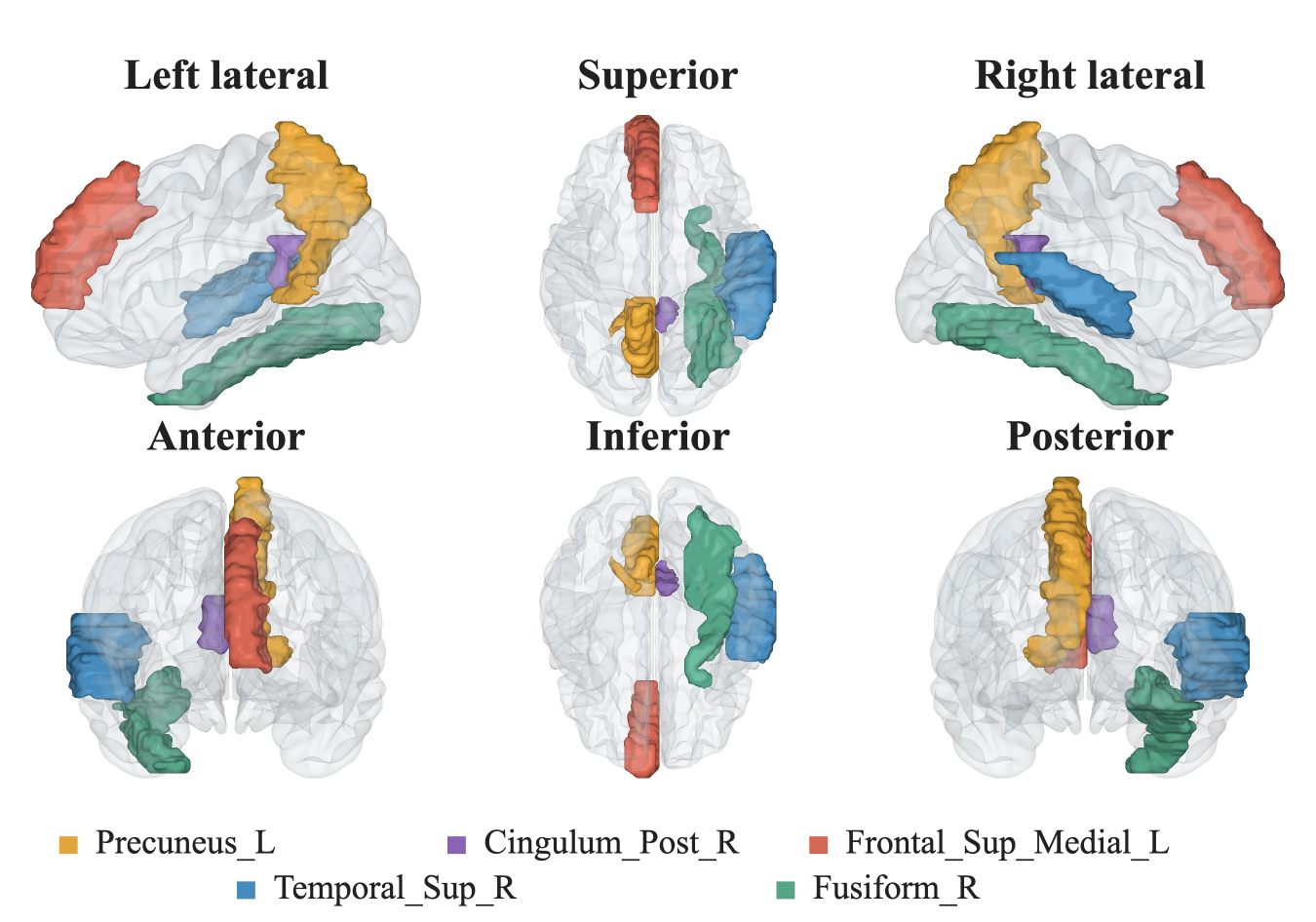}
        \caption{ABIDE}
        \label{fig:roi_abide}
    \end{subfigure}\hfill%
    \begin{subfigure}[t]{0.49\linewidth}
        \vspace{0pt}
        \centering
        \includegraphics[width=\linewidth]{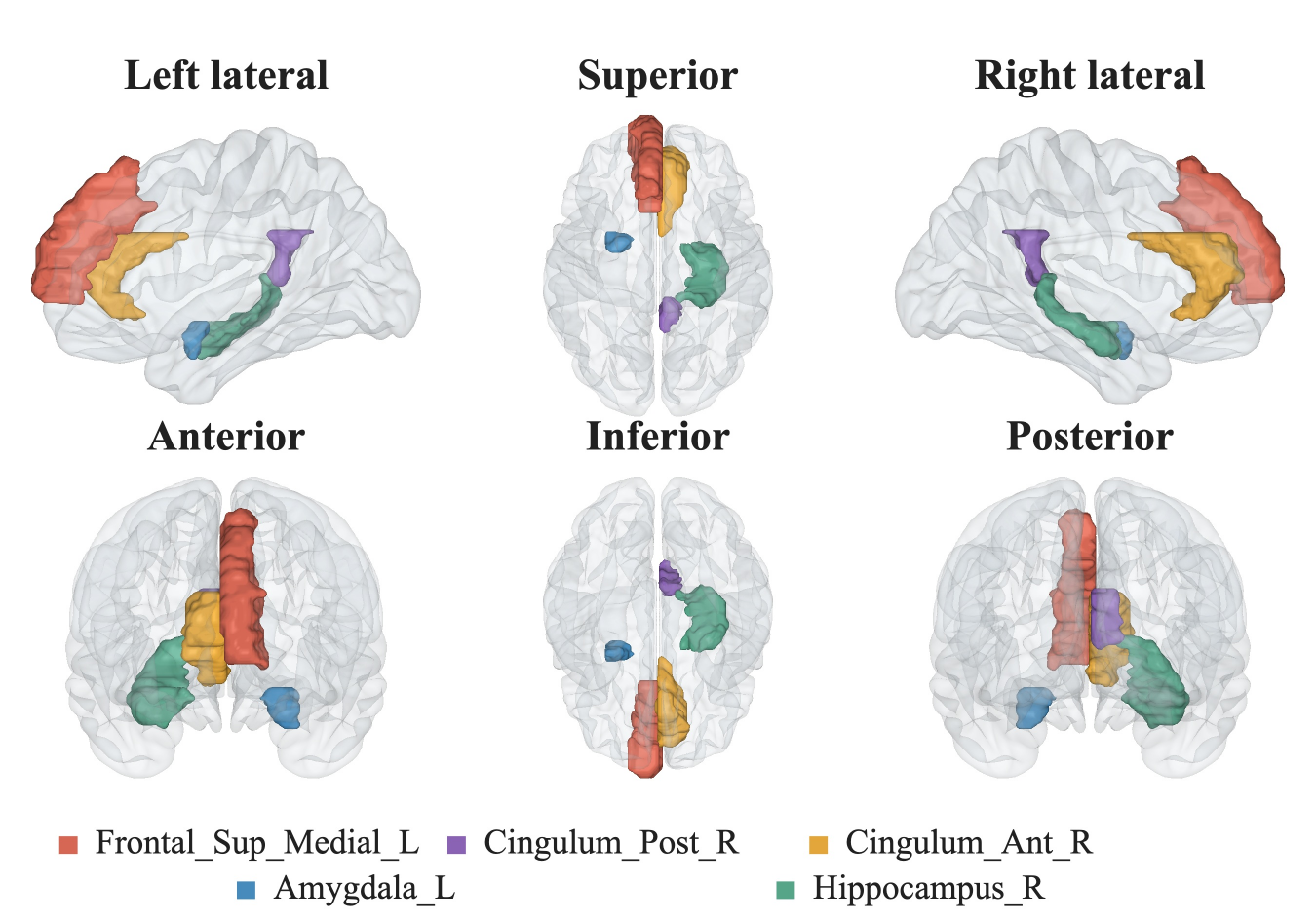}
        \caption{REST-meta-MDD}
        \label{fig:roi_mdd}
    \end{subfigure}

    \caption{Anatomical distribution of qualification-weighted ROI profiles on ABIDE and REST-meta-MDD.}
    \label{fig:roi_interpretability}
    \vspace{-0.2cm}
\end{figure*}



\input{table/algorithm}

%% file: table/notations.tex
\begin{table}[t]
    \centering
    \caption{Summary of key notations.}
    \label{tab:notation}
    \small
    \setlength{\tabcolsep}{6pt}
    \renewcommand{\arraystretch}{1.10}
    \begin{tabularx}{0.98\linewidth}{@{}lX@{}}
        \toprule
        \textbf{Symbol} & \textbf{Description} \\
        \midrule

        $\mathcal E_{\mathrm{tr}},\mathcal E_{\mathrm{te}},E$
        & Source and unseen test site sets, and the number of source sites. \\

        $\bm X_i^e,Y_i^e$
        & Preprocessed BOLD sequence and binary label of subject $i$ at site $e$. \\

        $T_i^e,P$
        & Numbers of valid time points and ROIs. \\

        $\mathcal I_{e,c},N_{e,c}$
        & Subject index set for site $e$ and class $c$, and its size. \\

        $\mathcal B_{e,c}$
        & Training mini-batch subset for site $e$ and class $c$. \\

        \midrule

        $\Psi,\mathcal R_k,K$
        & FC estimator, joint circular moving-block resampling operator, and number of FC re-estimates per subject. \\

        $\ell_b,\ell_{b,i}^{e}$
        & Fold-level block length estimated from training subjects, and its per-sequence clipped value, in time points. \\

        $\tau_{\mathrm{ac}},\ell_{\min},\ell_{\max}$
        & Autocorrelation threshold and clipping bounds of the block-length rule. \\

        $\bm A_i^{e,(0)},\bm A_i^{e,(k)}$
        & Full-scan FC graph and its $k$-th re-estimate, with $k\geq1$. \\

        $\bm S,\rho_g$
        & Fixed source-derived binary propagation support and per-ROI neighbor selection ratio. \\

        $\bm G_i^{e,(k)},\bm G_i^{e,(k),\pm}$
        & Supported FC adjacency and its positive weights ($+$) and negative-weight magnitudes ($-$). \\

        $\bm E^{\mathrm{ROI}},\bm H_i^{e,(k)}$
        & Shared ROI identity embeddings and ROI representations produced by the shared signed GNN. \\

        \midrule

        $J,\bm Q$
        & Number of connectome factors and learnable factor queries. \\

        $\bm\Pi,\overline{\bm\Pi}$
        & Shared balanced factor--ROI assignment and its cached value fixed after supervised warm-up. \\

        $\bm r_{i,j}^{e,(k)}$
        & Unit-norm representation of factor $j$ for subject $i$ at site $e$ and FC estimate $k$. \\

        $\bm s_i^{e,(k)}$
        & Graph representation formed by scaled factor concatenation in a fixed order. \\

        $\bm w_j,b$
        & Classifier weight block for factor $j$ and the prediction bias. \\

        $a_{i,j}^{e,(k)}$
        & Factor score $\bm w_j^\top\bm r_{i,j}^{e,(k)}$ before the common $1/\sqrt J$ scaling. \\

        $z_i^{e,(k)},\widehat p_i^{e,(0)}$
        & Prediction logit for FC estimate $k$ and full-scan prediction probability. \\

        \midrule

        $\mu_{e,c,j}$
        & Mean full-scan contribution of factor $j$ for site $e$ and class $c$. \\

        $V_{e,c,j}^{\mathrm{sub}},V_{c,j}^{\mathrm{pool}}$
        & Within-class contribution variance at a site and its equal-site average. \\

        $\Delta_{e,j},\Delta_j$
        & Full-scan class-mean contrast at a site and its signed equal-site average. \\

        $\beta_j,R_j^{\mathrm{task}},\pi_j^{\mathrm{task}}$
        & Relative squared classifier-block norm, unnormalized predictive relevance, and its normalization across factors. \\

        $D_{e,c,j}^{\mathrm{re}}$
        & Mean squared difference between paired re-estimated and full-scan contributions. \\

        $M_{e,j}^{\mathrm{task}},T_{e,c,j}$
        & Squared half-separation $\Delta_{e,j}^2/4$ and local reference scale $V_{e,c,j}^{\mathrm{sub}}+M_{e,j}^{\mathrm{task}}$. \\

        $q_{e,c,j}^{\mathrm{re}},\kappa_{e,c,j}$
        & Task-calibrated re-estimation support and full-scan separation ratio for factor $j$ at site $e$ and class $c$. \\

        \midrule

        $q_{e,e',c,j}^{\mathrm{pair}}$
        & Geometric mean of factor $j$'s re-estimation supports at sites $e$ and $e'$ for class $c$. \\

        $\alpha_{e,e',c,j},\overline{\alpha}_{e,e',c,j}$
        & Unnormalized pairwise qualification and its normalization determining the relative factor weight in alignment. \\

        $g_{e,e',c}$
        & Total pairwise qualification scaling alignment for sites $e,e'$ and class $c$. \\

        $\bm\varphi_{i\mid e,e',c}^{Q,d}$
        & Pair-specific graph representation at site $d\in\{e,e'\}$, formed from reweighted full-scan factors. \\

        $\widehat P_{e,c}^{Q[e,e']}$
        & Empirical class-conditional distribution at site $e$ in the common pairwise weighted space. \\

        $\mathcal S_{\varepsilon},\varepsilon$
        & Debiased Sinkhorn divergence and its entropic regularization parameter. \\

        \bottomrule
    \end{tabularx}
\end{table}

%% file: table/ID.tex
\begin{table*}[t]
    \centering
    \caption{ID results on ABIDE, ABIDE (CC200), REST-meta-MDD, SRPBS, and ABCD (ADHD-task) under ten-fold cross-validation stratified by site and class, reported as mean $\pm$ standard deviation across test folds. \textbf{Bold} results indicate the best performance.}
    \vspace{-0.3cm}
    \label{tab:main_results_id}
    \resizebox{\textwidth}{!}{
    \begin{tabular}{cc |cc |cc |cc |cc |cc}
    \toprule
    \multirow{2}{*}{Type} & \multirow{2}{*}{Method}
    & \multicolumn{2}{c|}{ABIDE}
    & \multicolumn{2}{c|}{ABIDE (CC200)}
    & \multicolumn{2}{c|}{REST-meta-MDD}
    & \multicolumn{2}{c|}{SRPBS}
    & \multicolumn{2}{c}{ABCD (ADHD-task)} \\
    &
    & AUC & ACC
    & AUC & ACC
    & AUC & ACC
    & AUC & ACC
    & AUC & ACC \\
    \midrule

    \multirow{3}{*}{\rotatebox{90}{\fontsize{9}{12}\selectfont GNN}}
    & GCN
    & 63.45$_{\pm 4.12}$ & 59.21$_{\pm 3.85}$
    & 51.88$_{\pm 4.02}$ & 51.05$_{\pm 3.74}$
    & 62.15$_{\pm 2.88}$ & 57.02$_{\pm 2.45}$
    & 84.12$_{\pm 3.85}$ & 75.34$_{\pm 4.12}$
    & 72.18$_{\pm 5.12}$ & 64.05$_{\pm 4.88}$ \\

    & GAT
    & 59.12$_{\pm 3.95}$ & 55.44$_{\pm 3.55}$
    & 58.62$_{\pm 4.15}$ & 57.18$_{\pm 3.82}$
    & 63.84$_{\pm 3.12}$ & 59.12$_{\pm 2.65}$
    & 81.75$_{\pm 4.05}$ & 70.88$_{\pm 4.35}$
    & 71.02$_{\pm 5.45}$ & 63.88$_{\pm 5.10}$ \\

    & GIN
    & 58.74$_{\pm 3.65}$ & 54.95$_{\pm 3.22}$
    & 53.15$_{\pm 4.33}$ & 50.42$_{\pm 3.95}$
    & 59.33$_{\pm 2.54}$ & 56.14$_{\pm 2.22}$
    & 82.05$_{\pm 3.78}$ & 71.45$_{\pm 3.92}$
    & 69.45$_{\pm 5.33}$ & 61.22$_{\pm 4.75}$ \\

    \midrule

    \multirow{2}{*}{\rotatebox{90}{\fontsize{9}{12}\selectfont OOD}}
    & CORAL
    & 56.88$_{\pm 4.05}$ & 53.25$_{\pm 3.75}$
    & 55.14$_{\pm 4.45}$ & 53.75$_{\pm 4.05}$
    & 59.95$_{\pm 3.15}$ & 57.85$_{\pm 2.85}$
    & 81.14$_{\pm 4.12}$ & 70.82$_{\pm 4.05}$
    & 70.12$_{\pm 5.25}$ & 63.45$_{\pm 4.85}$ \\
    & IRM
    & 60.15$_{\pm 3.88}$ & 56.78$_{\pm 3.45}$
    & 53.42$_{\pm 4.65}$ & 54.12$_{\pm 4.11}$
    & 60.75$_{\pm 2.95}$ & 57.44$_{\pm 2.75}$
    & 80.22$_{\pm 3.94}$ & 70.15$_{\pm 3.88}$
    & 71.25$_{\pm 5.08}$ & 60.14$_{\pm 4.95}$ \\

    \midrule

    \multirow{4}{*}{
        \rotatebox{90}{
            \shortstack{\fontsize{9}{12}\selectfont Graph\\OOD}
        }
    }
    & GSAT
    & 59.85$_{\pm 3.75}$ & 57.14$_{\pm 3.52}$
    & 52.45$_{\pm 3.95}$ & 52.12$_{\pm 3.65}$
    & 58.45$_{\pm 2.85}$ & 55.75$_{\pm 2.55}$
    & 82.45$_{\pm 3.65}$ & 72.15$_{\pm 3.85}$
    & 72.05$_{\pm 5.14}$ & 62.45$_{\pm 4.92}$ \\

    & DisC
    & 58.12$_{\pm 4.15}$ & 57.33$_{\pm 3.85}$
    & 54.15$_{\pm 4.25}$ & 53.85$_{\pm 4.02}$
    & 59.88$_{\pm 3.05}$ & 57.12$_{\pm 2.75}$
    & 83.12$_{\pm 3.85}$ & 73.55$_{\pm 4.15}$
    & 70.45$_{\pm 5.05}$ & 59.75$_{\pm 4.85}$ \\

    & CEPG
    & 53.45$_{\pm 4.55}$ & 51.15$_{\pm 4.12}$
    & 46.85$_{\pm 4.85}$ & 48.75$_{\pm 4.45}$
    & 57.15$_{\pm 3.25}$ & 55.45$_{\pm 2.95}$
    & 76.15$_{\pm 4.25}$ & 69.12$_{\pm 4.35}$
    & 64.12$_{\pm 5.85}$ & 58.45$_{\pm 5.15}$ \\

    & DiSCO
    & 58.25$_{\pm 3.85}$ & 56.45$_{\pm 3.65}$
    & 53.12$_{\pm 4.15}$ & 52.75$_{\pm 3.85}$
    & 60.15$_{\pm 2.75}$ & 57.15$_{\pm 2.55}$
    & 85.12$_{\pm 3.55}$ & 74.15$_{\pm 3.75}$
    & 68.75$_{\pm 5.15}$ & 62.85$_{\pm 4.85}$ \\

    \midrule

    \multirow{7}{*}{
        \rotatebox{90}{
            \shortstack{\fontsize{9}{12}\selectfont Brain\\Networks}
        }
    }
    & BrainNetTF
    & 63.75$_{\pm 3.55}$ & \textbf{61.85$_{\pm 3.25}$}
    & 58.45$_{\pm 4.15}$ & 54.15$_{\pm 3.85}$
    & 62.85$_{\pm 2.55}$ & 58.15$_{\pm 2.35}$
    & 79.45$_{\pm 4.15}$ & 71.15$_{\pm 4.35}$
    & 69.15$_{\pm 4.85}$ & 58.15$_{\pm 4.55}$ \\

    & XG-GNN
    & 61.45$_{\pm 3.75}$ & 55.45$_{\pm 3.45}$
    & 52.15$_{\pm 4.55}$ & 54.15$_{\pm 4.15}$
    & 60.45$_{\pm 2.95}$ & 56.75$_{\pm 2.65}$
    & 84.75$_{\pm 3.65}$ & 73.85$_{\pm 3.85}$
    & 72.45$_{\pm 4.85}$ & 57.15$_{\pm 4.65}$ \\

    & AGMGC
    & 59.85$_{\pm 3.85}$ & 52.75$_{\pm 3.55}$
    & 57.15$_{\pm 3.95}$ & 52.15$_{\pm 3.65}$
    & 58.85$_{\pm 2.85}$ & 55.15$_{\pm 2.55}$
    & 86.45$_{\pm 3.45}$ & 77.15$_{\pm 3.65}$
    & 71.85$_{\pm 5.05}$ & 64.15$_{\pm 4.75}$ \\

    & FC-HGNN
    & 57.15$_{\pm 4.15}$ & 53.85$_{\pm 3.85}$
    & 56.45$_{\pm 4.25}$ & 54.75$_{\pm 3.95}$
    & 61.45$_{\pm 3.05}$ & 57.15$_{\pm 2.75}$
    & 77.85$_{\pm 4.05}$ & 68.75$_{\pm 4.25}$
    & 72.15$_{\pm 4.95}$ & 60.45$_{\pm 4.85}$ \\

    & BrainOOD
    & 64.75$_{\pm 3.45}$ & 60.45$_{\pm 3.15}$
    & 63.85$_{\pm 3.65}$ & 60.15$_{\pm 3.35}$
    & 64.75$_{\pm 2.75}$ & 61.15$_{\pm 2.45}$
    & 86.45$_{\pm 3.25}$ & 77.45$_{\pm 3.45}$
    & 73.45$_{\pm 4.75}$ & 65.15$_{\pm 4.55}$ \\

    & DeCI
    & 57.45$_{\pm 4.05}$ & 61.45$_{\pm 3.55}$
    & 60.15$_{\pm 3.85}$ & \textbf{62.85$_{\pm 3.45}$}
    & 55.15$_{\pm 3.15}$ & 55.85$_{\pm 2.85}$
    & 87.45$_{\pm 3.35}$ & 79.85$_{\pm 3.55}$
    & 56.45$_{\pm 5.25}$ & 63.45$_{\pm 4.95}$ \\

    & CORE
    & 65.45$_{\pm 3.25}$ & 60.85$_{\pm 2.95}$
    & 67.45$_{\pm 3.45}$ & 62.15$_{\pm 3.15}$
    & 65.85$_{\pm 2.65}$ & 61.45$_{\pm 2.35}$
    & \textbf{88.75$_{\pm 3.05}$} & \textbf{81.45$_{\pm 3.25}$}
    & 73.15$_{\pm 4.55}$ & 67.45$_{\pm 4.35}$ \\

    \midrule

    & \method{}
    & \textbf{66.85$_{\pm 3.15}$} & 61.55$_{\pm 2.85}$
    & \textbf{68.75$_{\pm 3.25}$} & 62.65$_{\pm 3.05}$
    & \textbf{66.45$_{\pm 2.55}$} & \textbf{62.15$_{\pm 2.25}$}
    & 88.15$_{\pm 3.15}$ & 80.75$_{\pm 3.35}$
    & \textbf{76.45$_{\pm 4.25}$} & \textbf{69.15$_{\pm 4.05}$} \\

    \bottomrule
    \end{tabular}
    }
    \vspace{-0.2cm}
\end{table*}

%% file: table/algorithm.tex
\begin{algorithm}[ht]
\caption{Training and Inference of \method{}}
\label{alg:overall}
\small
\begin{algorithmic}[1]
\REQUIRE Source datasets $\{\mathcal D^e\}_{e\in\mathcal E_{\mathrm{tr}}}$ with $E\geq2$ and $N_{e,c}>0$; unseen-site sequences $\{\bm X_z\}_{z=1}^{N_{\mathrm{te}}}$; $\Psi$, $\rho_g$, $J$, and $K$; block-length rule with autocorrelation threshold $\tau_{\mathrm{ac}}$ and bounds $[\ell_{\min},\ell_{\max}]$ in time points; warm-up epochs $T_{\mathrm w}$; qualification interval $M$ in optimization steps; $\varepsilon$, $\lambda_{\mathrm{env}}$, and a training stopping criterion.
\ENSURE Unseen-site prediction probabilities $\widehat{\bm p}_{\mathrm{te}}$.

\STATE Cache source full-scan graphs $\bm A_i^{e,(0)}=\Psi(\bm X_i^e)$ and construct $\bm S$ using Eq.~\eqref{eq:source_support}.
\STATE Estimate the fold-level block length $\ell_b$ from autocorrelation horizons of training-subject ROI signals with threshold $\tau_{\mathrm{ac}}$ as in Appendix~\ref{app:preprocessing}; for every source subject set $\ell_{b,i}^{e}\leftarrow\operatorname{clip}\!\big(\operatorname{round}(\ell_b),\,\ell_{\min},\,\min(\ell_{\max},\lfloor T_i^e/2\rfloor)\big)$.
\STATE Initialize all learnable parameters $\Theta$; set $\Theta_{\mathrm{asg}}=\{\bm Q,\omega\}$ and $\mathcal C_Q\leftarrow\varnothing$.
\STATE Use mini-batches with equally sized nonempty subsets $\mathcal B_{e,c}$ for every source site and class; evaluate $\mathcal L_{\mathrm{cls}}$ using $\mathcal B_{e,c}$ and $|\mathcal B_{e,c}|$ in place of $\mathcal I_{e,c}$ and $N_{e,c}$.

\STATE \textbf{Stage 1: Supervised Warm-Up and Factor Anchoring}
\STATE Enable training mode and gradients.
\FOR{$t_{\mathrm w}=1,\ldots,T_{\mathrm w}$}
    \FOR{each source mini-batch $\mathcal B$}
        \STATE Compute $\bm\Pi$, full-scan factors, and predictions using Eqs.~\eqref{eq:signed_graphs}--\eqref{eq:factorized_task_prediction}; update $\Theta$ using $\mathcal L_{\mathrm{cls}}$.
    \ENDFOR
\ENDFOR
\STATE In evaluation mode without gradients, cache $\overline{\bm\Pi}\leftarrow\mathcal A_\omega(\bm Q,\bm E^{\mathrm{ROI}})$.
\STATE Freeze $\Theta_{\mathrm{asg}}$ and use $\overline{\bm\Pi}$ for subsequent pooling, while keeping ROI embeddings trainable.

\STATE \textbf{Stage 2: Periodic Qualification-Guided Learning}
\STATE Set $t\leftarrow0$.
\WHILE{the training stopping criterion is not met}
    \IF{$t\bmod M=0$}
        \STATE Switch to evaluation mode, disable gradients, and set $\mathcal C_Q\leftarrow\varnothing$.
        \STATE For every source subject $i\in\mathcal I_{e,c}$, construct $K$ FC re-estimates with block length $\ell_{b,i}^{e}$ using Eq.~\eqref{eq:fc_reestimation}.
        \STATE Compute paired contributions $\{a_{i,j}^{e,(k)}\}_{k=0}^{K}$ using fixed $\bm S$ and $\overline{\bm\Pi}$.
        \STATE Compute $R_j^{\mathrm{task}}$ using Eqs.~\eqref{eq:factor_task_statistics}--\eqref{eq:task_relevance}.
        \IF{$\sum_jR_j^{\mathrm{task}}>0$}
            \STATE Compute $\pi_j^{\mathrm{task}}$ using Eq.~\eqref{eq:task_relevance} and $q_{e,c,j}^{\mathrm{re}}$ using Eqs.~\eqref{eq:reestimation_variation}--\eqref{eq:reestimation_support}.
            \STATE Compute and cache $\mathcal C_Q\leftarrow\{g_{e,e',c},\overline{\alpha}_{e,e',c,j}\}_{e<e',c,j}$ using Eqs.~\eqref{eq:pairwise_factor_qualification}--\eqref{eq:pairwise_qualification_profile}.
        \ENDIF
        \STATE Discard temporary FC re-estimates and restore training mode and gradients.
    \ENDIF

    \STATE Sample $\mathcal B$; compute full-scan factors, predictions, and $\mathcal L_{\mathrm{cls}}$ using fixed $\bm S$ and $\overline{\bm\Pi}$.
    \STATE Set $\mathcal L_{\mathrm{env}}^{Q}\leftarrow0$.
    \IF{$\mathcal C_Q\neq\varnothing$}
        \STATE Construct weighted representations and empirical distributions using cached $\overline{\alpha}$ and Eqs.~\eqref{eq:pair_specific_representation}--\eqref{eq:qualified_site_distributions}.
        \STATE Evaluate $\mathcal L_{\mathrm{env}}^{Q}$ using cached $g$ and Eq.~\eqref{eq:qualification_guided_alignment}, with the same pairwise representation map and $\varepsilon$ in all cross- and self-transport terms.
    \ENDIF
    \STATE Update $\Theta\setminus\Theta_{\mathrm{asg}}$ using $\mathcal L_{\mathrm{cls}}+\lambda_{\mathrm{env}}\mathcal L_{\mathrm{env}}^{Q}$, holding $\bm S$, $\overline{\bm\Pi}$, and $\mathcal C_Q$ fixed.
    \STATE Set $t\leftarrow t+1$.
\ENDWHILE

\STATE \textbf{Stage 3: Inference on Unseen Sites}
\STATE Switch to evaluation mode and disable gradients.
\FOR{$z=1,\ldots,N_{\mathrm{te}}$}
    \STATE Compute $\bm A_z^{(0)}=\Psi(\bm X_z)$ and full-scan factors $\{\bm r_{z,j}^{(0)}\}_{j=1}^{J}$ using fixed $\bm S$, $\overline{\bm\Pi}$, and the trained mappings.
    \STATE Predict $\widehat p_z=\sigma\!\left(J^{-1/2}\sum_{j=1}^{J}\bm w_j^\top\bm r_{z,j}^{(0)}+b\right)$.
\ENDFOR
\RETURN $\widehat{\bm p}_{\mathrm{te}}=[\widehat p_1,\ldots,\widehat p_{N_{\mathrm{te}}}]^\top$.
\end{algorithmic}
\end{algorithm}